\documentclass[english]{article}

\usepackage{algorithmic}
\usepackage{algorithm}
\usepackage{amsfonts}
\usepackage{amsmath}
\usepackage{amssymb}
\usepackage{amsthm}
\usepackage{array}
\usepackage{arydshln}
\usepackage{bm}
\usepackage{cite}
\usepackage{color}
\usepackage{comment}
\usepackage{dsfont}
\usepackage{enumitem}
\usepackage{float}
\usepackage[T1]{fontenc}
\usepackage{geometry}
\usepackage{graphicx}
\usepackage{grffile}
\usepackage{hyperref}\hypersetup{colorlinks,linkcolor=blue,anchorcolor=blue,citecolor=blue}
\usepackage[latin9]{inputenc}
\usepackage{letltxmacro}
\usepackage{mathrsfs}
\usepackage{mathtools}
\usepackage{multirow}
\usepackage{nicefrac}
\usepackage{subcaption}
\usepackage{tablefootnote}
\usepackage{verbatim}

\usepackage{todonotes}

\newcommand{\br}{\bm{r}}

\newcommand{\bA}{\bm{A}}
\newcommand{\bB}{\bm{B}}

\newcommand{\bD}{\bm{D}}

\newcommand{\bF}{\bm{F}}

\newcommand{\bI}{\bm{I}}

\newcommand{\bM}{\bm{M}}

\newcommand{\bP}{\bm{P}}
\newcommand{\bQ}{\bm{Q}}

\newcommand{\bS}{\bm{S}}

\newcommand{\bU}{\bm{U}}

\newcommand{\bX}{\bm{X}}

\newcommand{\bDelta}{\bm{\Delta}}

\newcommand{\bSigma}{\bm{\Sigma}}

\newcommand{\cI}{\mathcal{I}}

\newcommand{\cL}{\mathcal{L}}
\newcommand{\cM}{\mathcal{M}}

\newcommand{\bcA}{\bm{\mathcal{A}}}
\newcommand{\bcB}{\bm{\mathcal{B}}}

\newcommand{\bcG}{\bm{\mathcal{G}}}

\newcommand{\bcK}{\bm{\mathcal{K}}}

\newcommand{\bcS}{\bm{\mathcal{S}}}
\newcommand{\bcT}{\bm{\mathcal{T}}}

\newcommand{\bcX}{\bm{\mathcal{X}}}
\newcommand{\bcY}{\bm{\mathcal{Y}}}

\newcommand{\RR}{\mathbb{R}}

\newcommand{\Matricize}[2]{\mathcal{M}_{#1}\left(#2\right)}
\newcommand{\Shrink}[2]{\mathcal{T}_{#1}\left(#2\right)}
\newcommand{\hosvd}[2]{\mathcal{H}_{#1}\left(#2\right)}
\newcommand{\norm}[1]{\left\lVert#1\right\rVert}
\newcommand{\inner}[2]{\left\langle#1,#2\right\rangle}

\newcommand{\mfk}{\mathfrak}

\DeclareMathOperator{\bcdot}{\boldsymbol{\cdot}}

\DeclareMathOperator{\dist}{\mathrm{dist}}
\DeclareMathOperator{\fro}{\mathsf{F}}
\DeclareMathOperator{\GL}{\mathrm{GL}}

\DeclareMathOperator{\op}{\mathsf{}}
\DeclareMathOperator{\rank}{\mathrm{rank}}
\DeclareMathOperator{\sgn}{\mathrm{sgn}}
\DeclareMathOperator{\supp}{\mathrm{supp}}
\DeclareMathOperator{\tr}{\mathrm{tr}}

\allowdisplaybreaks
\makeatletter

\theoremstyle{plain}\newtheorem{lemma}{\textbf{Lemma}} 

\newtheorem{theorem}{\textbf{Theorem}}

\theoremstyle{definition}\newtheorem{definition}{\textbf{Definition}}
 
\theoremstyle{remark}\newtheorem{remark}{\textbf{Remark}}

\definecolor{tian}{RGB}{0,150,0}
\definecolor{cm}{RGB}{250,0,200}
\definecolor{yc}{RGB}{255,0,0}
\definecolor{hd}{RGB}{0,180,200}

\definecolor{edits}{RGB}{250,100,100}

\begin{document}
\title{Fast and Provable Tensor Robust Principal Component Analysis \\ via Scaled Gradient Descent}

 \author
 {
 	Harry Dong\thanks{Department of Electrical and Computer Engineering, Carnegie Mellon University, Pittsburgh, PA 15213, USA; Emails:
 		\texttt{\{harryd,ttong1,yuejiec\}@andrew.cmu.edu}. The work of T. Tong was completed while he was a graduate student at CMU.} \\
	CMU \\
	\and
 	Tian Tong\footnotemark[1]\\
 		CMU \\
 		\and
 	Cong Ma\thanks{Department of Statistics, University of Chicago, Chicago, IL 60637, USA; Email:
 		\texttt{congm@uchicago.edu}.} \\
 		 UChicago \\
		\and
 	Yuejie Chi\footnotemark[1] \\
 	CMU
 }

\date{June 2022; \quad Revised: February 2023}

\setcounter{tocdepth}{2}
\maketitle

\begin{abstract}

An increasing number of data science and machine learning problems rely on computation with tensors, which better capture the multi-way relationships and interactions of data than matrices. When tapping into this critical advantage, a key challenge is to develop computationally efficient and provably correct algorithms for extracting useful information from tensor data that are simultaneously robust to corruptions and ill-conditioning. This paper tackles tensor robust principal component analysis (RPCA), which aims to recover a low-rank tensor from its observations contaminated by sparse corruptions, under the Tucker decomposition. 
To minimize the computation and memory footprints, we propose to directly recover the low-dimensional tensor factors---starting from a tailored spectral initialization---via scaled gradient descent (ScaledGD), coupled with an iteration-varying thresholding operation to adaptively remove the impact of corruptions. Theoretically, we establish that the proposed algorithm converges linearly to the true low-rank tensor at a constant rate that is independent with its condition number, as long as the level of corruptions is not too large.
Empirically, we demonstrate that the proposed algorithm achieves better and more scalable performance than state-of-the-art tensor RPCA algorithms through synthetic experiments and real-world applications.

\end{abstract}


\tableofcontents

% !TEX root = ./Tensor_RPCA.tex
\section{Introduction}
\label{sec:intro}

An increasing number of data science and machine learning problems rely on computation with tensors \cite{kolda2009tensor,papalexakis2016tensors}, which better capture the multi-way relationships and interactions of data than matrices; examples include recommendation systems \cite{karatzoglou2010multiverse}, topic modeling \cite{anandkumar2014tensor}, image processing \cite{liu2012tensor}, anomaly detection \cite{li2015low}, and so on. 
Oftentimes the data object of interest can be represented by a much smaller number of latent factors than what its ambient dimension suggests, which induces a low-rank structure in the underlying tensor. Unlike the matrix case, the flexibility of tensor modeling allows one to decompose a tensor under several choices of popular decompositions. The particular tensor decomposition studied in this paper is the Tucker decomposition, where a third-order tensor $\bcX_{\star}\in\RR^{n_1\times n_2\times n_3}$ is {\em low-rank} if it can be decomposed as\footnote{Note that there are several other popular notation for denoting the Tucker decomposition; our choice is made to facilitate the presentation of the analysis.}
\begin{align*}
\bcX_{\star}= \big( \bU^{(1)}_{\star},\bU^{(2)}_{\star},\bU^{(3)}_{\star} \big)\bcdot\bcG_{\star},
\end{align*}
where $\bU^{(1)}_{\star}\in\RR^{n_1\times r_1}$, $\bU^{(2)}_{\star}\in\RR^{n_2\times r_2}$, $\bU^{(3)}_{\star}\in\RR^{n_3\times r_3}$ are the factor matrices along each mode, $\bcG_{\star}\in\RR^{r_1\times r_2\times r_3}$ is the core tensor, and $\{r_i\}_{i=1}^3$ are the rank of each mode; see Section~\ref{sec:formulation} for the precise definition. If we flatten the tensor along each mode, then the obtained matrices are all correspondingly low-rank:
\begin{align*}
r_1 = \rank\big(\cM_1(\bcX_{\star})\big), \quad r_2 = \rank\big(\cM_2(\bcX_{\star})\big), \quad r_3 = \rank\big(\cM_3(\bcX_{\star})\big),
\end{align*}
where $\cM_k(\cdot)$ denotes the matricization of an input tensor along the $k$-th mode ($k=1,2,3$). Intuitively, this means that the fibers along each mode lie in the same low-dimensional subspace.   In other words, the tensor $\bcX_{\star}$ has a multi-linear rank $\br=(r_1,r_2,r_3)$, where typically $r_k\ll n_k$. Throughout the paper, we denote  $n \coloneqq \max_k n_k$ and $r \coloneqq \max_k r_k$.

This paper tackles tensor robust principal component analysis (RPCA), which aims to recover a low-rank tensor $\bcX_{\star}$ from its observations contaminated by sparse corruptions. Mathematically, imagine we have access to a set of measurements given as
\begin{align*}
\bcY = \bcX_{\star} + \bcS_{\star},
\end{align*} 
where $\bcS_{\star}\in\RR^{n_1\times n_2\times n_3}$ is a sparse tensor---in which the number of nonzero entries is much smaller than its ambient dimension---modeling corruptions or gross errors in the observations due to sensor failures, anomalies, or adversarial perturbations. Our goal is to recover $\bcX_{\star}$ from the corrupted observation $\bcY$ in a computationally efficient and provably correct manner.

%Many real-word applications encounter multi-way data  modeled as tensors, e.g.~matrices are second order tensors. A third order tensor $\bcX_{\star}\in\RR^{n_1\times n_2\times n_3}$ has $n_1\times n_2\times n_3$ entries. 

%There are two types of tensor decompositions: Tucker decomposition requires factors $\bU^{(1)}_{\star},\bU^{(2)}_{\star},\bU^{(3)}_{\star}$ to be orthonormal; CP decomposition requires the core tensor $\bcG_{\star}$ to be diagonal. Note that these two properties cannot be simultaneously guaranteed, which is a main difference from SVD for matrices. We will mainly focus on Tucker decomposition (because it is easier to analyze).

\subsection{Our approach}

In this paper, we propose a novel iterative method for tensor RPCA with provable convergence guarantees. To minimize the memory footprint, we aim to directly estimate the ground truth factors, collected in $\bF_{\star}=(\bU_{\star}^{(1)},\bU_{\star}^{(2)},\bU_{\star}^{(3)},\bcG_{\star})$, via optimizing the following objective function:
\begin{align}\label{eq:loss}
\cL(\bF,\bcS) \coloneqq \frac{1}{2}\left\| \big(\bU^{(1)},\bU^{(2)},\bU^{(3)} \big)\bcdot\bcG+\bcS-\bcY \right\|_{\fro}^{2},
\end{align}
where $\bF=(\bU^{(1)},\bU^{(2)},\bU^{(3)},\bcG)$ and $\bcS$ are the optimization variables for the tensor factors and the corruption tensor, respectively. Despite the nonconvexity of the objective function, a simple and intuitive approach is to update the tensor factors via gradient descent, which, unfortunately, converges slowly even when the problem instance is moderately ill-conditioned \cite{han2020optimal}. 
%On the other hand, the recently proposed scaled gradient descent (ScaledGD) \cite{tong2021scaling} method, which leverages preconditioning to provably accelerate ill-conditioned low-rank tensor estimation, does not handle outliers in the model.

On a high level, our proposed method alternates between corruption pruning (i.e., updating $\bcS$) and factor refinements (i.e., updating $\bF$). At the beginning of each iteration, we update the corruption tensor $\bcS$ via thresholding the observation residuals as
\begin{subequations}
\begin{align}\label{eq:outlier_threshold}
\bcS_{t+1}  = \Shrink{\zeta_{t+1}}{\bcY - \big(\bU_t^{(1)}, \bU_t^{(2)}, \bU_t^{(3)} \big) \bcdot \bcG_t}, \qquad t=0,1,\ldots
\end{align}
where $\bcS_{t+1}$ is the update of the corruption tensor at the $t$-th iteration, $\Shrink{\zeta_{t+1}}{\cdot}$ trims away the entries with magnitudes smaller than an iteration-varying threshold $\zeta_{t+1}$ that is carefully orchestrated, e.g., following a geometric decaying schedule. As the estimate of the data tensor $\bcX_t =\big(\bU_t^{(1)}, \bU_t^{(2)}, \bU_t^{(3)} \big) \bcdot \bcG_t$ gets more accurate, the observation residual becomes more aligned with the corruptions, therefore the thresholding operator \eqref{eq:outlier_threshold} becomes more effective in identifying and removing the impact of corruptions. Turning to the low-rank tensor factors $\bF$, motivated by the recent success of scaled gradient descent (ScaledGD) \cite{tong2021accelerating,tong2020low,tong2021scaling} for accelerating ill-conditioned low-rank estimation, we propose to 
update the tensor factors iteratively by descending along the scaled gradient directions:  
\begin{align}
\begin{split}
\bU_{t+1}^{(k)} &= \bU_{t}^{(k)} - \eta\nabla_{\bU_t^{(k)}}\cL(\bF_{t},\bcS_{t+1})\big(\breve{\bU}_t^{(k)\top} \breve{\bU}_t^{(k)} \big)^{-1}, \qquad k=1,2,3, \qquad \text{and}\\
\bcG_{t+1} &= \bcG_{t} - \eta \Big( \big(\bU_{t}^{(1)\top}\bU_{t}^{(1)} \big)^{-1}, \big(\bU_{t}^{(2)\top}\bU_{t}^{(2)} \big)^{-1}, \big(\bU_{t}^{(3)\top}\bU_{t}^{(3)}\big)^{-1} \Big)\bcdot\nabla_{\bcG_t}\cL(\bF_{t}, \bcS_{t+1}).
\end{split}\label{eq:ScaledGD}
\end{align}
Here, $\bF_t=(\bU_t^{(1)},\bU_t^{(2)},\bU_t^{(3)},\bcG_t)$ is the estimate of the tensor factors at the $t$-th iteration, $\nabla_{\bU^{(k)}}\cL(\bF , \bcS)$  and  $\nabla_{\bcG}\cL(\bF, \bcS)$ are the partial derivatives of $\cL(\bF,\bcS)$ with respect to the corresponding variables, $\eta>0$ is the learning rate, and
\begin{align*}  
\breve{\bU}_{t}^{(1)} & = \big(\bU_{t}^{(3)}\otimes\bU_{t}^{(2)} \big)\cM_{1}(\bcG_{t})^{\top}, \quad
\breve{\bU}_{t}^{(2)}  = \big(\bU_{t}^{(3)}\otimes\bU_{t}^{(1)} \big)\cM_{2}(\bcG_{t})^{\top}, \quad
\breve{\bU}_{t}^{(3)}  = \big(\bU_{t}^{(2)}\otimes\bU_{t}^{(1)} \big)\cM_{3}(\bcG_{t})^{\top}
\end{align*}
\end{subequations}
are used to construct the preconditioned directions of the gradients, with $\otimes$ denoting the Kronecker product. With the preconditioners, ScaledGD balances the tensor factors to find better descent directions, the benefits of which are more accentuated in ill-conditioned tensors where the convergence rate of vanilla gradient descent degenerates significantly, while ScaledGD is capable of maintaining a linear rate of convergence regardless of the condition number. 

%In \cite{tong2021scaling}, this balancing of tensor factors helps us derive the same linear convergence rate regardless of the condition number.

%ScaledGD differs from vanilla gradient descent in its scaling terms, $\big(\bU_{t}^{(k)\top}\bU_{t}^{(k)}\big)^{-1}$ and $\big(\breve{\bU}_t^{(k)\top} \breve{\bU}_t^{(k)} \big)^{-1}$ for all $k$. Intuitively, w

 \paragraph{Theoretical guarantees.} Coupled with a tailored spectral initialization scheme, the proposed ScaledGD method converges linearly to the true low-rank tensor in both the Frobenius norm and the entrywise $\ell_\infty$ norm at a constant rate that is independent of its condition number, as long as the level of corruptions---measured in terms of the fraction of nonzero entries per fiber---does not exceed the order of 
 $\tfrac{1}{\mu^2 \kappa r_1 r_2 r_3 }$, 
 where $\mu$ and $\kappa$ are respectively the incoherence parameter and the condition number of the ground truth tensor $\bcX_{\star}$ (to be formally defined later).  This not only enables fast global convergence by virtue of following the scaled gradients rather than the vanilla gradients \cite{tong2021scaling}, but also lends additional robustness to finding the low-rank Tucker decomposition despite the presence of corruptions and gross errors. Moreover, our work provides the first refined entrywise error analysis for tensor RPCA, suggesting the errors are distributed evenly across the entries when the ground low-rank truth tensor is incoherent. To corroborate the theoretical findings, we further demonstrate that the proposed ScaledGD algorithm achieves better and more scalable performance than state-of-the-art matrix and tensor RPCA algorithms through synthetic experiments and real-world applications.

\paragraph{Comparisons to prior art.} While tensor RPCA has been previously investigated under various low-rank tensor decompositions, e.g., \cite{lu2016tensor,anandkumar2016tensor,driggs2019tensor}, the development of provably efficient algorithms under the Tucker decomposition remains scarce. The most closely related work is \cite{cai2021generalized}, which proposed a Riemannian gradient descent algorithm for the same tensor RPCA model as ours. %As shall be detailed later, our theoretical guarantee allows a much larger level of corruptions than \cite{cai2021generalized}.
%\paragraph{Comparison with \cite{cai2021generalized}.}
Their algorithm is proven to also achieve a constant rate of convergence---at a higher per-iteration expense---as long as the fraction of outliers per fiber does not exceed the order of $\min\left\{\frac{1}{\mu_s^4\kappa_s^{14}r^2 \log^2 n},\, \frac{1}{\mu_s^{12}\kappa_s^{12} r^3} \right\}$ (cf. \cite[Theorem 5.1]{cai2021generalized}), where $\mu_s$ and $\kappa_s$ are the spikiness parameter and the worst-case condition number of $\bcX_{\star}$, respectively. Using the relation $\mu\leq \mu_s^2 \kappa_s^2$ (cf. \cite[Lemma 13.5]{cai2021generalized}) and $\kappa\leq \kappa_s$ (cf. \eqref{eq:compare_kappa}) to conservatively translate our bound, our algorithm succeeds as long as the corruption level is below the order of $\frac{1}{\mu_s^4 \kappa_s^5 r^3}$, which is still significantly higher than that allowed in \cite{cai2021generalized}, when the outliers are evenly distributed across the fibers. See additional numerical comparisons in Section~\ref{sec:numerical}.

\subsection{Related works}
 
Broadly speaking, our work falls under the recent surge of developing both computationally efficient and provably correct algorithms for high-dimensional signal estimation via nonconvex optimization, which has been particularly fruitful for problems with inherent low-rank structures; we refer interested readers to the recent overviews \cite{chi2019nonconvex,chen2018harnessing} for further pointers. In the rest of this section, we focus on works that are most closely related to our paper.

\paragraph{Provable algorithms for matrix RPCA.} The matrix RPCA problem, which aims to decompose a low-rank matrix and a sparse matrix from their sum, has been heavily investigated since its introduction in the seminar papers \cite{candes2009robustPCA,chandrasekaran2011siam}. Convex relaxation based approaches, which minimize a weighted sum of the nuclear norm of the data matrix and the $\ell_1$ norm of the corruption matrix, have been demonstrated to achieve near-optimal performance guarantees \cite{candes2009robustPCA,wright2009robust,chandrasekaran2011siam,lin2010augmented,chen2021bridging,chen2013spectral}. However, their computational and memory complexities are prohibitive when applied to large-scale problem instances; for example, solving the resulting semidefinite programs via accelerated proximal gradient descent \cite{toh2010accelerated} only results in a sublinear rate of convergence with a per-iteration complexity that scales cubically with the matrix dimension. To address the computational bottleneck, nonconvex methods have been developed to achieve both statistical and computational efficiencies simultaneously \cite{netrapalli2014non,cai2019accelerated,gu2016low,yi2016fast,tong2021accelerating,cai2021learned}. Our tensor RPCA algorithm draws inspiration from \cite{tong2021accelerating,tong2021scaling}, which adopt a factored representation of the low-rank object and invoke scaled gradient updates to bypass the dependence of the convergence rate on the condition number. The matrix RPCA method in \cite{cai2021learned} differs from \cite{tong2021accelerating} by using a threshold-based trimming procedure---which we also adopt---rather than a sorting-based one to identify the sparse matrix, for further computational savings.

\paragraph{Provable algorithms for tensor RPCA.} Moving onto tensors, although one could unfold a tensor and feed the resulting matrices into a matrix RPCA algorithm \cite{goldfarb2014robust,zhang2019robust}, destroying the tensor structure through matricizations can result in suboptimal performance because it ignores the higher-order interactions \cite{yuan2016tensor}. Therefore, it is desirable to directly operate in the tensor space. However, tensor algorithms encounter unique issues not present for matrices.
For instance, while it appears straightforward to generalize the convex relaxation approach to tensors, it has been shown that computing the tensor nuclear norm is in fact NP-hard \cite{friedland2018nuclear}; a similar drawback is applicable to the atomic norm formulation studied in \cite{driggs2019tensor}. Tensor RPCA has also been studied under different low-rank tensor decompositions, a small number of samples including the tubal rank \cite{lu2016tensor,lu2019tensor} and the CP-rank \cite{anandkumar2016tensor,driggs2019tensor}. These algorithms are not directly comparable with ours which uses the multilinear rank.

%worked around this obstacle while sticking to convex optimization but resorts to partitioning tensors into matrix slices. 
%
%however, it is still expensive computationally. 
%\cite{anandkumar2016tensor} opts for an iterative nonconvex model that better preserves the tensor structure and guarantees convergence to the global minimum at the cost of $\cO(n^{4+c} r^2)$ for some small constant $c$ per iteration.
%\cite{driggs2019tensor} also takes a nonconvex approach with a substitution of the tensor nuclear norm with a regularizer derived through Burer-Monteiro factorization.
%However, optimization by their algorithm may return a local minimum, although they provide conditions to determine if such an output is a good approximation to the global minimum which includes having a large enough CP-rank. 
%\cite{lu2016tensor} uses the tubal rank of a tensor, and \cite{anandkumar2016tensor,driggs2019tensor} use the CP-rank, so we cannot directly compare these methods with ours which uses the multilinear rank.
%%Despite this, \cite{lu2016tensor,anandkumar2016tensor,driggs2019tensor} all show improvement from matrix RPCA methods.
%\hd{add applications and unfolding literature}

\paragraph{Robust low-rank tensor recovery.} Broadly speaking, tensor RPCA concerns with reconstructing a high-dimensional tensor with certain low-dimensional structures from incomplete and corrupted observations. Pertaining to works that deal with the Tucker decomposition, \cite{xia2019polynomial} proposed a gradient descent based algorithm for tensor completion, \cite{tong2021scaling,tong2022accelerating} proposed scaled gradient descent algorithms for  tensor regression and tensor completion (which our algorithm also adopts), \cite{zhang2021low} proposed a Gauss-Newton algorithm for tensor regression that achieves quadratic convergence, \cite{wang2021entrywise} proposed a Riemannian gradient method with entrywise convergence guarantees, and \cite{ahmed2020tensor} studied tensor regression assuming the underlying tensor is simultaneously low-rank and sparse.

\begin{comment}

One application of RPCA is in autoencoders \cite{zhou2017anomaly}.

\begin{table}
\renewcommand{\arraystretch}{2}
\begin{tabular}{c ||c|c|c|c|c} 
 \hline
 Algorithm &GD &AltProj &AccAltProj &ScaledGD &LRPCA \\ 
 \hline
 $\alpha$ tolerance &$\frac{1}{\max\{ \mu r^{3/2} \kappa^{3/2}, \mu r \kappa^2\}}$ &$\frac{1}{\mu r}$ &$\frac{1}{\max\{ \mu r^2 \kappa^3, \mu^{3/2} r^2 \kappa, \mu^2 r^2\}}$ &$\frac{1}{ \mu r^{3/2} \kappa}$ &$\frac{1}{ \mu r^{3/2} \kappa}$ \\
 \hline
 Runtime complexity &$n^2r$ &$n^2 r^2$ &$n^2 r$ &$n^2 r$ &$n^2 r$ \\ 
 \hline
\end{tabular}
\caption{Sparsity tolerance and runtime comparison among different matrix RPCA algorithms. For simplicity, we are assuming $\bX_\star \in \RR^{n \times n}$.}
\label{table:matrix_compare}
\end{table}

\end{comment}

\subsection{Notation and tensor preliminaries}

Throughout this paper, we use boldface calligraphic letters (e.g.~$\bcX$) to denote tensors, and boldface capitalized letters (e.g.~$\bX$) to denote matrices. For any matrix $\bX$, let $\sigma_{i}(\bX)$ be its $i$-th largest singular value, and $\sigma_{\max}(\bX)$ (resp.~$\sigma_{\min}(\bX)$) to denote its largest (resp.~smallest) nonzero singular value. Let $\|\bX\|_{\op}$, $\|\bX\|_{\fro}$, $\|\bX\|_{2,\infty}$, and $\|\bX\|_{\infty}$ be the spectral norm, the Frobenius norm, the $\ell_{2,\infty}$ norm (largest $\ell_2$ norm of the rows), and the entrywise $\ell_{\infty}$ norm of a matrix $\bX$, respectively. The $r\times r$ identity matrix is denoted by $\bI_{r}$. The set of invertible matrices in $\RR^{r\times r}$ is denoted by $\GL(r)$.
% Let $\bX_{i,j}$, $\bX_{i,:}$ and $\bX_{:,j}$ be the $(i,j)$-th entry, the $i$-th row, and the $j$-th column of $\bX$, respectively. 

We now describe some preliminaries on tensor algebra that are used throughout this paper. For a tensor $\bcX\in\mathbb{R}^{n_1\times n_2\times n_3}$, let $[\bcX]_{i,j,k}$ be its $(i,j,k)$-th entry. For a tensor $\bcX \in \mathbb{R}^{n_1\times n_2 \times n_3}$, suppose it can be represented via the multilinear multiplication  
 $$\bcX = \big( \bU^{(1)}, \bU^{(2)}, \bU^{(3)} \big) \bcdot \bcG,$$
where $\bU^{(k)}\in\mathbb{R}^{n_k\times r_k}$, $k=1,2,3,$, and $\bcG\in  \mathbb{R}^{r_1\times r_2 \times r_3}$. Equivalently, the entries of $\bcX$ can be expressed as
\begin{equation*}
	\big[\bcX \big]_{i_1, i_2, i_3} = \sum_{j_1 = 1}^{r_1} \sum_{j_2 = 1}^{r_2} \sum_{j_3 = 1}^{r_3} \big[\bU^{(1)}\big]_{i_1, j_1} \big[\bU^{(2)}\big]_{i_2, j_2} \big[\bU^{(3)}\big]_{i_3, j_3} \big[\bcG\big]_{j_1, j_2, j_3}.
\end{equation*}
The multilinear multiplication possesses several nice properties. A crucial one is that for any $\bB^{(k)} \in \RR^{r_k \times r_k}$, $k=1,2,3$, it holds that
\begin{align} \label{eq:multilinear}
    \big(\bU^{(1)} \bB^{(1)}, \bU^{(2)} \bB^{(2)}, \bU^{(3)} \bB^{(3)} \big) \bcdot \bcG = \big(\bU^{(1)}, \bU^{(2)}, \bU^{(3)} \big) \bcdot \Big( \big(\bB^{(1)}, \bB^{(2)}, \bB^{(3)} \big) \bcdot \bcG \Big).
\end{align} 
In addition, if we flatten the tensor $\bcX$ along different modes, the obtained matrices obey the following low-rank decompositions: 
\begin{subequations} \label{eq:matricization_property}
\begin{align}
    \Matricize{1}{\bcX} &= \bU^{(1)}  \Matricize{1}{\bcG} \big(\bU^{(3)} \otimes \bU^{(2)} \big)^\top = \bU^{(1)}  \Breve{\bU}^{(1)\top}, \qquad \Breve{\bU}^{(1)} : =  \big(\bU^{(3)} \otimes \bU^{(2)} \big) \Matricize{1}{\bcG}^{\top} ,\label{matricize1}\\
    \Matricize{2}{\bcX} &= \bU^{(2)} \Matricize{2}{\bcG}(\bU^{(3)} \otimes \bU^{(1)})^\top   =\bU^{(2)}  \Breve{\bU}^{(2)\top},  \qquad \Breve{\bU}^{(2)}:= \big(\bU^{(3)} \otimes \bU^{(1)} \big) \Matricize{2}{\bcG}^{\top} , \label{matricize2} \\
    \Matricize{3}{\bcX} &= \bU^{(3)}  \Matricize{3}{\bcG}(\bU^{(2)} \otimes \bU^{(1)})^\top  =\bU^{(3)}  \Breve{\bU}^{(3)\top} , \qquad \Breve{\bU}^{(3) }:= \big(\bU^{(2)} \otimes \bU^{(1)} \big) \Matricize{3}{\bcG}^{\top}. \label{matricize3}
\end{align}
\end{subequations}

Given two tensors $\bcA$ and $\bcB$, their inner product is defined as
$\langle\bcA, \bcB \rangle = \sum_{i_1,i_2,i_3} \bcA_{i_1,i_2,i_3} \bcB_{i_1,i_2,i_3}$. The inner product satisfies the following property:
\begin{equation}\label{eq:tensor_properties_d} 
\left\langle \big(\bU^{(1)},\bU^{(2)},\bU^{(3)} \big)\bcdot\bcG, \bcX\right\rangle = \left\langle\bcG, \big(\bU^{(1)\top},\bU^{(2)\top},\bU^{(3)\top} \big)\bcdot\bcX\right\rangle.
\end{equation}
Denote the Frobenius norm and the $\ell_{\infty}$ norm of $\bcX$ as $\|\bcX\|_{\fro}=\sqrt{\langle\bcX,\bcX\rangle}$ and $\|\bcX\|_{\infty} = \max_{i_1,i_2,i_3}| \bcX_{i_1,i_2,i_3}|$, respectively. It follows that for $\bQ_k \in \RR^{r_k\times r_k}$, $k=1,2,3$:
\begin{align} \label{eq:tensor_properties_e}
\left\|\big(\bQ_{1},\bQ_{2},\bQ_{3} \big)\bcdot\bcG\right\|_{\fro} &\le \|\bQ_{1}\|_{\op}\|\bQ_{2}\|_{\op}\|\bQ_{3}\|_{\op}\|\bcG\|_{\fro} .
\end{align}

%A useful relation is that 
%\begin{align}\label{eq:tensor_inner}
%\langle\bcX_{1},\bcX_{2}\rangle=\langle\cM_{k}(\bcX_{1}),\cM_{k}(\bcX_{2})\rangle , \quad k=1,2,3,
%\end{align}
%which allows one to move between the tensor representation and the unfolded matrix representation.

Let $\br= (r_1, r_2, r_3)$. For a tensor $\bcX\in\mathbb{R}^{n_1 \times n_2 \times n_3}$, let its rank-$\br$ higher-order singular value decomposition (HOSVD) $\hosvd{\br}{\bcX}$ be  
\begin{equation}
    \hosvd{\br}{\bcX} = (\bU^{(1)}, \bU^{(2)}, \bU^{(3)}, \bcG),
\end{equation}
where $\bU^{(k)}$ is the top $r_k$ left singular vectors of $\Matricize{k}{\bcX}$, $k=1,2,3$, and $\bcG = (\bU^{(1) \top}, \bU^{(2) \top}, \bU^{(3) \top})\bcdot\bcX$ is the core tensor. The HOSVD is an extension of the matrix SVD and can be seen as a special case of the Tucker decomposition; see \cite{bergqvist2010higher} for an exposition. Although there are faster methods---such as \cite{vannieuwenhoven2012new}---available, one straightforward way of computing the HOSVD is to obtain the singular vectors from performing matrix SVD on each matricization of $\bcX$. With these vectors, we can construct each $\bU^{(k)}$, followed by finding $\bcG = (\bU^{(1) \top}, \bU^{(2) \top}, \bU^{(3) \top})\bcdot\bcX$ to complete the process. In contrast to its matrix counterpart, the core tensor $\bcG$ will not necessarily be diagonal.

% By Equation \ref{multilinear}, we have $\bcG = ((\bU^{(1)})^\top , (\bU^{(2)})^\top , (\bU^{(3)})^\top) \bcdot \bcX$.

% !TEX root = ./Tensor_RPCA.tex

\section{Main results}

\subsection{Problem formulation}
\label{sec:formulation}

% tensor mathematical model, assumption, goal
 
Suppose that the ground truth tensor $\bcX_\star \in \RR^{n_1 \times n_2 \times n_3}$ with multilinear rank $\br = (r_1, r_2, r_3)$ admits the following Tucker decomposition
\begin{equation} \label{tucker_star}
    \bcX_\star = \big( \bU_\star^{(1)}, \bU_\star^{(2)}, \bU_\star^{(3)} \big) \bcdot \bcG_\star,
\end{equation}
where $\bU^{(1)}_{\star}\in\RR^{n_1\times r_1}$, $\bU^{(2)}_{\star}\in\RR^{n_2\times r_2}$, $\bU^{(3)}_{\star}\in\RR^{n_3\times r_3}$ are the factor matrices along each mode,  and $\bcG_{\star}\in\RR^{r_1\times r_2\times r_3}$ is the core tensor. The Tucker decomposition is not unique since for any $\bQ^{(k)} \in \text{GL}(r_k)$, $k=1,2,3$, in view of \eqref{eq:multilinear}, we have
\begin{align*}
    \big(\bU^{(1)}, \bU^{(2)}, \bU^{(3)} \big) \bcdot \bcG = \big(\bU^{(1)} \bQ^{(1)}, \bU^{(2)} \bQ^{(2)}, \bU^{(3)} \bQ^{(3)} \big) \bcdot \bcG_{\bQ}
\end{align*}
where $\bcG_{\bQ} = \big((\bQ^{(1)})^{-1}, (\bQ^{(2)})^{-1}, (\bQ^{(3)})^{-1} \big) \bcdot \bcG$. Without loss of generality, to address ambiguity, we set the ground truth $\bF_{\star}=(\bU_{\star}^{(1)},\bU_{\star}^{(2)},\bU_{\star}^{(3)},\bcG_{\star})$ to satisfy that for each mode, $\bU_{\star}^{(k)} \in \RR^{n_k \times r_k}$ to have orthonormal columns, and 
\begin{equation} \label{eq:singular_value_matrix}
\Matricize{k}{\bcG_{\star}} \Matricize{k}{\bcG_{\star}}^\top = \big(\bSigma_{\star}^{(k)} \big)^2,
\end{equation} 
the squared singular value matrix $\bSigma_{\star}^{(k)}$ of $\Matricize{k}{\bcX_{\star}}$. This can be easily met, for example, by taking the tensor factors $\bF_{\star}$ as the HOSVD of $\bcX_{\star}$.

\paragraph{Observation model and goal.} Suppose that we collect a set of corrupted observations of $\bcX_{\star}$ as
\begin{equation}\label{eq:obs_model}
\bcY = \bcX_{\star} + \bcS_{\star},
\end{equation}
where $\bcS_{\star}$ is the corruption tensor. The problem of tensor RPCA seeks to separate $\bcX_{\star}$ and $\bcS_{\star}$ from their sum $\bcY$ as efficiently and accurately as possible. %\edits{Note that when there are no corruptions, taking the HOSVD of $\bcY$ will return \eqref{tucker_star}, the desired Tucker decomposition of $\bcX_\star$.}
 
\paragraph{Key quantities.} Obviously, the tensor RPCA problem is ill-posed without imposing additional constraints on the low-rank tensor $\bcX_{\star}$ and the corruption tensor $\bcS_{\star}$, which are crucial in determining the performance of the proposed algorithm. We first introduce the incoherence parameter of the tensor $\bcX_\star $.

\begin{definition}[Incoherence]
The incoherence parameter $\mu$ of $\bcX_\star $ is defined as
\begin{align}
     \mu := \max_k \left\{\frac{n_k}{r_k} \norm{\bU_\star^{(k)}}_{2, \infty}^2 \right\},
\end{align}
where $\bcX_\star = \big(\bU_\star^{(1)}, \bU_\star^{(2)}, \bU_\star^{(3)}\big) \bcdot \bcG_\star$ is its Tucker decomposition.
\end{definition}

The incoherence parameter roughly measures how spread the energy of $\bcX_\star$ is over its entries---the energy is more spread as $\mu$ gets smaller. Moreover, we define a new notion of condition number that measures the conditioning of the ground truth tensor $\bcX_\star$ as follows, which is weaker than previously used notions. 

\begin{definition}[Condition number]
The condition number $\kappa$ of $\bcX_{\star}$ is defined as
\begin{equation} \label{eq:reduced_kappa}
	\kappa  := \frac{\min_k \sigma_{\max}(\Matricize{k}{\bcX_\star})}{\min_k \sigma_{\min}(\Matricize{k}{\bcX_\star})}.
\end{equation}
\end{definition}
\noindent 
With slight abuse of terminology, denote 
\begin{align} \label{eq:sigma_min_X}
\sigma_{\min}(\bcX_{\star}) = \min_{k} \sigma_{\min}(\cM_k(\bcX_\star))
\end{align} 
as the minimum nonzero singular value of $\bcX_\star$.

\begin{remark} The above-defined condition number can be much smaller than the {\em worst-case} condition number $\kappa_s$ used in prior analyses \cite{tong2021scaling,cai2021generalized,han2020optimal}, which is defined as
\begin{equation} \label{eq:compare_kappa}
	\kappa_s := \frac{\max_k \sigma_{\max}(\Matricize{k}{\bcX_\star})}{\min_k \sigma_{\min}(\Matricize{k}{\bcX_\star})} \geq\frac{\min_k \sigma_{\max}(\Matricize{k}{\bcX_\star})}{\min_k \sigma_{\min}(\Matricize{k}{\bcX_\star})} =  \kappa.
\end{equation}
Furthermore, the condition number $\kappa$ is also upper bounded by the largest condition number of the matricization along different modes, i.e., $\kappa \leq \max_k \kappa_k  =\max_k  \frac{\sigma_{\max}(\Matricize{k}{\bcX_\star})}{\sigma_{\min}(\Matricize{k}{\bcX_\star})}  $.  
\end{remark}

Turning to the corruption tensor, we consider a deterministic sparsity model following the matrix case~\cite{chandrasekaran2011siam,netrapalli2014non,yi2016fast}, where $\bcS_{\star}$ contains at most a small fraction of nonzero entries per fiber. This is captured in the following definition. 

%\begin{definition}[$\alpha$-fraction sparsity]
%The corruption tensor $\bcS_{\star}$ is $\alpha$-fraction sparse, i.e., $\bcS_{\star}\in\bcS_{\alpha}$, where  
%\begin{multline}
%\bcS_{\alpha}\coloneqq \Big\{\bcS\in\RR^{n_{1}\times n_{2}\times n_3}:\; \|\bcS_{i_1,:,:}\|_{0}\le\alpha n_{2} n_{3}, \; \|\bcS_{:,i_2,:}\|_{0}\le\alpha n_{1} n_{3}, \; \|\bcS_{:,:,i_3}\|_{0}\le\alpha n_{1}n_{2},  \\
%  \mbox{ for all  } 1\leq i_k \leq n_k, \quad k=1,2,3 \Big\}.\label{eq:S_alpha}
%\end{multline}
%\end{definition}

\begin{definition}[$\alpha$-fraction sparsity]
The corruption tensor $\bcS_{\star}$ is $\alpha$-fraction sparse, i.e., $\bcS_{\star}\in\bcS_{\alpha}$, where  
\begin{multline}
\bcS_{\alpha}\coloneqq \Big\{\bcS\in\RR^{n_{1}\times n_{2}\times n_3}:\; \|\bcS_{i_1,i_2,:}\|_{0}\le\alpha n_{3}, \; \|\bcS_{i_1,:,i_3}\|_{0}\le\alpha n_{2}, \; \|\bcS_{:,i_2,i_3}\|_{0}\le\alpha n_{1},  \\
  \mbox{ for all  } 1\leq i_k \leq n_k, \quad k=1,2,3 \Big\}.\label{eq:S_alpha}
\end{multline}
\end{definition}
% consider mentioning some applications

%\edits{Equivalently, each row of $\Matricize{k}{\bcS_\star}$ contains at most $\alpha$-fraction nonzero entries for $k=1, 2, 3$.} 
With this setup in hand, we are now ready to describe the proposed algorithm.

\subsection{Proposed algorithm}

Our algorithm alternates between corruption removal and factor refinements. To remove the corruption, we use the following soft-shrinkage operator that trims the magnitudes of the entries by the amount of some carefully pre-set threshold.

\begin{definition}[Soft-shrinkage operator]
For an order-$3$ tensor $\bcX$, the soft-shrinkage operator $\Shrink{\zeta}{\cdot}:\, \mathbb{R}^{n_1\times n_2\times n_3} \mapsto \mathbb{R}^{n_1\times n_2\times n_3}$ with threshold $\zeta>0$ is defined as
\begin{align*}
    \big[\Shrink{\zeta}{\bcX}\big]_{i_1, i_2, i_3} := \sgn\big([\bcX]_{i_1, i_2, i_3}\big) \,\cdot\,\max\big(0, \big| [\bcX]_{i_1, i_2, i_3} \big| - \zeta \big) .
\end{align*}
\end{definition}
The soft-shrinkage operator $\Shrink{\zeta}{}$ sets entries with magnitudes smaller than $\zeta$ to $0$, while uniformly shrinking the magnitudes of the other entries by $\zeta$. At the beginning of each iteration, the corruption tensor is updated via
\begin{subequations} \label{eq:alg1_updates}
\begin{align} \label{eq:update_corruption} 
\bcS_{t+1}  = \Shrink{\zeta_{t+1}}{\bcY - \big(\bU_t^{(1)}, \bU_t^{(2)}, \bU_t^{(3)} \big) \bcdot \bcG_t},   
\end{align}
with the schedule $\zeta_t$ to be specified shortly. With the newly updated estimate of the corruption tensor, the tensor factors are then updated by scaled gradient descent \cite{tong2021scaling}, for which they they are computed according to \eqref{eq:ScaledGD} with respect to $\cL(\bF_t, \bcS_{t+1})$ in \eqref{eq:loss}:
\begin{align}
    \bU^{(k)}_{t+1} &= \bU^{(k)}_t - \eta \nabla_{\bU^{(k)}_t} \cL(\bF_t, \bcS_{t+1}) \big(\Breve{\bU}_t^{(k)\top} \Breve{\bU}_t^{(k)} \big)^{-1} \nonumber \\
%    &= \bU^{(k)}_t - \eta \left( \bU_t^{(k)} \Breve{\bU}_t^{(k)\top}   + \Matricize{k}{\bcS_{t+1}} -  \Matricize{k}{\bcY}\right) \Breve{\bU}_t^{(k)} \big(\Breve{\bU}_t^{(k)\top} \Breve{\bU}_t^{(k)} \big)^{-1} \nonumber \\
    &= (1 - \eta) \bU_t^{(k)} - \eta \big(\Matricize{k}{\bcS_{t+1}} -  \Matricize{k}{\bcY} \big)\Breve{\bU}_t^{(k)} \big(\Breve{\bU}_t^{(k)\top} \Breve{\bU}_t^{(k)} \big)^{-1} \label{a_update3}
    \end{align}
for $k=1,2,3$ and
\begin{align}    
    \bcG_{t+1} &= \bcG_t - \eta \left( \big( \bU^{(1)\top}_t \bU_t^{(1)} \big)^{-1}, \big(\bU^{(2)\top}_t \bU_t^{(2)} \big)^{-1}, \big(\bU^{(3)\top}_t \bU_t^{(3)}\big)^{-1} \right) \bcdot \nabla_{\bcG_t} \cL(\bF_t , \bcS_{t+1}) \nonumber \\
%    &= \bcG_t - \eta \left( \big(\bU^{(1)\top}_t \bU_t^{(1)}\big)^{-1} \bU_t^{(1)\top}, \big(\bU^{(2)\top}_t \bU_t^{(2)}\big)^{-1} \bU_t^{(2)\top}, \big(\bU^{(3)\top}_t \bU_t^{(3)}\big)^{-1} \bU_t^{(3)\top} \right)  \nonumber \\
%    & \qquad \qquad\qquad \bcdot \left( \big(\bU_t^{(1)}, \bU_t^{(2)}, \bU_t^{(3)} \big) \bcdot \bcG_t + \bcS_{t+1} - \bcY \right) \nonumber\\
    &= (1-\eta) \bcG_t - \eta\left( \big(\bU^{(1)\top}_t \bU_t^{(1)}\big)^{-1} \bU_t^{(1)\top}, \big(\bU^{(2)\top}_t \bU_t^{(2)}\big)^{-1} \bU_t^{(2)\top}, \big(\bU^{(3)\top}_t \bU_t^{(3)}\big)^{-1} \bU_t^{(3)\top} \right)  \bcdot \left(\bcS_{t+1} - \bcY \right). \label{g_update3}
\end{align}
\end{subequations}
Here, $\eta >0$ is the learning rate, and
\begin{equation*}
\Breve{\bU}_t^{(1)}:= \big(\bU_t^{(3)} \otimes \bU_t^{(2)} \big) \Matricize{1}{\bcG_t}^{\top}, \quad \Breve{\bU}_t^{(2)}:= \big(\bU_t^{(3)} \otimes \bU_t^{(1)} \big) \Matricize{2}{\bcG_t}^{\top},\; \mbox{and} \quad \Breve{\bU}_t^{(3)}:= \big(\bU_t^{(2)} \otimes \bU_t^{(1)} \big) \Matricize{3}{\bcG_t}^{\top}.
\end{equation*}

To complete the algorithm description, we still need to specify how to initialize the algorithm. We will estimate the tensor factors via the spectral method, by computing the HOSVD of the observation after applying the soft-shrinkage operator:
$$ \big (\bU_0^{(1)}, \bU_0^{(2)}, \bU_0^{(3)}, \bcG_0 \big) = \hosvd{\br}{\bcY -\bcS_0}, \qquad \mbox{where} \quad \bcS_0 =  \Shrink{\zeta_0}{\bcY}. $$
Altogether, we arrive at Algorithm \ref{alg:tensor_RPCA}, which we still dub as ScaledGD for simplicity.

\begin{algorithm}[t]
\caption{ScaledGD for tensor robust principal component analysis}\label{alg:tensor_RPCA} 
\begin{algorithmic} 
\STATE \textbf{Input:} the observed tensor $\bcY$, the multilinear rank $\br$, learning rate $\eta$, and threshold schedule $\{\zeta_t\}_{t=0}^{T}$. 
\STATE \textbf{Initialization:} $\bcS_0 = \Shrink{\zeta_0}{\bcY}$ and $\big (\bU_0^{(1)}, \bU_0^{(2)}, \bU_0^{(3)}, \bcG_0 \big) = \hosvd{\br}{\bcY -\bcS_0}$.
\FOR{$t = 0, 1, \dots, T-1$}
    \STATE Update the corruption tensor $\bcS_{t+1}$ via \eqref{eq:update_corruption};
%    $$\bcS_{t+1} = \Shrink{\zeta_{t+1}}{\bcY - \big(\bU_t^{(1)}, \bU_t^{(2)}, \bU_t^{(3)} \big) \bcdot \bcG_t};$$
    \STATE Update the tensor factors $\bF_{t+1} = \big( \bU^{(1)}_{t+1},  \bU^{(2)}_{t+1},  \bU^{(3)}_{t+1},  \bcG_{t+1} \big)$ via \eqref{a_update3} and \eqref{g_update3};
    %\FOR{$k= 1, 2, 3$}
        % \State $\Breve{\bU}_t^{(k)} \gets (\bU^{(N)}_t \otimes \dots \otimes \bU^{(k+1)}_t \otimes \bU^{(k - 1)}_t \otimes \dots \otimes \bU^{(1)}_t) \Matricize{k}{\bcG_t}^\top$
%        \begin{align*} 
%         \bU^{(k)}_{t+1} & = (1 - \eta) \bU_t^{(k)} - \eta \Matricize{k}{\bcS_{t+1} - \bcY}\Breve{\bU}_t^{(k)} \big( \Breve{\bU}_t^{(k)\top} \Breve{\bU}_t^{(k)} \big)^{-1}, \qquad k=1,2,3;\\
%  %  \ENDFOR
%      \bcG_{t+1}  & = (1-\eta) \bcG_t - \eta \left( \big( \bU^{(1)\top}_t \bU_t^{(1)} \big)^{-1} \bU_t^{(1)\top}, \big( \bU^{(2)\top}_t \bU_t^{(2)} \big)^{-1} \bU_t^{(2)\top}, \big( \bU^{(3)\top}_t \bU_t^{(3)} \big)^{-1} \bU_t^{(3)\top}\right) \bcdot \left(\bcS_{t+1} - \bcY \right).
%      \end{align*}
%      where $\Breve{\bU}_t^{(1)}:= \big(\bU_t^{(3)} \otimes \bU_t^{(2)} \big) \Matricize{1}{\bcG_t}^{\top}$, $\Breve{\bU}_t^{(2)}:= \big(\bU_t^{(3)} \otimes \bU_t^{(1)} \big) \Matricize{2}{\bcG_t}^{\top}$, and $\Breve{\bU}_t^{(3)}:= \big(\bU_t^{(2)} \otimes \bU_t^{(1)} \big) \Matricize{3}{\bcG_t}^{\top}$.
\ENDFOR
\STATE \textbf{Output:} the tensor factors $ \bF_T = \big (\bU_{T}^{(1)}, \bU_T^{(2)}, \bU_T^{(3)}, \bcG_T \big) $.
\end{algorithmic} 
\end{algorithm}

%Un-matricizing (folding) the gradient, we will get $\nabla_{\bcG_t} \cL(\bF_t)$.
%\begin{align*}
%    \nabla_{\bcG_t} \cL(\bF_t) &= ((\bU_t^{(1)})^\top \bU_t^{(1)}, (\bU_t^{(2)})^\top \bU_t^{(2)}, (\bU_t^{(3)})^\top \bU_t^{(3)}) \bcdot \bcG_t + ((\bU_t^{(1)})^\top, (\bU_t^{(2)})^\top, (\bU_t^{(3)})^\top) \bcdot (\bcS_{t+1} - \bcY) \\
%    &= ((\bU_t^{(1)})^\top, (\bU_t^{(2)})^\top, (\bU_t^{(3)})^\top) \bcdot \left( ( \bU_t^{(1)}, \bU_t^{(2)}, \bU_t^{(3)}) \bcdot \bcG_t +  (\bcS_{t+1} - \bcY) \right)
%\end{align*}

%\begin{align*}
%    \nabla_{\bU_t^{(k)}} \cL(\bF_t) &= \frac{1}{2} \nabla_{\bU_t^{(k)}} \norm{\Matricize{k}{(\bU_t^{(1)}, \bU_t^{(2)}, \bU_t^{(3)}) \bcdot \bcG_t + \bcS_{t+1} - \bcY}}_{\fro}^2 \\
%    % &= \frac{1}{2} \nabla_{\bU_t^{(k)}} \norm{\bU_t^{(k)} \Matricize{k}{\bcG_t}(\bU_t^{(N)} \otimes \dots \otimes \bU_t^{(k+1)} \otimes \bU_t^{(k-1)} \otimes \dots \otimes \bU_t^{(1)})^\top + \Matricize{k}{\bcS_{t+1}} -  \Matricize{k}{\bcY}}_{\fro}^2 \\
%    &= \frac{1}{2} \nabla_{\bU_t^{(k)}} \norm{\bU_t^{(k)} (\Breve{\bU}_t^{(k)})^\top + \Matricize{k}{\bcS_{t+1}} -  \Matricize{k}{\bcY}}_{\fro}^2 \\
%    &= \frac{1}{2} \nabla_{\bU_t^{(k)}} \norm{\Breve{\bU}_t^{(k)} (\bU_t^{(k)})^\top  + \Matricize{k}{\bcS_{t+1}}^\top -  \Matricize{k}{\bcY}^\top}_{\fro}^2 \\
%    &= \left[(\Breve{\bU}_t^{(k)})^\top \left( \Breve{\bU}_t^{(k)} (\bU_t^{(k)})^\top  + \Matricize{k}{\bcS_{t+1}}^\top -  \Matricize{k}{\bcY}^\top\right)\right]^\top \\
%    &= \left( \bU_t^{(k)} (\Breve{\bU}_t^{(k)})^\top   + \Matricize{k}{\bcS_{t+1}} -  \Matricize{k}{\bcY}\right) \Breve{\bU}_t^{(k)}
%\end{align*}

\paragraph{Computational benefits.} It is worth highlighting that the proposed tensor RPCA algorithm possesses several computational benefits which might be of interest in applications.
\begin{itemize}

\item {\em Advantages over matrix RPCA algorithms.} While it is possible to matricize the input tensor and then apply the matrix RPCA algorithms, they can only exploit the low-rank structure along the mode that the tensor is unfolded, rather than along multiple rows simultaneously as in the tensor RPCA algorithm. In addition, the space complexity of storing and computing the factors is much higher for the matrix RPCA algorithms, where the size of the factors become multiplicative in terms of the tensor dimensions due to unfolding, rather than linear as in the tensor RPCA algorithm.
\item {\em Generalization to $N$-th order tensors.} Although the description of Algorithm \ref{alg:tensor_RPCA} is tailored to an order-$3$ tensor, our algorithm is easily generalizable to any $N$-th order tensor; in fact, Algorithm~\ref{alg:tensor_RPCA} can be applied almost verbatim by redefining
$$\Breve{\bU}_t^{(k)} = \big(\bU_t^{(N)} \otimes \cdots \otimes \bU_t^{(k+1)} \otimes \bU_t^{(k-1)} \otimes \cdots \otimes \bU_t^{(1)} \big) \Matricize{k}{\bcG_t}^\top, \qquad k=1,\ldots, N $$
to its natural high-order counterpart. This extension is numerically evaluated in our experiments in Section~\ref{sec:numerical}.

\item {\em Parallelizability.} At each iteration of the proposed algorithm, each tensor factor is updated independently as done in \eqref{a_update3} and \eqref{g_update3}, therefore we can update them in a parallel manner. This improvement becomes more apparent as the order of the tensor increases. %While the left and right matrix in the matrix algorithm can be parallelized, matricization of the input matrix will result in the update of the right matrix having enormous matrix operations.

    \item \textit{Selective modes to update}: If we know the underlying ground truth tensor is only low-rank along certain mode, we can choose to skip the iterative updates of the rest of the modes after initialization to reduce computational costs, which we demonstrate empirically in Section \ref{background_subtraction}. %While the theorems presented do not apply to this refinement, this works empirically, as demonstrated in Section \ref{background_subtraction}.
    
\end{itemize}

\subsection{Performance guarantees}

Motivated by the analysis in \cite{tong2021scaling}, we consider the following distance metric, which not only resolves the ambiguity in the Tucker decomposition, but also takes the preconditioning factor into consideration. 
\begin{definition}[Distance metric]
Letting $\bF := \big(\bU^{(1)}, \bU^{(2)}, \bU^{(3)}, \bcG \big)$ and $\bF_\star := \big(\bU_\star^{(1)}, \bU_\star^{(2)}, \bU_\star^{(3)}, \bcG_\star \big)$, denote
\begin{equation}\label{dist}
    \dist^2(\bF, \bF_\star) := \inf_{\bQ^{(k)} \in \GL(r_k)} \sum_{k=1}^{3} \norm{ \big(\bU^{(k)} \bQ^{(k)} - \bU^{(k)}_\star \big) \bSigma_\star^{(k)}}_{\fro}^2 + \norm{\big( (\bQ^{(1)})^{-1}, (\bQ^{(2)})^{-1}, (\bQ^{(3)})^{-1} \big) \bcdot \bcG- \bcG_\star}_{\fro}^2,
\end{equation}
where we recall $\bSigma_\star^{(k)}$ is the singular value matrix of $\Matricize{k}{\bcX_{\star}}$, $k=1,2,3$. Moreover, if the infimum is attained at the arguments $\{\bQ^{(k)}\}_{k=1}^3$, they are called the optimal alignment matrices between $\bF$ and $\bF_{\star}$.  
\end{definition}

%If $\dist(\bF, \bF_\star) < \sigma_{\min}(\bcX_\star)$, then the infimum is attained \cite[Lemma 6]{tong2021scaling}.

Fortunately, the proposed ScaledGD algorithm (cf.~Algorithm \ref{alg:tensor_RPCA}) provably recovers the ground truth tensor---as long as the fraction of corruptions is not too large---with proper choices of the tuning parameters, as captured in following theorem.
\begin{theorem} \label{main}
Let $\bcY = \bcX_\star + \bcS_\star \in \RR^{n_1 \times n_2 \times n_3}$, where $\bcX_\star$ is $\mu$-incoherent with multilinear rank $\br = (r_1, r_2, r_3)$, and $\bcS_\star$ is $\alpha$-sparse. Suppose that the threshold values $\{\zeta_k\}_{k=0}^\infty$ obey that $  \norm{\bcX_\star}_\infty \leq \zeta_0  \leq 2     \norm{\bcX_\star}_\infty$ and  $\zeta_{t+1} = \rho \zeta_{t}$, $t\geq 1$, for some properly tuned $\zeta_1: =  8 \sqrt{\frac{\mu^3 r_1 r_2 r_3}{n_1 n_2 n_3}}   \sigma_{\min}(\bcX_\star)
$ and $\frac{1}{7} \leq \eta \leq \frac{1}{4}$, where $\rho = 1-0.45\eta$. Then,  the iterates $\bcX_t =  \big(\bU_t^{(1)}, \bU_t^{(2)}, \bU_t^{(3)} \big) \bcdot \bcG_t$ satisfy
\begin{subequations} \label{eq:thm_claims}
\begin{align}
    \norm{ \bcX_t - \bcX_\star}_{\fro} &\leq 0.03 \rho^t  \sigma_{\min}(\bcX_\star) , \label{eq:fro_contraction}  \\ 
 \norm{ \bcX_t - \bcX_\star}_{\infty}  & \leq 8 \rho^t \sqrt{\frac{\mu^3 r_1 r_2 r_3}{n_1 n_2 n_3}}   \sigma_{\min}(\bcX_\star), \label{eq:entry_contraction} \\
 \norm{ \bcS_t - \bcS_\star}_{\infty}  & \leq 16 \rho^{t-1} \sqrt{\frac{\mu^3 r_1 r_2 r_3}{n_1 n_2 n_3}}   \sigma_{\min}(\bcX_\star) \label{eq:sparse_contraction}
\end{align}
\end{subequations}
for all $t\geq 0$, as long as  the level of corruptions obeys
 $\alpha \leq \frac{c_0}{\mu^2 r_1 r_2 r_3 \kappa}$
 for some sufficiently small $c_0>0$.
\end{theorem}

The value of $\rho$ was selected to simplify the proof and should not be taken as an optimal convergence rate.
In a nutshell, Theorem \ref{main} has the following immediate consequences: 
\begin{itemize}
\item \textbf{Exact recovery.} Upon appropriate choices of the parameters, if the level of corruptions $\alpha$ is small enough, i.e. not exceeding the order of $\frac{1}{\mu^2 r_1 r_2 r_3 \kappa}$, we can ensure that the proposed Algorithm \ref{alg:tensor_RPCA} exactly recovers the ground truth tensor $\bcX_\star$ even when the gross corruptions are arbitrary and adversarial. As mentioned earlier, our result significantly enlarges the range of allowable corruption levels for exact recovery when the outliers are evenly distributed across the fibers, compared with the prior art established in \cite{cai2021generalized}. 

\item \textbf{Constant linear rate of convergence.} The proposed ScaledGD algorithm (cf.~Algorithm \ref{alg:tensor_RPCA}) finds the ground truth tensor at a {\em constant} linear rate, which is independent of the condition number, from a carefully designed spectral initialization. Consequently, the proposed ScaledGD algorithm inherits the computational robustness against ill-conditioning as \cite{tong2021scaling}, even in the presence of gross outliers, as long as the thresholding operations are properly carried out.

\item \textbf{Refined entrywise error guarantees.} Furthermore, when $\mu=O(1)$ and $r =O(1)$, the entrywise error bound \eqref{eq:entry_contraction}---which is smaller than the Frobenius error \eqref{eq:fro_contraction} by a factor of $\sqrt{\frac{1}{n_1n_2n_3}}$---suggests the errors are distributed in an evenly manner across the entries for incoherent and low-rank tensors. The same applies to the entrywise error bound of the sparse tensor \eqref{eq:sparse_contraction} which exhibits similar behavior as \eqref{eq:entry_contraction}. To the best of our knowledge, this is the first time such a refined entrywise error analysis is established for tensor RPCA.
\end{itemize}

%the upper bound on distance to the ground truth and the data matrix error linearly converge to $0$ for appropriate choices of the step size $\eta$ and thresholding parameters, $\zeta_0$ and $\zeta_1$.This is similar to the matrix RPCA guarantee in \cite[Theorem 1]{cai2021learned}, which shows the existence of hyperparameters $\{\zeta_k\}_{k=0}^\infty$ such that recovery is guaranteed. In contrast, Theorem \ref{main} essentially guarantees recovery with just appropriately tuned $\zeta_0$ and $\zeta_1$.
%Later in this paper, we will show that if we choose our $\zeta_0$ and $\zeta_1$ carefully, we do not need to tune $\{\zeta_t\}_{t=2}^\infty$.

% 
%\begin{remark}
%\yc{to be merged as a comparison/discussion about algorithm design.}
% \cite{cai2021learned} also presents a variant of deep unfolding methods to tune the thresholding parameters $\zeta_t$ and step sizes for every iteration.
%However, while they guarantee convergence for a specific choice of $\{\zeta_t\}_{k=0}^\infty$, this choice is not implementable since this choice depends on knowledge of the distance between our estimated data matrix and the ground truth.
%For our tensor RPCA algorithm, we will make linear convergence more implementable by not only showing a range of values $\{\zeta_t\}_{k=0}^\infty$ can take and still guarantee linear convergence, but also demonstrating that one can just tune $\zeta_0$ and $\zeta_1$ instead of an infinite number of parameters for this guarantee to hold.
%\end{remark}

\section{Outline of the analysis}

In this section, we outline the proof of Theorem \ref{main}. The proof is inductive in nature, where we aim to establish the following induction hypothesis at all the iterations: 
\begin{subequations} \label{eq:induction_hypothesis}
\begin{align}
    \dist(\bF_t, \bF_\star) &\leq \epsilon_0 \rho^t  \sigma_{\min}(\bcX_\star) , \label{eq:dist_contraction}  \\ 
    \max_{k} \left\{\sqrt{\frac{n_k}{r_k}} \norm{(\bU_t^{(k)} \bQ_t^{(k)} - \bU_\star^{(k)}) \bSigma_\star^{(k)}}_{2, \infty} \right\} & \leq \rho^t \sqrt{\mu} \sigma_{\min}(\bcX_\star) ,\label{eq:incoh_contraction} 
\end{align}
\end{subequations}
where $\rho =  1-0.45\eta $, $\epsilon_0<0.01$ is some sufficiently small constant, and $\{\bQ_t^{(k)}\}_{k=1}^3$ are the optimal alignment matrices between $\bF_t$ and $\bF_{\star}$. The claims \eqref{eq:thm_claims} in Theorem~\ref{main} follow immediately with the aid of Lemma~\ref{delta_x_inf_bound_contract_incoherence} and Lemma~\ref{more_perturb} (see Appendix~\ref{sec:technical_lemmas}). The following set of lemmas, whose proofs are deferred to Appendix~\ref{sec:proof_main_lemmas}, establishes the induction hypothesis \eqref{eq:induction_hypothesis} for both the induction case and the base case.

\paragraph{Induction: local contraction.} We start by outlining the local contraction of the proposed Algorithm \ref{alg:tensor_RPCA}, by establishing the induction hypothesis \eqref{eq:induction_hypothesis} continues to hold at the $(t+1)$-th iteration, assuming it holds at the $t$-th iteration, as long as the corruption level is not too large.

\begin{lemma}[Distance contraction] \label{dist_contract_incoherence_contract}
Let $\bcY = \bcX_\star + \bcS_\star \in \RR^{n_1 \times n_2 \times n_3}$, where $\bcX_\star$ is $\mu$-incoherent with multilinear rank $\br = (r_1, r_2, r_3)$, and $\bcS_\star$ is $\alpha$-sparse. 
Let $\bF_t := \big(\bU_t^{(1)}, \bU_t^{(2)}, \bU_t^{(3)}, \bcG_t \big)$ be the $t$-th iterate of Algorithm \ref{alg:tensor_RPCA}. Suppose that the induction hypothesis \eqref{eq:induction_hypothesis} holds at the $t$-th iteration.
Under the assumption $\alpha \leq \frac{c_0 \epsilon_0}{ \sqrt{\mu^{3} r_1 r_2 r_3 r}}$ for some sufficiently small constant $c_0$ and the choice of $\zeta_{t+1}$ in Theorem~\ref{main},
%Define $\bcX_t = (\bU_t^{(1)}, \bU_t^{(2)}, \bU_t^{(3)}) \bcdot \bcG_t$
%Choosing $\zeta_{t+1}: = \rho \zeta_t$ such that
%\begin{align*}
%    \norm{\bcX_t - \bcX_\star}_\infty \leq \zeta_{t+1} \leq \sqrt{\frac{\mu^3 r_1 r_2 r_3}{n_1 n_2 n_3}}  \left[8 \epsilon_0  + 7 \right] \rho^t \sigma_{\min}(\bcX_\star)     
%\end{align*} 
%if $\dist(\bF_t, \bF_\star) \leq \epsilon_0 \rho^t \sigma_{\min}(\bcX_\star)$ where $\epsilon_0 := 0.01 / C$ for some constant $C \geq 1$ and $\bF_t$ satisfies the incoherence condition
%\begin{align*}
%    \max_{k} \left\{\sqrt{\frac{n_k}{r_k}} \norm{(\bU_t^{(k)} \bQ_t^{(k)} - \bU_\star^{(k)}) \bSigma_\star^{(k)}}_{2, \infty} \right\} \leq \rho^t \sqrt{\mu} \sigma_{\min}(\bcX_\star)
%\end{align*}
the $(t+1)$-th iterate  $\bF_{t+1}$ satisfies
\begin{align*}
    \dist(\bF_{t+1}, \bF_\star) &\leq  \epsilon_0 \rho^{t+1} \sigma_{\min}(\bcX_\star) 
\end{align*}
as long as $\eta \leq 1/4$.
\end{lemma}

While Lemma \ref{dist_contract_incoherence_contract} guarantees the contraction of the distance metric, the next Lemma~\ref{incoherence_contract} establishes the contraction of the incoherence metric, so that we can repeatedly apply Lemma \ref{dist_contract_incoherence_contract} and Lemma~\ref{incoherence_contract} for induction.

\begin{lemma}[Incoherence contraction]
\label{incoherence_contract}
Let $\bcY = \bcX_\star + \bcS_\star \in \RR^{n_1 \times n_2 \times n_3}$, where $\bcX_\star$ is $\mu$-incoherent with multilinear rank $\br = (r_1, r_2, r_3)$, and $\bcS_\star$ is $\alpha$-sparse. 
Let $\bF_t := \big(\bU_t^{(1)}, \bU_t^{(2)}, \bU_t^{(3)}, \bcG_t \big)$ be the $t$-th iterate of Algorithm \ref{alg:tensor_RPCA}. Suppose that the induction hypothesis \eqref{eq:induction_hypothesis} holds at the $t$-th iteration. Under the assumption that $\alpha \leq \frac{c_1 }{ \mu^2 r_1 r_2 r_3}$ for some sufficiently small constant $c_1$ and the choice of $\zeta_{t+1}$ in Theorem~\ref{main},
the $(t+1)$-th iterate  $\bF_{t+1}$ satisfies
\begin{align*}
    \max_{k} \left\{\sqrt{\frac{n_k}{r_k}} \norm{ \big(\bU_{t+1}^{(k)} \bQ_{t+1}^{(k)} - \bU_\star^{(k)} \big) \bSigma_\star^{(k)}}_{2, \infty} \right\} \leq \rho^{t+1} \sqrt{\mu} \sigma_{\min}(\bcX_\star)
\end{align*}
as long as $1/7\leq \eta \leq 1/4$, where $\{\bQ_t^{(k)}\}_{k=1}^3$ are the optimal alignment matrices between $\bF_t$ and $\bF_{\star}$.
\end{lemma}

It is worthwhile to note that, the local linear convergence of the proposed Algorithm \ref{alg:tensor_RPCA}, as ensured by the above two lemmas collectively, require the corruption level to not exceed the order of $\frac{1}{\mu^2 r_1 r_2 r_3}$, which is also independent of the condition number. Indeed, the range of the corruption level is mainly constrained by the spectral initialization, as demonstrated next.

\paragraph{Base case: spectral initialization.} To establish the induction hypothesis, we still need to check the spectral initialization. 
The following lemmas state that the spectral initialization satisfies the induction hypothesis \eqref{eq:induction_hypothesis} at the base case $t=0$,  allowing us to invoke local contraction.

\begin{lemma}[Distance at initialization]
\label{dist_init}
Let $\bcY = \bcX_\star + \bcS_\star \in \RR^{n_1 \times n_2 \times n_3}$, where $\bcX_\star$ is $\mu$-incoherent with rank $\br = (r_1, r_2, r_3)$, and $\bcS_\star$ is $\alpha$-sparse.
Let $\bF_0 := \big(\bU_0^{(1)}, \bU_0^{(2)}, \bU_0^{(3)}, \bcG_0 \big)$ be the output of spectral initialization with the threshold obeying $\norm{\bcX_\star}_\infty  \leq \zeta_0 \leq     2\norm{\bcX_\star}_\infty $. If $\alpha \leq \frac{c_0}{\sqrt{\mu^3 r_1 r_2 r_3 r }  \kappa}$ for some constant $c_0>0$,  we have
\begin{align*}
    \dist(\bF_0, \bF_\star) &\leq 54.1 c_0 \sigma_{\min}(\bcX_\star) .
\end{align*}
\end{lemma}

Lastly, the next lemma ensures that our initialization satisfies the incoherence condition, which requires nontrivial efforts to exploit the algebraic structures of the Tucker decomposition.

\begin{lemma}[Incoherence at initialization]
\label{incoherence_init}
Let $\bcY = \bcX_\star + \bcS_\star \in \RR^{n_1 \times n_2 \times n_3}$, where $\bcX_\star$ is $\mu$-incoherent with rank $\br = (r_1, r_2, r_3)$, and $\bcS_\star$ is $\alpha$-sparse.
Let $\bF_0 := (\bU_0^{(1)}, \bU_0^{(2)}, \bU_0^{(3)}, \bcG_0)$ be the output of spectral initialization with the threshold obeying $ \norm{\bcX_\star}_\infty  \leq \zeta_0 \leq     2\norm{\bcX_\star}_\infty $.
If $\alpha \leq \frac{c_0}{\mu^2 r_1 r_2 r_3  \kappa}$ for some sufficiently small constant $c_0$, then the spectral initialization satisfies the incoherence condition
\begin{align*}
    \max_k \left\{\sqrt{\frac{n_k}{r_k}} \norm{\big(\bU_0^{(k)} \bQ_0^{(k)} - \bU_\star^{(k)} \big) \bSigma_\star^{(k)}}_{2, \infty} \right\} \leq \sqrt{\mu} \sigma_{\min}(\bcX_\star) ,
\end{align*}
where $\{\bQ_0^{(k)}\}_{k=1}^3$ are the optimal alignment matrices between $\bF_0$ and $\bF_{\star}$.
\end{lemma}

 \section{Numerical experiments}  
\label{sec:numerical}

\subsection{Experiments on synthetic data}

%\subsubsection{Rank, sparsity, and condition number exploration}
 
 We begin with evaluating the phase transition performance of ScaledGD (cf.~Algorithm~\ref{alg:tensor_RPCA}) with respect to the multilinear rank and the level of corruption. For each $\kappa$, we randomly generate an $n \times n \times n$ tensor, with $n=100$, 
multilinear rank $\br = (r, r, r)$, $r \in \{2, 5, 10, 20, \dots, 80\}$, and level of corruption $\alpha \in \{0.1, 0.2, \dots, 0.9, 1\}$. The factor matrices are generated uniformly at random with orthonormal columns, and a diagonal core tensor $\bcG_\star$ is generated such that $[\bcG_\star]_{i, i, i} = \kappa^{-(i-1)/(r-1)}$ for $i =1, 2, \dots, r$. We further randomly corrupt $\alpha$-fraction of the entries, by 
 adding uniformly sampled numbers from the range $[- \sum_{i,j,k} |[\bcX_\star]_{i, j, k}| /n^3,   \sum_{i,j,k} |[\bcX_\star]_{i, j, k}| /n^3 ]$ to the selected entries, where $\sum_{i,j,k} |[\bcX_\star]_{i, j, k}| /n^3$ is the mean of the entry-wise magnitudes of $\bcX_\star$.  
To tune the constant step size $\eta$, and the hyperparameters $\zeta_0$, $\zeta_1$, and the decay rate $\rho$ of the thresholding parameter for each tensor automatically, we used the Bayesian optimization method described in \cite{akiba2019optuna}. Specifically, we run the toolbox \cite{akiba2019optuna} for 200 trials or until the tuned parameters satisfy $\frac{\norm{\bcX_T - \bcX_\star}_{\fro}}{ \norm{\bcX_\star}_{\fro}} < 10^{-6}$ for $T=200$, whichever happened first.
%The final hyperparameters were used to find $\frac{\norm{\bcX_T - \bcX_\star}_{\fro}}{ \norm{\bcX_\star}_{\fro}}$, and the log median of these values over the 20 tensors is shown in Figure 

Figure~\ref{phase_median} shows the log median of the relative reconstruction error $\frac{\norm{\bcX_T - \bcX_\star}_{\fro}}{ \norm{\bcX_\star}_{\fro}}$ when $T=200$, over 20 random tensor realizations for $\kappa =1,5,10$. Our results show a distinct negative linear relationship between the corruption level and the multilinear rank with regards to the final relative loss of ScaledGD.  In particular, the performance is almost independent of the condition number $\kappa$, suggesting the performance of  ScaledGD is indeed quite insensitive to the condition number.
 
\begin{figure}[ht]
\centering
\includegraphics[width=0.85\textwidth]{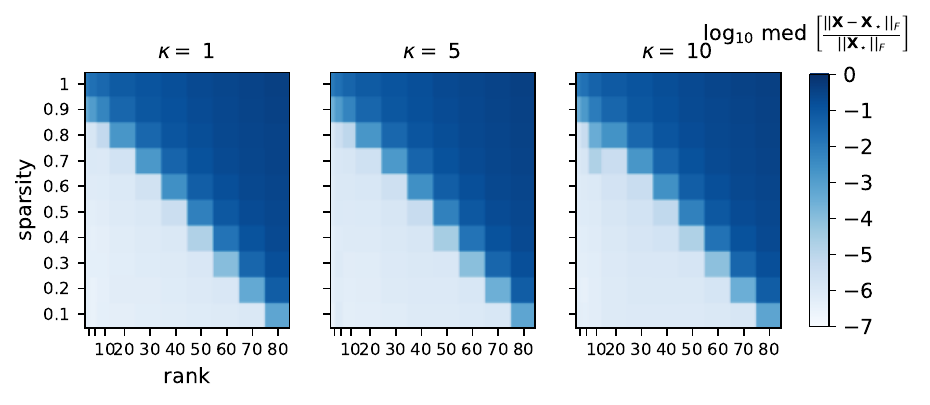}
\caption{Log median of the relative reconstruction error of $\frac{\norm{\bcX_T - \bcX_\star}_{\fro}}{ \norm{\bcX_\star}_{\fro}}$ across 20 randomly generated tensors with varying ranks and levels of corruption when the condition number is set as $\kappa=1,5,10$.  } %The first two ticks on the x-axis are ranks 2 and 5, respectively. 
\label{phase_median}
\end{figure}

We further investigate the effect of the decay rate $\rho$ of the thresholding parameters while fixing the other hyperparameters tuned as earlier. Using the same method, we generate  a $100 \times 100 \times 100$ tensor with $\kappa = 5$ and 20\% of the entries corrupted. Figure~\ref{varying_decay} shows the relative reconstruction error versus the iteration count using different decay rates $\rho$. It can be seen that ScaledGD enables exact recovery over a wide range of $\rho$ as long as it is not too small. Moreover, within the range of decay rates that still admits exact recovery, the smaller $\rho$ is, the faster ScaledGD converges. Note that the tuned decay rate $\rho \approx 0.931$ does not achieve the fastest convergence rate since the stopping criteria for hyperparameter tuning were not set to optimize the convergence rate but some accuracy-speed trade-off. 

\begin{figure}[ht]
\centering
\includegraphics[width=0.45\textwidth]{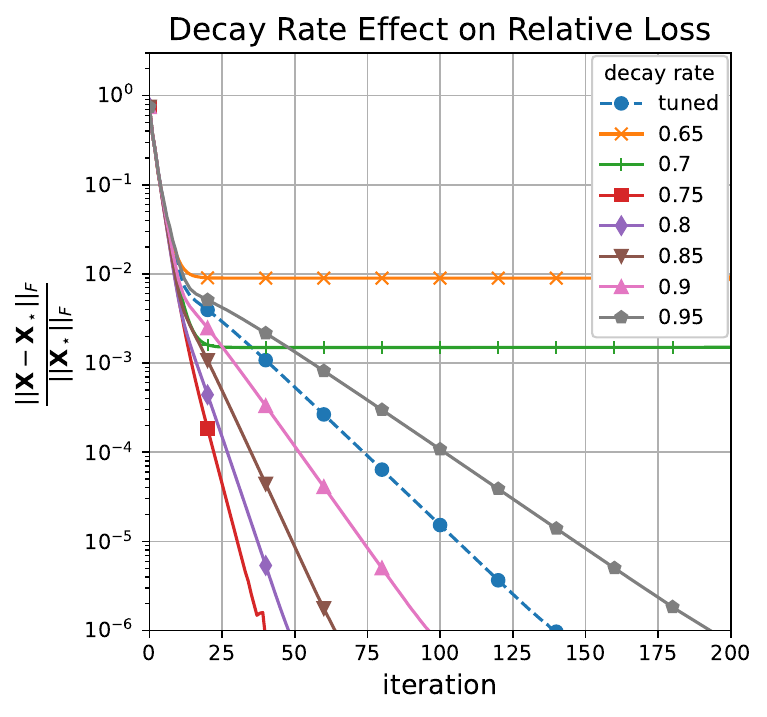}
\caption{The relative reconstruction error $\frac{\norm{\bcX_T - \bcX_\star}_{\fro}}{ \norm{\bcX_\star}_{\fro}}$ with respect to the iteration count, when varying the decay rate $\rho$ with other hyperparameters fixed.}
\label{varying_decay}
\end{figure}
 
% In addition, our algorithm incurs a much smaller per-iteration computation cost than the algorithm in \cite{cai2021generalized}, which requires the evaluation of a rank-$\br$ HOSVD.  
 
 Next, we also examine the performance of ScaledGD with shot noise, corruptions drawn from a Poisson distribution, with comparisons to the Riemannian gradient descent (RiemannianGD) algorithm in \cite{cai2021generalized}. More specifically, if corrupted, a ground truth entry $\bcX_{i, j, k}$ has noise drawn from $10^{-5} \textsf{Poisson}(10^5 |\bcX_{i, j, k}|)$ added to itself, where the $10^5$ scaling is to encourage draws that are nonzero. This type of noise is nonnegative and perturbs higher magnitude entries more.  
  Note that the per-iteration cost of RiemannianGD is significantly higher than ours, due to the fact that it requires an evaluation of a rank-$\br$ HOSVD. Therefore, 
  for this experiment, we generate tensors of a smaller size $50\times 50\times 50$ to accommodate the high computation need of  RiemannianGD. 
  We similarly tune the hyperparameters of RiemannianGD using the same method mentioned earlier for $100$ trials. Figure~\ref{fig:shot_comparison} shows the log median of the relative reconstruction error $\frac{\norm{\bcX_T - \bcX_\star}_{\fro}}{ \norm{\bcX_\star}_{\fro}}$ when $T=100$, over 20 random tensor realizations when $\kappa =5$. It can be observed that the empirical performance of the two methods, indicated by the phase transition curves, are similar when tuned properly. However, the ScaledGD method is considerably faster, the difference of which is accentuated for larger and lower rank tensors, due to the fact that it works in the factor space and does not need to perform rank-$\br$ HOSVD at every iteration.

% \cite{cai2021generalized}, which was tuned using the same method for $100$ trials. We refer to their algorithm as RiemannianGD for their use of Riemannian gradient descent. The reason for scaling down from our previous setup is the amount of compute required for the algorithm in \cite{cai2021generalized}, as it requires taking an HOSVD per iteration, a hefty process. The results are shown in Figure~\ref{fig:shot_comparison}. Like in Figure~\ref{phase_median}, we observe a similar rank-sparsity tradeoff for both models, though RiemannianGD outperforms ScaledGD at higher sparsity levels. However, we note that our algorithm is considerably faster, the difference of which is accentuated for larger and lower rank tensors. For example, our method is approximately 9.5 times faster when $r=40$ and approximately 30 times faster when $r=10$ for tensors of shape $(50, 50, 50)$. At the scale of $100 \times 100 \times 100$ tensors, our method is approximately 33 times faster when $r=10$.

\begin{figure}[ht]
    \begin{minipage}[b]{1\linewidth}
     \centering
     \begin{subfigure}[b]{.47\linewidth}
         \centering
		 \includegraphics[width=0.94\linewidth]{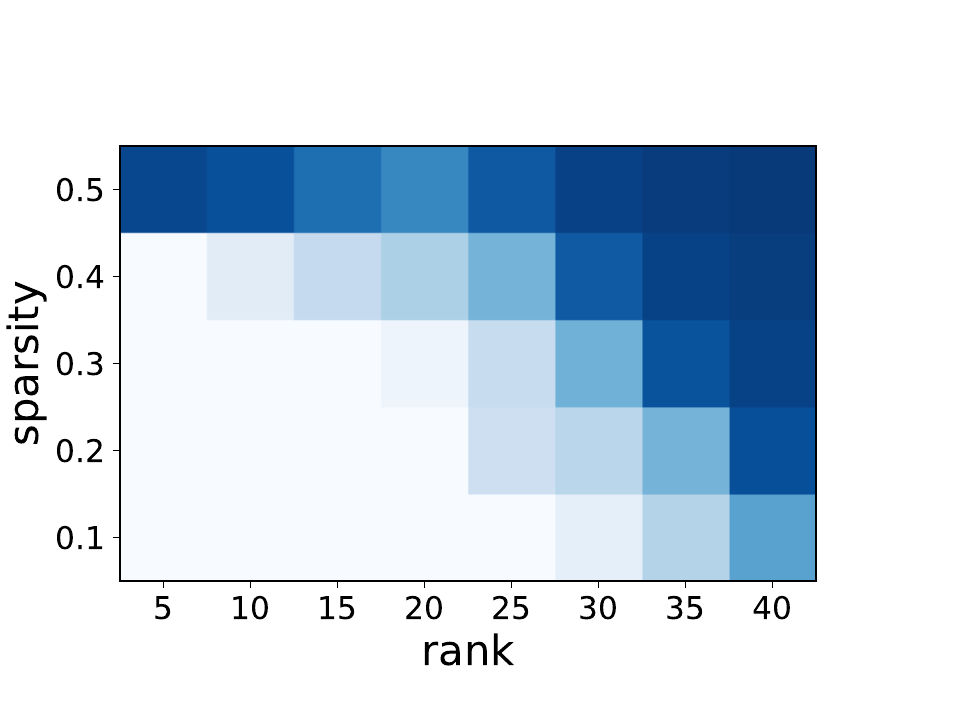}
         \caption{RiemannianGD}
         \label{fig:shot_riemannian}
     \end{subfigure}
     \begin{subfigure}[b]{.47\linewidth}
         \centering
         \includegraphics[width=0.95\linewidth]{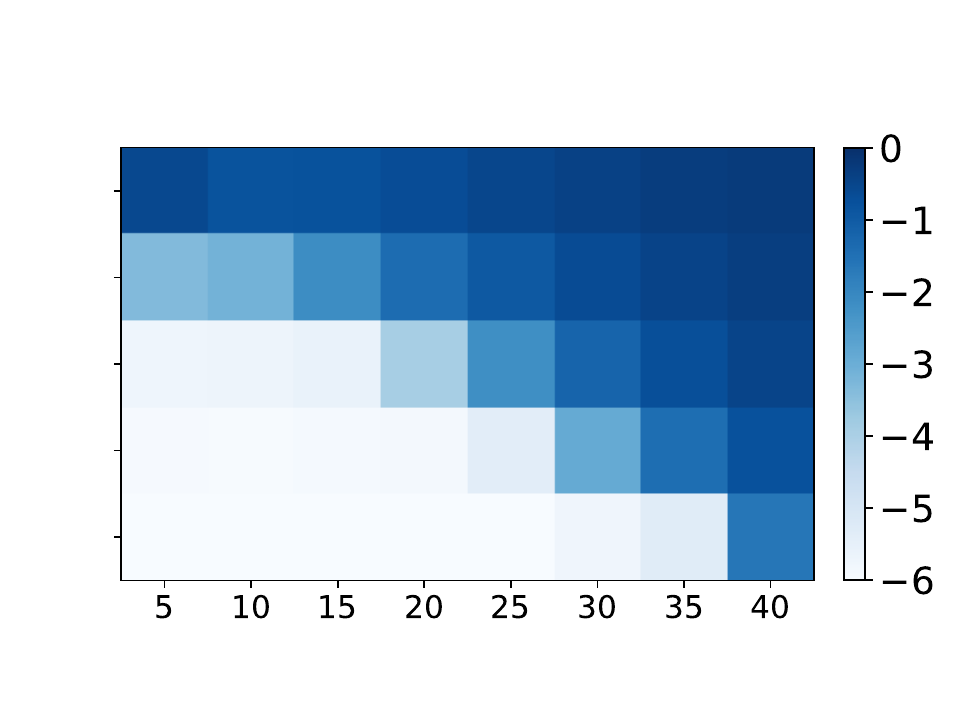}
         \caption{ScaledGD}
         \label{fig:shot}
     \end{subfigure}
     \end{minipage}
        \caption{Comparison between RiemannianGD \cite{cai2021generalized} and ScaledGD on the log median of the relative recovery error, $\frac{\norm{\bcX_\star - \bcX_T}_{\fro}}{\norm{\bcX_\star}_{\fro}}$, across 20 randomly generated tensors with varying ranks and levels of shot noise corruption when the condition number is set as $\kappa=5$.}
        \label{fig:shot_comparison}
\end{figure}

\subsection{Image denoising and outlier detection}
% \hd{try with much less data?}
In this experiment, we examine the performance of ScaledGD for imaging denoising and outlier detection, with comparisons to the tensor RPCA algorithm proposed in \cite{lu2019tensor} called TNN for their use of a newly defined tensor nuclear norm (TNN). We consider a sequence of handwritten digits ``2'' from the MNIST database \cite{lecun2010mnist} containing 5958 images of size $28 \times 28$, leading to a $3$-way tensor. We assume the tensor is low-rank along the image sequence, but not within the image for simplicity; in other words, the multilinear rank is assumed as $\bm{r} = (5, 28, 28)$. For both algorithms, the hyperparameters are best tuned by hands.

\begin{figure*}[ht]
\begin{subfigure}[b]{0.19\textwidth} 
\centering
    \includegraphics[width=\linewidth]{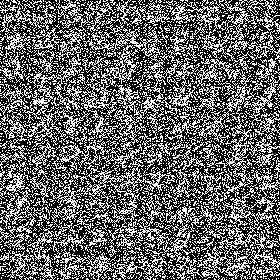}
    \includegraphics[width=\linewidth]{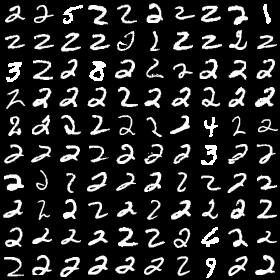}
     \includegraphics[width=\linewidth]{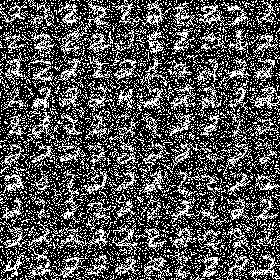}
    \caption{} 
\end{subfigure}
\hfill
\begin{subfigure}[b]{0.19\textwidth}
    \centering
    \includegraphics[width=\linewidth]{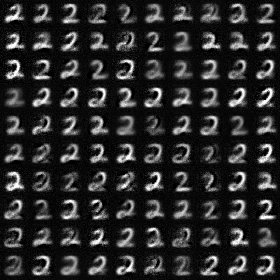}
    \includegraphics[width=\linewidth]{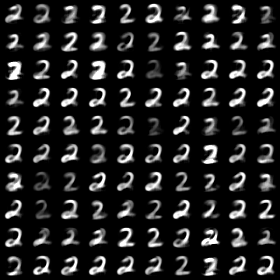}
     \includegraphics[width=\linewidth]{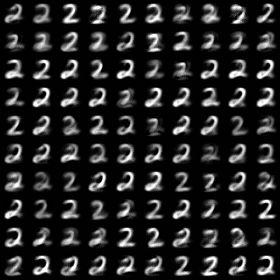}
      \caption{}  % \caption{low-rank component from ScaledGD}
\end{subfigure}
\hfill
\begin{subfigure}[b]{0.19\textwidth}
    \centering
    \includegraphics[width=\linewidth]{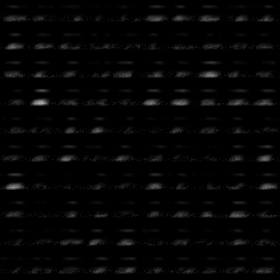}
    \includegraphics[width=\linewidth]{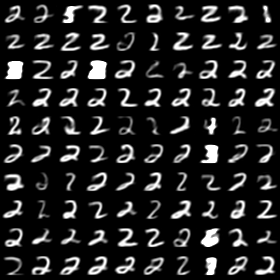}
     \includegraphics[width=\linewidth]{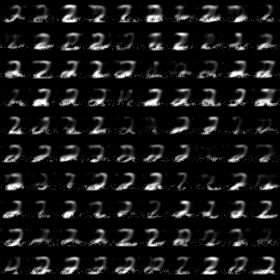}        \caption{} %\caption{low-rank component from TNN}
\end{subfigure}
\hfill
\begin{subfigure}[b]{0.19\textwidth}
    \centering
    \includegraphics[width=\linewidth]{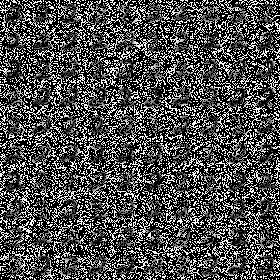}
     \includegraphics[width=\linewidth]{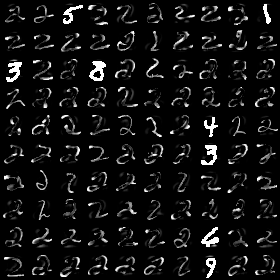}
      \includegraphics[width=\linewidth]{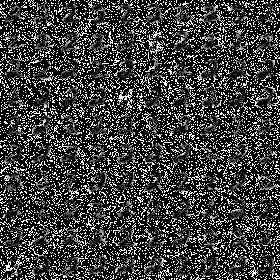}
       \caption{} % \caption{sparse component from ScaledGD}
\end{subfigure}
\hfill
\begin{subfigure}[b]{0.19\textwidth}
    \centering
    \includegraphics[width=\linewidth]{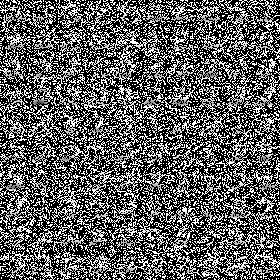}
     \includegraphics[width=\linewidth]{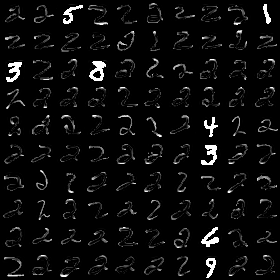}
      \includegraphics[width=\linewidth]{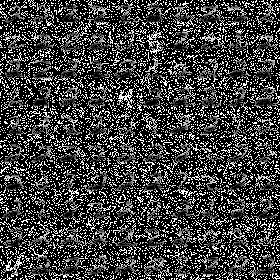}
        \caption{} %\caption{sparse component from TNN}
\end{subfigure}
\caption{Imaging denoising and outlier removal on an image sequence of handwritten digits with various corruption scenarios, using tensor RPCA via ScaledGD and TNN \cite{lu2019tensor}. From top to bottom, the rows show results on the first 100 images when 1) corrupted with 70\% salt and pepper noise; 2) 500 randomly swapped images; and 3) 50\% salt and pepper noise and 500 randomly swapped images. From left to right, the columns show (a) the corrupted input, (b) the low-rank output of ScaledGD, (c) the low-rank output of TNN, (d) the sparse output of ScaledGD, and (e) the sparse output of TNN.} 
\label{image_seq}
\end{figure*}

We examine the performance of ScaledGD and TNN when the image sequence is contaminated in the following scenarios: 1) 70\% salt and pepper noise; 2) 500 out of the total images are randomly selected and swapped by random images from the entire MNIST training set;  and 3) 50\% salt and pepper noise and 500 randomly swapped images. Figure~\ref{image_seq} demonstrates the performance of the compared algorithms on the first 100 instances for each corruption scenario. In all situations, ScaledGD recovers the low-rank component corresponding to the correct digit more accurately than TNN from a visual inspection. Furthermore, ScaledGD corrected the oddly-shaped or outlying digits to make the low-rank component be more homogeneous, but TNN mostly preserved these cases in the low-rank output. More importantly, ScaledGD runs much faster as a scalable nonconvex approach, while TNN is more computationally expensive using convex optimization.

\subsection{Background subtraction via selective mode update} \label{background_subtraction}
We now apply ScaledGD to the task of background subtraction using videos from the VIRAT dataset \cite{oh2011large}, where the height and width of the videos are reduced by a factor of 4 due to hardware limitations. The video data can be thought as a multi-way tensor spanning across the height, width, frames, as well as different color channels of the scene. Here, the low-rank tensor corresponds to the background in the video which is fairly static over the frames, and the sparse tensor corresponds to the foreground containing moving objects which takes a small number of pixels. In particular, it is reasonable to assume that the background tensor is low-rank for the mode corresponding to the frames, but full rank in other modes. Motivated by this observation, one might be tempted to selectively only update the core tensor and the factor matrix corresponding to the mode for frames while keeping the other factor matrices fixed after the spectral initialization. We compare the results using this selective mode update strategy with the original ScaledGD algorithm in Figure \ref{background1}, where the same hyperparameters are used in both. It can be seen that skipping updates of the full-rank factor matrices produced qualitatively similar results while gaining a significant per-iteration speed-up of about 4.6 to 5 times. We expect the speed improvement to be greater for larger tensors, as more computation can be bypassed.

\begin{figure*}[ht] 
\centering
\includegraphics[width=0.24\textwidth]{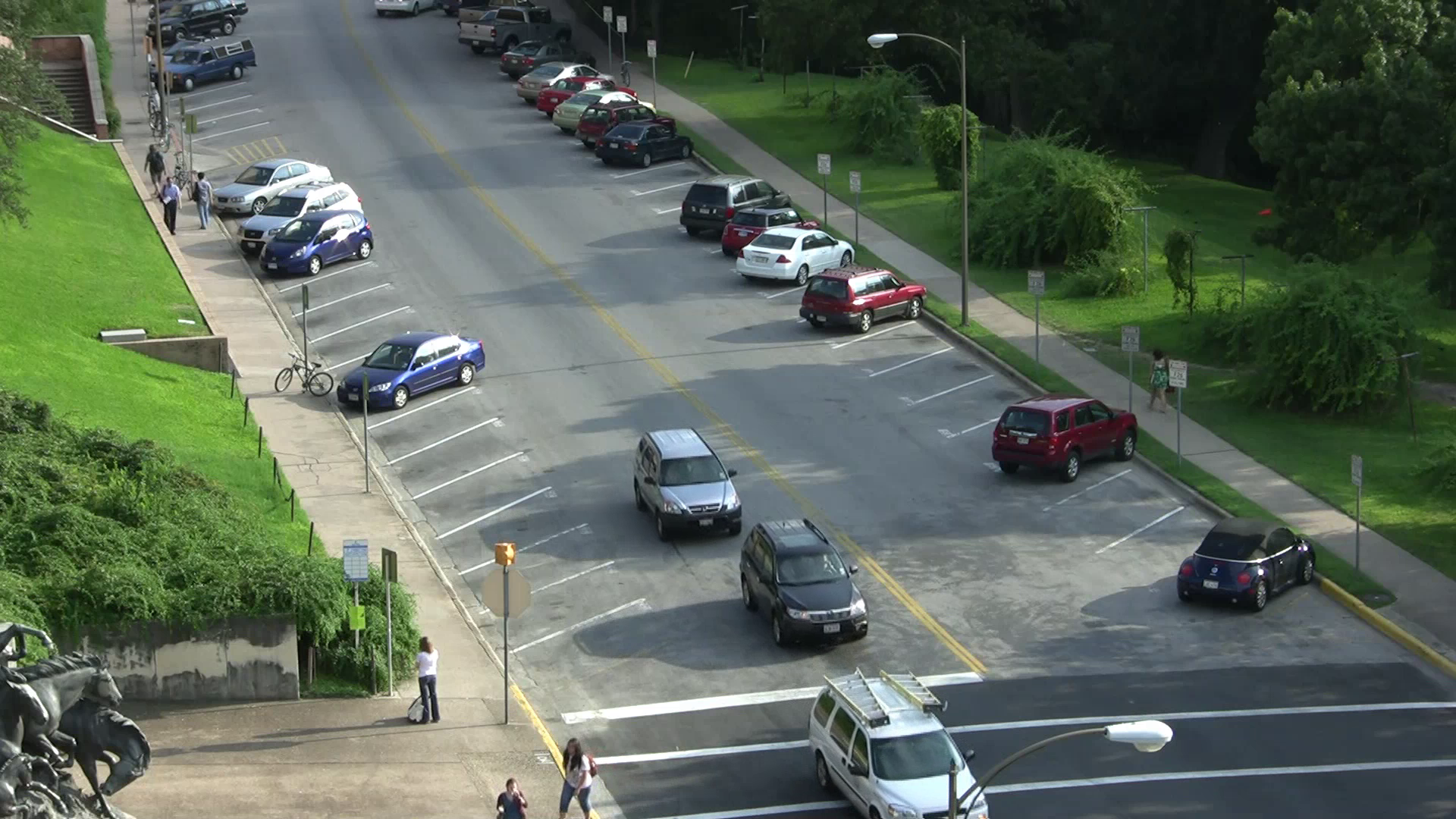}
\includegraphics[width=0.24\textwidth]{red4_VIRAT_S_050100_10_001410_001458_frame140.png}
\includegraphics[width=0.24\textwidth]{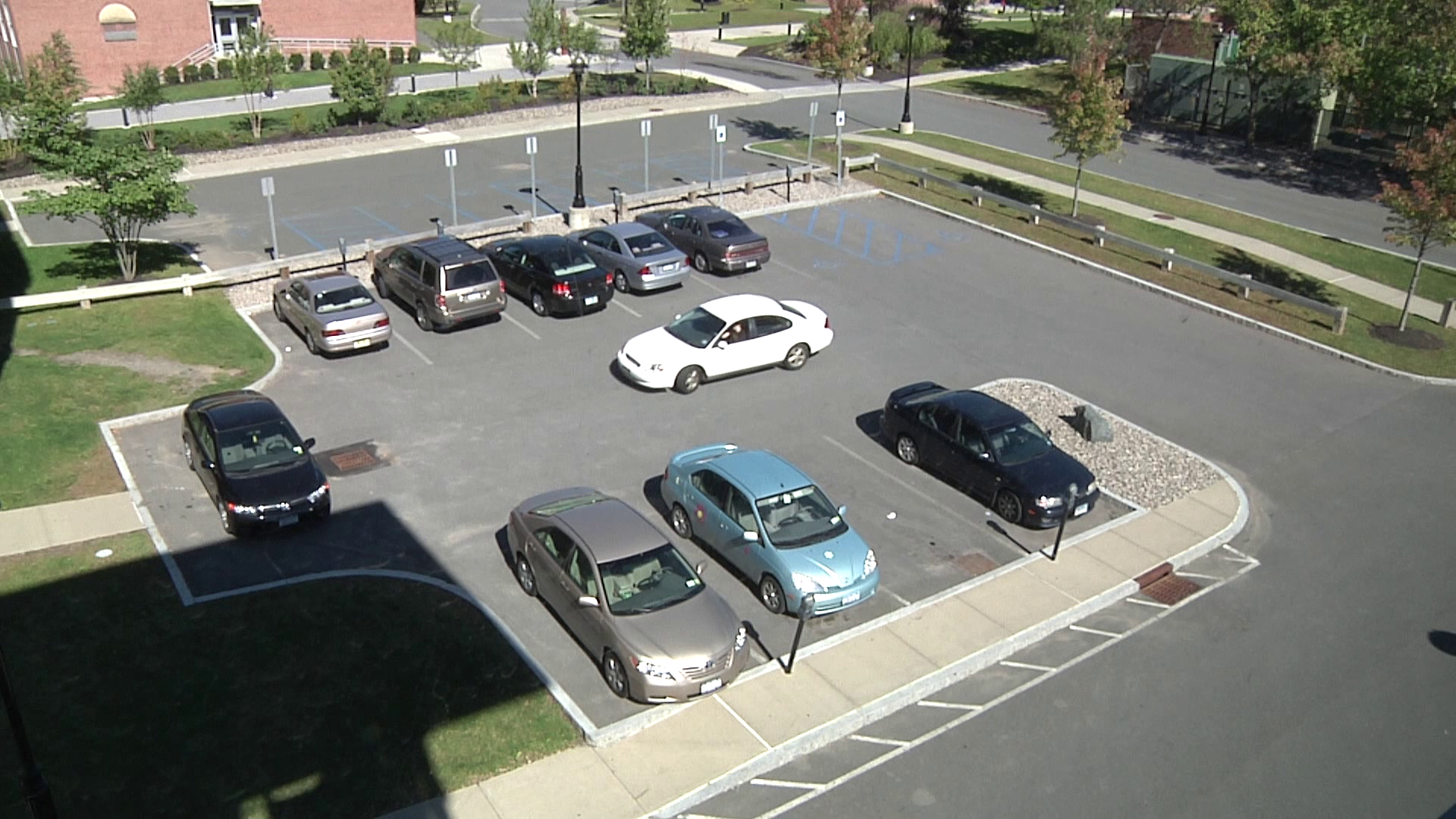}
\includegraphics[width=0.24\textwidth]{VIRAT_S_040101_06_001557_001590_frame380.png}

\includegraphics[width=0.24\textwidth]{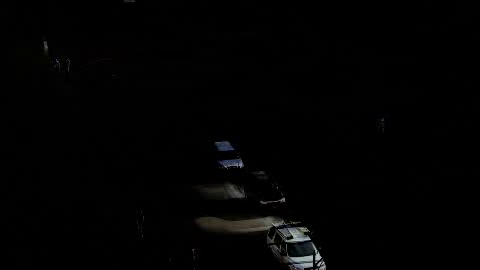}
\includegraphics[width=0.24\textwidth]{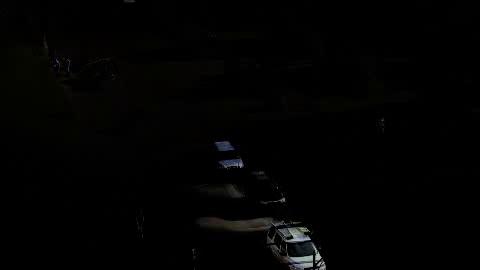}
\includegraphics[width=0.24\textwidth]{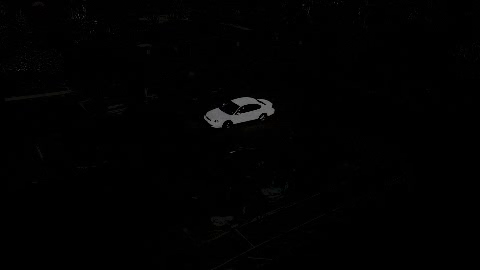}
\includegraphics[width=0.24\textwidth]{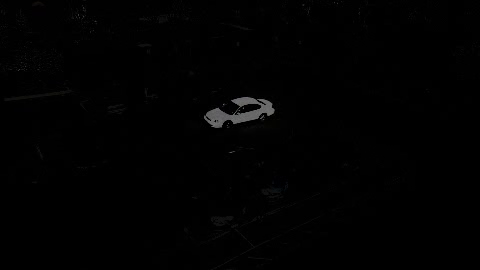}

\includegraphics[width=0.24\textwidth]{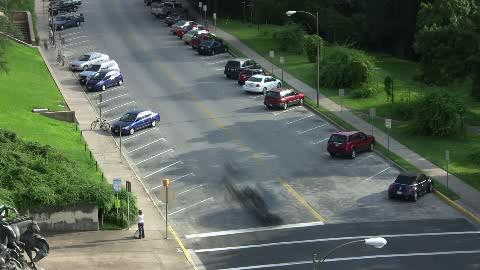}
\includegraphics[width=0.24\textwidth]{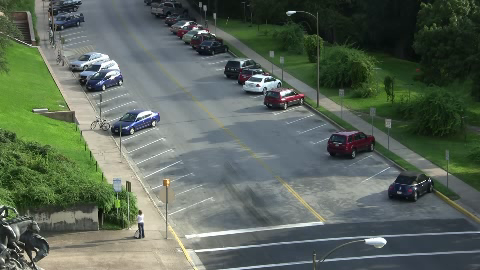}
\includegraphics[width=0.24\textwidth]{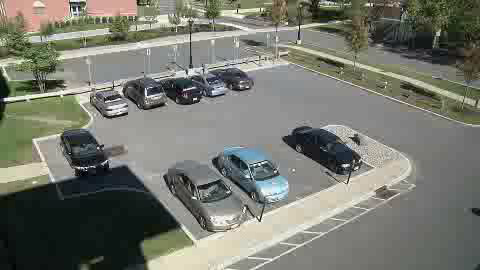}
\includegraphics[width=0.24\textwidth]{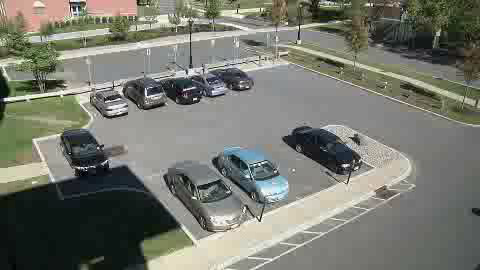}

\caption{Examples of extracted background and foreground in video surveillance with and without selective mode updates. The first two columns use the same street surveillance video, and the last two columns use the same parking lot video, where the same frame is shown for each pair. In a pair, the left column is the result when using the original ScaledGD algorithm, and the right column is the result employing selective mode updates. The first row is the original frame, and second row is the sparse foreground, and the third row is the background.}
\label{background1}
\end{figure*}

\section{Conclusions}
 
In this paper, we proposed a new algorithm for tensor RPCA, which applies scaled gradient descent with an iteration-varying thresholding operation to adaptively remove the impact of corruptions. The algorithm is demonstrated to be fast, provably accurate, and achieve competitive performance over real-world applications.  It opens several interesting directions for future research.
\begin{itemize}
\item {\em Dependency with the condition number from spectral initialization.} As seen from the analysis, the local linear convergence of Algorithm \ref{alg:tensor_RPCA} succeeds under a larger range of the sparsity level $\alpha$ independent of the condition number $\kappa$. The constraint on $\alpha$ with respect to the condition number $\kappa$ mainly stems from the spectral initialization, and it is of great interest to see if it is possible to refine the analysis in terms of the dependency with $\kappa$, which likely will require new tools.

\item {\em Missing data.} An important extension is to handle missing data in tensor RPCA, which seeks to recover a low-rank ground truth tensor from its partially observed entries possibly corrupted by gross errors. Our proposed algorithm can be adapted to this case in a straightforward fashion by considering the loss function defined only over the observed entries, and understanding its performance guarantees is a natural step.

\item {\em Streaming data.} An equally interesting direction is to perform tensor RPCA over online and streaming data, where the fibers or slices of the tensor arrive sequentially over time, a situation that is common in data analytics \cite{balzano2018streaming,vaswani2018static}. It is of great interest to develop low-complexity algorithms that can estimate and track the low-rank tensor factors as quickly as possible.
 
 \item {\em Hyperparameters.} The proposed algorithm contains several hyperparameters that need to be tuned carefully to fully unleash its potential. A recent follow-up \cite{dong2022deep} examined a learned approach based on algorithm unfolding and self-supervised learning to enable automatic hyperparameter tuning. In addition, understanding the performance when the rank is only imperfectly specified is also of great importance, which is closely related to \cite{xu2023power}.
 
\end{itemize}

 \section*{Acknowledgements}
 
The work of H. Dong, T.~Tong and Y.~Chi is supported in part by Office of Naval Research under N00014-19-1-2404, by Air Force Research Laboratory under FA8750-20-2-0504, by National Science Foundation under CAREER ECCS-1818571, CCF-1901199 and ECCS-2126634, and by Department of Transportation under 693JJ321C000013. The work of H. Dong is also supported by CIT Dean's Fellowship, Liang Ji-Dian Graduate Fellowship, and
Michel and Kathy Doreau Graduate Fellowship in Electrical and Computer Engineering at Carnegie Mellon University. The U.S. Government is authorized to reproduce and distribute reprints for Governmental purposes notwithstanding any copyright notation thereon.
 
\bibliographystyle{alphaabbr}
\bibliography{bibfileTensor.bib,bibfileNonconvexScaledGD.bib}

\appendix

% !TEX root = ./Tensor_RPCA.tex
% Main lemmas used to prove the theorem

\section{Proof of Lemma \ref{dist_contract_incoherence_contract}}\label{sec:proof_main_lemmas}

Since $\dist(\bF_t,\bF_\star) < \sigma_{\min}(\bcX_\star)$, \cite[Lemma 6]{tong2021scaling} ensures that the optimal alignment matrices $\big\{\bQ_t^{(k)}\big\}_{k = 1}^3$ between $\bF_t$ and $\bF_\star$ exist. Since $\big\{\bQ_t^{(k)}\big\}_{k = 1}^3$ may be a suboptimal alignment between $\bF_{t+1}$ and $\bF_\star$, we therefore have
\begin{align} \label{eq:dist_contract_t1}
    \dist^2(\bF_{t+1}, \bF_\star) \leq \sum_{k=1}^{3} \norm{ \big(\bU_{t+1}^{(k)} \bQ_t^{(k)} - \bU^{(k)}_\star \big) \bSigma_\star^{(k)}}_{\fro}^2 + \norm{ \big((\bQ_t^{(1)})^{-1}, (\bQ_t^{(2)})^{-1}, (\bQ_t^{(3)})^{-1} \big) \bcdot \bcG_{t+1} - \bcG_\star}_{\fro}^2 .
\end{align}
Before we embark on the control of the terms on the right hand side of \eqref{eq:dist_contract_t1},  
we introduce the following short-hand notations:
\begin{align} \label{eq:short_hand}
    \bU^{(k)} &:= \bU_{t}^{(k)} \bQ_t^{(k)} , \qquad\qquad     \Breve{\bU}^{(k)} := \Breve{\bU}_t^{(k)} (\bQ_t^{(k)})^{-\top} , \qquad 
  %  \bcG &:= ((\bQ_t^{(1)})^{-1}, (\bQ_t^{(2)})^{-1}, (\bQ_t^{(3)})^{-1}) \bcdot \bcG_{t} \\
    \bcS := \bcS_{t+1} , \nonumber \\
    \bDelta_{\bU^{(k)}} &:= \bU^{(k)}-\bU^{(k)}_\star, \qquad     \bDelta_{\Breve{\bU}^{(k)}} := \Breve{\bU}^{(k)} - \Breve{\bU}_\star^{(k)}, \qquad
       \bDelta_{\bcS} := \bcS - \bcS_\star , \\
        \bcG &:= \big((\bQ_t^{(1)})^{-1}, (\bQ_t^{(2)})^{-1}, (\bQ_t^{(3)})^{-1} \big) \bcdot \bcG_{t} , \quad \qquad\qquad
  \bDelta_{\bcG} := \bcG - \bcG_\star. \nonumber
\end{align}
In addition, Lemma~\ref{delta_x_inf_bound_contract_incoherence} in conjunction with the induction hypothesis \eqref{eq:induction_hypothesis} tells us that 
\begin{align}\label{eq:incoherence_contract_zeta_assump}
    \norm{\bcX_t - \bcX_\star}_\infty & \leq 8 \sqrt{\frac{\mu^3 r_1 r_2 r_3}{n_1 n_2 n_3}} \rho^t \sigma_{\min}(\bcX_\star) =:   \zeta_{t+1}   .
\end{align}

\paragraph{Step 1: bounding the first term of \eqref{eq:dist_contract_t1}.} Using the update rule \eqref{a_update3} for $\bU_{t+1}^{(k)}$, we have
\begin{align*}
   & \big(\bU_{t+1}^{(k)} \bQ_{t}^{(k)} - \bU^{(k)}_\star \big) \bSigma_\star^{(k)} \\
   &= \left[\left(  (1 - \eta) \bU_t^{(k)} - \eta \big(\Matricize{k}{\bcS_{t+1}} -  \Matricize{k}{\bcY} \big)\Breve{\bU}_t^{(k)} \big(\Breve{\bU}_t^{(k)\top} \Breve{\bU}_t^{(k)} \big)^{-1} \right) \bQ_t^{(k)} - \bU^{(k)}_\star\right] \bSigma_\star^{(k)} \\
   &= \left[\left(  (1 - \eta) \bU_t^{(k)} - \eta \big(\Matricize{k}{\bcS_{t+1}-\bcS_{\star}} -  \bU^{(k)}_\star \Breve{\bU}_\star^{(k)\top} \big)\Breve{\bU}_t^{(k)} \big(\Breve{\bU}_t^{(k)\top} \Breve{\bU}_t^{(k)} \big)^{-1} \right) \bQ_t^{(k)} - \bU^{(k)}_\star\right] \bSigma_\star^{(k)} ,
\end{align*}
where the second equality follows from $\bcY = \bcX_\star + \bcS_\star$ as well as the matricization property \eqref{eq:matricization_property}. With the set of notation \eqref{eq:short_hand} in place, simple algebraic simplifications yield
\begin{align}  
 \big(\bU_{t+1}^{(k)} \bQ_{t}^{(k)} - \bU^{(k)}_\star \big) \bSigma_\star^{(k)}&=  \left[(1-\eta) \bU^{(k)} - \eta \left(\Matricize{k}{\bDelta_{\bcS}} - \bU^{(k)}_\star  \Breve{\bU}_\star^{(k)\top} \right) \Breve{\bU}^{(k)} \big( \Breve{\bU}^{(k)\top} \Breve{\bU}^{(k)} \big)^{-1} - \bU^{(k)}_\star\right] \bSigma_\star^{(k)} \nonumber \\
     &= \left[(1-\eta) \bDelta_{\bU^{(k)}} - \eta \left(\Matricize{k}{\bDelta_{\bcS}} + \bU^{(k)}_\star \bDelta_{\Breve{\bU}^{(k)}}^\top\right) \Breve{\bU}^{(k)}\big(\Breve{\bU}^{(k)\top} \Breve{\bU}^{(k)}\big)^{-1} \right] \bSigma_\star^{(k)} \nonumber\\
    &= (1-\eta) \bDelta_{\bU^{(k)}} \bSigma_\star^{(k)} - \eta \left(\Matricize{k}{\bDelta_{\bcS}} + \bU^{(k)}_\star \bDelta_{\Breve{\bU}^{(k)}}^\top\right) \Breve{\bU}^{(k)} \big(\Breve{\bU}^{(k)\top} \Breve{\bU}^{(k)} \big)^{-1} \bSigma_\star^{(k)} . \label{eq:update_uk}
\end{align}
Detailed in Appendix~\ref{sec:proof_delta_a_bound}, we claim the following bound holds: 
\begin{align} \label{delta_a_bound}
    \sum_{k=1}^3 \norm{(\bU_{t+1}^{(k)} \bQ_t^{(k)} - \bU^{(k)}_\star) \bSigma_\star^{(k)}}_{\fro}^2 &\leq (1-\eta)^2 \sum_{k=1}^3 \norm{\bDelta_{\bU^{(k)}} \bSigma_\star^{(k)}}_{\fro}^2 + 0.15 \eta (1-\eta) \sum_{k=1}^3 \norm{\bDelta_{\bU^{(k)}} \bSigma_\star^{(k)}}_{\fro} \epsilon_0 \rho^t \sigma_{\min}(\bcX_\star) \nonumber \\
    &\quad + 2 \eta^2 \epsilon_0^2 \rho^{2t} \sigma_{\min}^2(\bcX_\star) + 0.06 \eta (1-\eta) \epsilon_0^2 \rho^{2t} \sigma_{\min}^2(\bcX_\star).
\end{align}

\paragraph{Step 2: bounding the second term of \eqref{eq:dist_contract_t1}.} Using the update rule \eqref{g_update3} for $\bcG_{t+1}$, we have 
\begin{align}
    &\big((\bQ_t^{(1)})^{-1}, (\bQ_t^{(2)})^{-1}, (\bQ_t^{(3)})^{-1} \big) \bcdot \bcG_{t+1} - \bcG_\star \nonumber\\
    & = (1-\eta) \big((\bQ_t^{(1)})^{-1}, (\bQ_t^{(2)})^{-1}, (\bQ_t^{(3)})^{-1} \big) \bcdot  \bcG_t  \nonumber \\
    & \qquad - \eta  \left( \big(\bU^{(1)\top}_t \bU_t^{(1)}\bQ_t^{(1)}\big)^{-1} \bU_t^{(1)\top}, \big(\bU^{(2)\top}_t \bU_t^{(2)}\bQ_t^{(2)}\big)^{-1} \bU_t^{(2)\top}, \big(\bU^{(3)\top}_t \bU_t^{(3)}\bQ_t^{(3)}\big)^{-1} \bU_t^{(3)\top} \right)  \bcdot \left(\bcS_{t+1} - \bcY \right) - \bcG_\star \nonumber\\
    &= (1-\eta) \bcG - \eta \left( \big(\bU^{(1)\top} \bU^{(1)}\big)^{-1} \bU^{(1)\top},  \big(\bU^{(2)\top} \bU^{(2)}\big)^{-1} \bU^{(2)\top},  \big(\bU^{(3)\top} \bU^{(3)}\big)^{-1} \bU^{(3)\top} \right) \bcdot \left(  \bDelta_{\bcS} - \bcX_{\star} \right) - \bcG_\star  \nonumber\\
    & = (1-\eta) \bDelta_{\bcG} - \eta  \left( \big(\bU^{(1)\top} \bU^{(1)}\big)^{-1} \bU^{(1)\top},  \big(\bU^{(2)\top} \bU^{(2)}\big)^{-1} \bU^{(2)\top},  \big(\bU^{(3)\top} \bU^{(3)}\big)^{-1} \bU^{(3)\top} \right) \nonumber\\
    &\qquad\qquad\qquad\qquad\qquad \bcdot \left( \big(\bU^{(1)}, \bU^{(2)}, \bU^{(3)} \big) \bcdot \bcG_\star - \bcX_\star + \bDelta_{\bcS} \right) , \label{eq:core_tensor_dist}
\end{align}
 where the last two lines make use of the short-hand notation in \eqref{eq:short_hand}, as well as the multilinear property \eqref{eq:multilinear}. Detailed in Appendix~\ref{sec:proof_delta_core_bound}, we claim the following bound holds:  
\begin{align} \label{delta_core_bound}
   \norm{ \big( (\bQ_t^{(1)})^{-1}, (\bQ_t^{(2)})^{-1}, (\bQ_t^{(3)})^{-1} \big) \bcdot \bcG_{t+1} - \bcG_\star}_{\fro}^2 
    &\leq (1-\eta)^2 \norm{\bDelta_{\bcG}}_{\fro}^2 + 2 \cdot 0.15 \eta(1-\eta) \norm{\bDelta_{\bcG}}_{\fro} \epsilon_0 \rho^t \sigma_{\min}(\bcX_\star) \nonumber
    \\
    & \quad + 0.02 \eta(1-\eta) \epsilon_0^2 \rho^{2t} \sigma_{\min}^2(\bcX_\star)  + 0.06 \eta^2 \epsilon_0^2 \rho^{2t}\sigma_{\min}^2(\bcX_\star) .
\end{align}

\paragraph{Step 3: putting the bounds together.} 
Plugging the bounds~\eqref{delta_a_bound} and \eqref{delta_core_bound} into~\eqref{eq:dist_contract_t1} yields
\begin{align} \label{dist_contract_t1_banana}
  &  \dist^2(\bF_{t+1}, \bF_\star) \nonumber \\
  &\leq (1-\eta)^2\left( \sum_{k=1}^3 \norm{\bDelta_{\bU^{(k)}} \bSigma_\star^{(k)}}_{\fro}^2  + \norm{\bDelta_{\bcG}}_{\fro}^2 \right) +  0.15 \eta (1-\eta) \left( \sum_{k=1}^3 \norm{\bDelta_{\bU^{(k)}} \bSigma_\star^{(k)}}_{\fro} + 2 \norm{\bDelta_{\bcG}}_{\fro} \right) \epsilon_0 \rho^t \sigma_{\min}(\bcX_\star) \nonumber \\
  & \quad + 2.1 \eta^2 \epsilon_0^2 \rho^{2t} \sigma_{\min}^2(\bcX_\star)  + 0.08 \eta (1-\eta) \epsilon_0^2 \rho^{2t} \sigma_{\min}^2(\bcX_\star)\nonumber  \\
  & \overset{\mathrm{(i)}}{\leq}  (1-\eta)^2\dist^2(\bF_t, \bF_\star) + 0.4 \eta (1-\eta)\dist(\bF_t, \bF_\star) \epsilon_0 \rho^t \sigma_{\min}(\bcX_\star) + 2.1 \eta^2 \epsilon_0^2 \rho^{2t} \sigma_{\min}^2(\bcX_\star) \\ 
  &\quad + 0.08 \eta (1-\eta) \epsilon_0^2 \rho^{2t} \sigma_{\min}^2(\bcX_\star) \nonumber \\
  & \overset{\mathrm{(ii)}}{\leq}    ((1-\eta)^2 + 0.4 \eta (1-\eta)  + 2.1 \eta^2 + 0.08 \eta (1-\eta)) \epsilon_0^2 \rho^{2t} \sigma_{\min}^2(\bcX_\star) \nonumber \\
    &= ((1-\eta)^2 + 0.5 \eta (1-\eta) + 2.1 \eta^2 ) \epsilon_0^2 \rho^{2t} \sigma_{\min}^2(\bcX_\star) ,
\end{align}
where (i) follows from the definition of $\dist^2(\bF_t, \bF_\star)$ and Cauchy-Schwarz, and (ii) follows from the induction hypothesis $\dist(\bF_t, \bF_\star) \leq \epsilon_0 \rho^t \sigma_{\min}(\bcX_\star)$.
For $0 < \eta \leq 1/4$ and $\rho = 1-0.45\eta$, this simplifies to the claimed bound
\begin{align*} 
    \dist^2(\bF_{t+1}, \bF_\star) \leq (1-0.45\eta)^2\epsilon_0^2 \rho^{2t} \sigma_{\min}^2(\bcX_\star) = \epsilon_0^2 \rho^{2t+2} \sigma_{\min}^2(\bcX_\star).
\end{align*}

\subsection{Proof of \eqref{delta_a_bound} }
\label{sec:proof_delta_a_bound}

Taking the squared norm on both sides of \eqref{eq:update_uk}, we obtain
\begin{align}
    \norm{(\bU_{t+1}^{(k)} \bQ_{t, k} - \bU^{(k)}_\star) \bSigma_\star^{(k)}}_{\fro}^2 &= (1-\eta)^2 \norm{\bDelta_{\bU^{(k)}} \bSigma_\star^{(k)}}_{\fro}^2 + \eta^2 \underbrace{\norm{\Matricize{k}{\bDelta_{\bcS}} \Breve{\bU}^{(k)} \big(\Breve{\bU}^{(k)\top} \Breve{\bU}^{(k)} \big)^{-1} \bSigma_\star^{(k)}}_{\fro}^2}_{=: \mfk{A}_{1,k}} \nonumber\\
    & \quad - 2 \eta (1-\eta) \underbrace{\inner{\bDelta_{\bU^{(k)}} \bSigma_\star^{(k)}}{\bU^{(k)}_\star \bDelta_{\Breve{\bU}^{(k)}}^\top \Breve{\bU}^{(k)} \big(\Breve{\bU}^{(k)\top} \Breve{\bU}^{(k)}\big)^{-1} \bSigma_\star^{(k)}}}_{=: \mfk{A}_{2,k}} \nonumber\\
    & \quad + \eta^2 \underbrace{\norm{\bU^{(k)}_\star \bDelta_{\Breve{\bU}^{(k)}}^\top \Breve{\bU}^{(k)} \big(\Breve{\bU}^{(k)\top} \Breve{\bU}^{(k)}\big)^{-1} \bSigma_\star^{(k)}}_{\fro}^2}_{=: \mfk{A}_{3,k}} \nonumber\\
    & \quad - 2 \eta (1-\eta) \underbrace{\inner{\bDelta_{\bU^{(k)}} \bSigma_\star^{(k)}}{\Matricize{k}{\bDelta_{\bcS}} \Breve{\bU}^{(k)} \big(\Breve{\bU}^{(k)\top} \Breve{\bU}^{(k)}\big)^{-1} \bSigma_\star^{(k)}}}_{=: \mfk{A}_{4,k}} \nonumber\\
    & \quad + 2 \eta^2 \underbrace{\inner{\Matricize{k}{\bDelta_{\bcS}} \Breve{\bU}^{(k)}\big(\Breve{\bU}^{(k)\top} \Breve{\bU}^{(k)}\big)^{-1} \bSigma_\star^{(k)}}{\bU^{(k)}_\star \bDelta_{\Breve{\bU}^{(k)}}^\top \Breve{\bU}^{(k)}\big(\Breve{\bU}^{(k)\top} \Breve{\bU}^{(k)}\big)^{-1} \bSigma_\star^{(k)}}}_{=: \mfk{A}_{5,k}}. \label{dist_contract_mfk_a_combine}
\end{align}
In the sequel, we shall bound each term separately.
%In addition, we claim the following inequalities, which hold under the assumption of Lemma~\ref{dist_contract_incoherence_contract}.%; the proofs are deferred to the end of this subsection.

\begin{itemize}
\item \textbf{Bounding $\mfk{A}_{1,k}$.} Since the quantity inside the norm is of rank $r_k$, we have
\begin{align*}
    \mfk{A}_{1,k} &= \norm{\Matricize{k}{\bDelta_{\bcS}} \Breve{\bU}^{(k)} \big(\Breve{\bU}^{(k)\top} \Breve{\bU}^{(k)} \big)^{-1} \bSigma_\star^{(k)}}_{\fro}^2 \\
    &\leq r_k \norm{\Matricize{k}{\bDelta_{\bcS}} \Breve{\bU}^{(k)} \big(\Breve{\bU}^{(k)\top} \Breve{\bU}^{(k)} \big)^{-1} \bSigma_\star^{(k)}}_{\op}^2 \\
    &\leq r_k \norm{\Matricize{k}{\bDelta_{\bcS}}}_{\op}^2 \norm{\Breve{\bU}^{(k)} \big(\Breve{\bU}^{(k)\top} \Breve{\bU}^{(k)} \big)^{-1} \bSigma_\star^{(k)}}_{\op}^2 \leq \frac{r_k}{(1-\epsilon_0)^6} \norm{\Matricize{k}{\bDelta_{\bcS}}}_{\op}^2,
\end{align*}
where the last inequality follows from Lemma~\ref{more_perturb} (cf.~\eqref{eq:perturb_spec_f}). To continue, notice that the choice of $\zeta_{t+1}$ (cf. \eqref{eq:incoherence_contract_zeta_assump}) guarantees that  $\bDelta_{\bcS}$ (and hence $\Matricize{k}{\bDelta_{\bcS}}$)  is $\alpha$-sparse (cf.~Lemma \ref{corrupt_iter}). This allows us to invoke Lemma \ref{sparsity_norm_eq} and obtain
\begin{align}\label{eq:mk_delta_S_op}
\norm{\Matricize{k}{\bDelta_{\bcS}}}_{\op} \leq \alpha \sqrt{n_1 n_2 n_3}\norm{\Matricize{k}{\bDelta_{\bcS}}}_{\infty} = \alpha \sqrt{n_1 n_2 n_3} \norm{\bDelta_{\bcS}}_\infty \leq  2\alpha \sqrt{n_1 n_2 n_3} \zeta_{t+1}.
\end{align}
Plugging this into the previous inequality, we arrive at
\begin{align*}
  \mfk{A}_{1,k} & \leq  \frac{4\alpha^2 n_1 n_2 n_3 r_k }{(1-\epsilon_0)^6} \zeta_{t+1}^2 \leq  \frac{256 \alpha^2 \mu^3 r_1 r_2 r_3 r_k}{(1-\epsilon_0)^6}  \rho^{2t} \sigma_{\min}^2(\bcX_\star) ,
\end{align*}
where the second inequality follows from the choice of $\zeta_{t+1}$ (cf. \eqref{eq:incoherence_contract_zeta_assump}). Finally, with the assumption on the sparsity level $\alpha \leq \frac{c_0 \epsilon_0} { \sqrt{\mu^3 r_1 r_2 r_3 r }}$, we have
\begin{align} \label{dist_contract_mfk_a_1}
    \mfk{A}_{1,k} & \leq \frac{256 c_0^2}{(1-\epsilon_0)^6} \epsilon_0^2 \rho^{2t} \sigma_{\min}^2(\bcX_\star)  \leq 0.02 \epsilon_0^2 \rho^{2t} \sigma_{\min}^2(\bcX_\star)
\end{align} 
for sufficiently small $c_0$ and $\epsilon_0<0.01$.

\item \textbf{Bounding $\mfk{A}_{2,k}$.}  
This term is identical to the term that is bounded in \cite[Section B.1]{tong2021scaling}, which obeys 
\begin{align*}
    \mfk{A}_{2,k} &\geq \inner{\bcK^{(k)}}{\bcK^{(1)} + \bcK^{(2)} + \bcK^{(3)}} - C_1 \epsilon_0 \rho^t \dist^2(\bF_t, \bF_\star) 
\end{align*}
for some constant $C_1 >1$ with 
\begin{align*}
    \bcK^{(1)} &:= \big( \bU_\star^{(1)\top} \bDelta_{\bU^{(1)}}, \bI_{r_2}, \bI_{r_3}\big) \bcdot \bcG_\star, \\
    \bcK^{(2)} &:= \big(\bI_{r_1}, \bU_\star^{(2)\top} \bDelta_{\bU^{(2)}}, \bI_{r_3}\big) \bcdot \bcG_\star,\\
    \bcK^{(3)} &:= \big(\bI_{r_1}, \bI_{r_2}, \bU_\star^{(3)\top} \bDelta_{\bU^{(3)}}\big) \bcdot \bcG_\star.
\end{align*}
As long as the choice of $\epsilon_0$ is small enough such that $C_1 \epsilon_0 \rho^t \leq C_1 \epsilon_0 <0.01$, the induction hypothesis \eqref{eq:dist_contraction} tells us that. 
\begin{align}
 \mfk{A}_{2,k}    
    &\geq \inner{\bcK^{(k)}}{\bcK^{(1)} + \bcK^{(2)} + \bcK^{(3)}} - 0.01\epsilon_0^2 \rho^{2t} \sigma_{\min}^2(\bcX_\star) . \label{dist_contract_mfk_a_2}
\end{align}

\item \textbf{Bounding $\mfk{A}_3$.} 
This term is identical to the term that is bounded in \cite[Section B.2]{tong2021scaling}, which obeys  
\begin{align*}
    \mfk{A}_{3,k} &\leq \norm{\bcK^{(1)} + \bcK^{(2)} + \bcK^{(3)}}_{\fro}^2 + C_2 \epsilon_0 \rho^t \dist^2(\bF_t, \bF_\star) 
    \end{align*} 
    for some constant $C_2>1$. As long as the choice of $\epsilon_0$ is small enough such that $C_2 \epsilon_0 \rho^t \leq C_2 \epsilon_0 <0.01$, the induction hypothesis \eqref{eq:dist_contraction} results in 
\begin{align}
     \mfk{A}_{3,k}       &\leq \norm{\bcK^{(1)} + \bcK^{(2)} + \bcK^{(3)}}_{\fro}^2 + 0.01\epsilon_0^2 \rho^{2t} \sigma_{\min}^2(\bcX_\star) . \label{dist_contract_mfk_a_3}
\end{align}

\item \textbf{Bounding $\mfk{A}_{4,k}$.} To control $|\mfk{A}_{4,k}|$, we apply the definition of the matrix inner product to rewrite it as
\begin{align*}
    |\mfk{A}_{4,k}| &= \left|\tr\left(\Matricize{k}{\bDelta_{\bcS}} \Breve{\bU}^{(k)} \big(\Breve{\bU}^{(k)\top}  \Breve{\bU}^{(k)} \big)^{-1} \big(\bSigma_\star^{(k)} \big)^2 \bDelta_{\bU^{(k)}}^\top\right)\right| \\
    & \leq \norm{\Matricize{k}{\bDelta_{\bcS}}}_{\op} \norm{\Breve{\bU}^{(k)} \big(\Breve{\bU}^{(k)\top} \Breve{\bU}^{(k)} \big)^{-1} \big(\bSigma_\star^{(k)} \big)^2 \bDelta_{\bU^{(k)}}^\top}_*  \\
    & \leq 2\alpha \sqrt{n_1 n_2 n_3} \zeta_{t+1} \norm{\Breve{\bU}^{(k)} \big(\Breve{\bU}^{(k)\top} \Breve{\bU}^{(k)} \big)^{-1} \big(\bSigma_\star^{(k)} \big)^2 \bDelta_{\bU^{(k)}}^\top}_* ,
\end{align*}
where the second line follows from H\"{o}lder's inequality, and the third line follows from the bound of $\norm{\Matricize{k}{\bDelta_{\bcS}}}_{\op}$ from \eqref{eq:mk_delta_S_op}. To bound the remaining term, we have
\begin{align*}
 \norm{\Breve{\bU}^{(k)} \big(\Breve{\bU}^{(k)\top} \Breve{\bU}^{(k)} \big)^{-1} \big(\bSigma_\star^{(k)} \big)^2 \bDelta_{\bU^{(k)}}^\top}_*   &\leq   \sqrt{r_k}\norm{\Breve{\bU}^{(k)} \big(\Breve{\bU}^{(k)\top} \Breve{\bU}^{(k)} \big)^{-1} \big(\bSigma_\star^{(k)} \big)^2 \bDelta_{\bU^{(k)}}^\top}_{\fro} \\
 & \leq \sqrt{r_k} \norm{\Breve{\bU}^{(k)} \big(\Breve{\bU}^{(k)\top} \Breve{\bU}^{(k)} \big)^{-1}  \bSigma_\star^{(k)} } \norm{\bDelta_{\bU^{(k)}} \bSigma_\star^{(k)} }_{\fro} \\
 & \leq \frac{\sqrt{r_k}}{(1-\epsilon_0 )^3}\norm{\bDelta_{\bU^{(k)}} \bSigma_\star^{(k)} }_{\fro},
\end{align*}
where the last line follows from Lemma~\ref{more_perturb} (cf.~\eqref{eq:perturb_spec_f}). Plugging this into the previous inequality, we reach 
\begin{align*}
    |\mfk{A}_{4,k}| &\leq \frac{2\alpha\zeta_{t+1} \sqrt{n_1n_2n_3r_k} }{(1-\epsilon_0 )^3}\norm{\bDelta_{\bU^{(k)}} \bSigma_\star^{(k)} }_{\fro} \leq  \frac{16 \alpha \sqrt{\mu^3 r_1 r_2 r_3 r_k}}{(1-\epsilon_0 )^3}   \rho^t \sigma_{\min}(\bcX_\star) \norm{\bDelta_{\bU^{(k)}} \bSigma_\star^{(k)}}_{\fro} ,
\end{align*}    
where the second inequality follows from the choice of $\zeta_{t+1}$ (cf. \eqref{eq:incoherence_contract_zeta_assump}). Finally, with the assumption on the sparsity level $\alpha \leq \frac{c_0 \epsilon_0} { \sqrt{\mu^3 r_1 r_2 r_3 r}}$ for sufficiently small $c_0$ and $\epsilon_0<0.01$, we have
\begin{align}\label{dist_contract_mfk_a_4}
    |\mfk{A}_{4,k}| & \leq \frac{16c_0}{(1-\epsilon_0 )^3}\norm{\bDelta_{\bU^{(k)}} \bSigma_\star^{(k)}}_{\fro} \epsilon_0 \rho^t \sigma_{\min}(\bcX_\star)  \leq 0.15 \norm{\bDelta_{\bU^{(k)}} \bSigma_\star^{(k)}}_{\fro} \epsilon_0 \rho^t \sigma_{\min}(\bcX_\star) .
\end{align}

\item \textbf{Bounding $\mfk{A}_{5,k}$.} Similar to $\mfk{A}_{4,k}$, we first apply the definition of the matrix inner product and then H\"{o}lder's inequality, leading to
\begin{align*}
    |\mfk{A}_{5,k}|     &= \left| \tr\left(\Matricize{k}{\bDelta_{\bcS}} \Breve{\bU}^{(k)}\big( \Breve{\bU}^{(k)\top} \Breve{\bU}^{(k)} \big)^{-1} \big(\bSigma_\star^{(k)} \big)^2 \big(\Breve{\bU}^{(k)\top} \Breve{\bU}^{(k)} \big)^{-1} \Breve{\bU}^{(k)\top} \bDelta_{\Breve{\bU}^{(k)}} \bU^{(k)\top}_\star  \right) \right| \\
    &\leq \norm{\Matricize{k}{\bDelta_{\bcS}}}_{\op} \norm{\Breve{\bU}^{(k)} \big(\Breve{\bU}^{(k)\top} \Breve{\bU}^{(k)} \big)^{-1} \big(\bSigma_\star^{(k)} \big)^2 \big(\Breve{\bU}^{(k)\top} \Breve{\bU}^{(k)} \big)^{-1} \Breve{\bU}^{(k)\top} \bDelta_{\Breve{\bU}^{(k)}} \bU^{(k)\top}_\star }_* \\
 & \leq 2\zeta_{t+1} \alpha \sqrt{n_1 n_2 n_3 r_k} \norm{\Breve{\bU}^{(k)} \big(\Breve{\bU}^{(k)\top} \Breve{\bU}^{(k)} \big)^{-1} \big(\bSigma_\star^{(k)} \big)^2 \big(\Breve{\bU}^{(k)\top} \Breve{\bU}^{(k)} \big)^{-1} \Breve{\bU}^{(k)\top} \bDelta_{\Breve{\bU}^{(k)}} \bU^{(k)\top}_\star }_{\fro} ,
\end{align*}
 where the last line follows from \eqref{eq:mk_delta_S_op}, as well as the norm relation $\|\bA\|_* \leq \sqrt{r_k} \| \bA\|_{\fro}$ for a matrix of rank at most $r_k$. To continue, noting that  $\bU^{(k)}_\star$ has orthonormal columns, we have
\begin{align*}
& \norm{\Breve{\bU}^{(k)} \big(\Breve{\bU}^{(k)\top} \Breve{\bU}^{(k)} \big)^{-1} \big(\bSigma_\star^{(k)} \big)^2 \big(\Breve{\bU}^{(k)\top} \Breve{\bU}^{(k)} \big)^{-1} \Breve{\bU}^{(k)\top} \bDelta_{\Breve{\bU}^{(k)}} \bU^{(k)\top}_\star }_{\fro} \\
& \leq \norm{\Breve{\bU}^{(k)} \big(\Breve{\bU}^{(k)\top} \Breve{\bU}^{(k)} \big)^{-1}  \bSigma_\star^{(k)} }^2 \norm{  \bDelta_{\Breve{\bU}^{(k)}} }_{\fro} \\
&\leq \frac{\left(1 + \epsilon_0 +  \epsilon_0^2/3\right) }{(1-\epsilon_0)^6} \left(\norm{\bDelta_{\bU^{(2)}} \bSigma_\star^{(2)}}_{\fro} + \norm{\bDelta_{\bU^{(3)}} \bSigma_\star^{(3)}}_{\fro} + \norm{\bDelta_{\bcG}}_{\fro}\right)  \leq \frac{2\left(1 + \epsilon_0 +  \epsilon_0^2/3\right) }{(1-\epsilon_0)^6}   \dist(\bF_t, \bF_\star),
\end{align*}
where the penultimate line follows from Lemma~\ref{more_perturb} (cf. \eqref{eq:perturb_spec_f}  and \eqref{eq:perturb_fro_c}), and the last line follows from Cauchy-Schwarz inequality.
Plug this into the previous bound to arrive at
\begin{align*}
    |\mfk{A}_{5,k}| &\leq \frac{4\zeta_{t+1} \alpha \sqrt{n_1 n_2 n_3 r_k}\left(1 + \epsilon_0 +  \epsilon_0^2/3\right)}{(1-\epsilon_0)^6}   \dist(\bF_t, \bF_\star) 
    \leq \frac{32\alpha \sqrt{\mu^3 r_1 r_2 r_3 r_k}\left(1 + \epsilon_0 +  \epsilon_0^2/3\right)}{(1-\epsilon_0)^6}    \epsilon_0  \rho^{2t} \sigma_{\min}^2(\bcX_\star) ,
\end{align*}
where we have used $\dist(\bF_t, \bF_\star) \leq \epsilon_0 \rho^t \sigma_{\min}(\bcX_\star)$ and the choice of $\zeta_{t+1}$ (cf. \eqref{eq:incoherence_contract_zeta_assump}).  Finally, with the assumption on the sparsity level $\alpha \leq \frac{c_0 \epsilon_0} { \sqrt{\mu^3 r_1 r_2 r_3 r}}$ for sufficiently small $c_0$ and $\epsilon_0<0.01$, we have
\begin{align} \label{dist_contract_mfk_a_5}
    |\mfk{A}_{5,k}| &  \leq \frac{32c_0\left(1 + \epsilon_0 +  \epsilon_0^2/3\right)}{(1-\epsilon_0)^6}    \epsilon_0^2  \rho^{2t} \sigma_{\min}^2(\bcX_\star)  \leq 0.3 \epsilon_0^2  \rho^{2t} \sigma_{\min}^2(\bcX_\star) .
\end{align}

\end{itemize}

\paragraph{Putting things together.} 
Summing \eqref{dist_contract_mfk_a_combine} over all $k$, we obtain
\begin{align*}
    &\sum_{k=1}^3 \norm{(\bU_{t+1}^{(k)} \bQ_t^{(k)} - \bU^{(k)}_\star) \bSigma_\star^{(k)}}_{\fro}^2\\
    &= (1-\eta)^2 \sum_{k=1}^3 \norm{\bDelta_{\bU^{(k)}} \bSigma_\star^{(k)}}_{\fro}^2 + \eta^2 \sum_{k=1}^3 (\mfk{A}_{1,k} +\mfk{A}_{3,k} +2 \mfk{A}_{5,k} ) - 2 \eta (1-\eta) \sum_{k=1}^3 (\mfk{A}_{2,k} +\mfk{A}_{4,k}) .
\end{align*} 
Plugging in our bounds in \eqref{dist_contract_mfk_a_1}-\eqref{dist_contract_mfk_a_5}, we have
\begin{align*}
    \sum_{k=1}^3& \norm{(\bU_{t+1}^{(k)} \bQ_t^{(k)} - \bU^{(k)}_\star) \bSigma_\star^{(k)}}_{\fro}^2 \leq (1-\eta)^2 \sum_{k=1}^3 \norm{\bDelta_{\bU^{(k)}} \bSigma_\star^{(k)}}_{\fro}^2  \\
    &\quad + 2\eta^2 \epsilon_0^2 \rho^{2t} \sigma_{\min}^2(\bcX_\star)  - 2 \eta (1-\eta) \norm{\bcK^{(1)} + \bcK^{(2)} + \bcK^{(3)}}_{\fro}^2 + 0.06 \eta (1-\eta) \epsilon_0^2 \rho^{2t} \sigma_{\min}^2(\bcX_\star) \\
    &\quad  + 3 \eta^2 \norm{\bcK^{(1)} + \bcK^{(2)} + \bcK^{(3)}}_{\fro}^2 + 0.15 \eta (1-\eta) \sum_{k=1}^3 \norm{\bDelta_{\bU^{(k)}} \bSigma_\star^{(k)}}_{\fro} \epsilon_0 \rho^t \sigma_{\min}(\bcX_\star).
\end{align*}
Note that when $0< \eta < 2/5$, 
$$- 2 \eta (1-\eta) \norm{\bcK^{(1)} + \bcK^{(2)} + \bcK^{(3)}}_{\fro}^2 + 3 \eta^2 \norm{\bcK^{(1)} + \bcK^{(2)} + \bcK^{(3)}}_{\fro}^2 = - \eta (2-5\eta) \norm{\bcK^{(1)} + \bcK^{(2)} + \bcK^{(3)}}_{\fro}^2<0. $$
 Therefore, the previous bound can be simplified to
\begin{align*}
    \sum_{k=1}^3 \norm{(\bU_{t+1}^{(k)} \bQ_t^{(k)} - \bU^{(k)}_\star) \bSigma_\star^{(k)}}_{\fro}^2 &\leq (1-\eta)^2 \sum_{k=1}^3 \norm{\bDelta_{\bU^{(k)}} \bSigma_\star^{(k)}}_{\fro}^2 + 0.15 \eta (1-\eta) \sum_{k=1}^3 \norm{\bDelta_{\bU^{(k)}} \bSigma_\star^{(k)}}_{\fro} \epsilon_0 \rho^t \sigma_{\min}(\bcX_\star) \\
    &\quad + 2 \eta^2 \epsilon_0^2 \rho^{2t} \sigma_{\min}^2(\bcX_\star) + 0.06 \eta (1-\eta) \epsilon_0^2 \rho^{2t} \sigma_{\min}^2(\bcX_\star) .
\end{align*}

\subsection{Proof of \eqref{delta_core_bound} }
 
\label{sec:proof_delta_core_bound}
 
Taking the squared Frobenius norm of \eqref{eq:core_tensor_dist}, it follows
\begin{align} \label{dist_contract_mfk_b_combine}
    \norm{\big((\bQ_t^{(1)})^{-1}, (\bQ_t^{(2)})^{-1}, (\bQ_t^{(3)})^{-1} \big) \bcdot \bcG_{t+1} - \bcG_\star}_{\fro}^2 =(1-\eta)^2 \norm{\bDelta_{\bcG}}_{\fro}^2 - 2\eta(1-\eta) \mfk{B}_1 + \eta^2 \mfk{B}_2,
\end{align}
where 
\begin{align*}
    \mfk{B}_1 &= \left \langle\bDelta_{\bcG}, \Big( \big(\bU^{(1)\top} \bU^{(1)}\big)^{-1} \bU^{(1)\top},  \ldots,  \big(\bU^{(3)\top} \bU^{(3)}\big)^{-1} \bU^{(3)\top} \Big) \bcdot \Big( \big(\bU^{(1)}, \bU^{(2)}, \bU^{(3)} \big) \bcdot \bcG_\star - \bcX_\star + \bDelta_{\bcS} \Big) \right\rangle,
    \\
    \mfk{B}_2 &= \left\| \left( \big(\bU^{(1)\top} \bU^{(1)}\big)^{-1} \bU^{(1)\top},  \ldots ,  \big(\bU^{(3)\top} \bU^{(3)}\big)^{-1} \bU^{(3)\top} \right)   \bcdot \left( \big(\bU^{(1)}, \bU^{(2)}, \bU^{(3)} \big) \bcdot \bcG_\star - \bcX_\star + \bDelta_{\bcS} \right) \right\|_{\fro}^2 .
\end{align*}
We will now bound $\mfk{B}_1$ and $\mfk{B}_2$ separately.
\paragraph{Bounding $\mfk{B}_1$.}
We start by breaking up the inner product into
\begin{align*}
    \mfk{B}_1 &= \underbrace{\inner{\bDelta_{\bcG}}{\left(\big(\bU^{(1)\top} \bU^{(1)} \big)^{-1} \bU^{(1)\top}, \ldots, (\bU^{(3)\top} \bU^{(3)})^{-1}  \bU^{(3)\top} \right) \bcdot \left( \big(\bU^{(1)}, \bU^{(2)}, \bU^{(3)} \big) \bcdot \bcG_\star - \bcX_\star \right)}}_{ =: \mfk{B}_{1,1}} \\
    &\quad + \underbrace{\inner{\bDelta_{\bcG}}{\left( \big(\bU^{(1)\top} \bU^{(1)} \big)^{-1} \bU^{(1)\top}, \dots, \big(\bU^{(3)\top} \bU^{(3)} \big)^{-1} \bU^{(3)\top} \right) \bcdot \bDelta_{\bcS}}}_{ =: \mfk{B}_{1,2}} .
\end{align*}
Note that $\mfk{B}_{1, 1}$ is identical to the term that is bounded in \cite[Section B.3]{tong2021scaling}, which obeys
\begin{align}
    \mfk{B}_{1,1} &\geq \sum_{k=1}^3 \norm{\bD^{(k)}}_{\fro}^2  - C_1 \epsilon_0 \rho^t \dist^2(\bF_t, \bF_\star) 
    \geq  \sum_{k=1}^3 \norm{\bD^{(k)}}_{\fro}^2  - 0.01\epsilon_0^2 \rho^{2t} \sigma_{\min}^2(\bcX_\star), \label{dist_contract_mfk_b_1_1}
\end{align}
where we have used the induction hypothesis \eqref{eq:dist_contraction} and $C_1 \epsilon_0 \rho^t \leq C_1 \epsilon_0 <0.01$
as long as $\epsilon_0$ is small enough. Here,
 $$   \bD^{(k)} = \big( \bU^{(k)\top} \bU^{(k)}\big)^{-1/2} \bU^{(k)\top} \bDelta_{\bU^{(k)}} \bSigma_\star^{(k)}, \qquad k=1,2,3. $$

Turning to $\mfk{B}_{1,2}$, since the  inner product is invariant to matricization, we flatten the tensor along the first mode  to bound it as
%\todo[inline]{could have matricized along the axis with lowest $r_k$}
\begin{align*}
    |\mfk{B}_{1,2}| 
    &= \left|\inner{\Matricize{1}{\bDelta_{\bcG}}}{ \big( \bU^{(1)\top} \bU^{(1)}\big)^{-1} \bU^{(1)\top} \Matricize{1}{\bDelta_{\bcS}} \left((\bU^{(3)\top} \bU^{(3)})^{-1} (\bU^{(3)})^\top \otimes \big(\bU^{(2)\top} \bU^{(2)} \big)^{-1}  \bU^{(2)\top} \right)^\top}\right|
    \\
 %   &=\tr\left(\Matricize{1}{\bDelta_{\bcG}}^\top \big(\bU^{(1)\top} \bU^{(1)} \big)^{-1} \bU^{(1)\top} \Matricize{1}{\bDelta_{\bcS}} \left( \big( \bU^{(3)\top} \bU^{(3)} \big)^{-1}  \bU^{(3)\top} \otimes \big(\bU^{(2)\top} \bU^{(2)} \big)^{-1} \bU^{(2)\top} \right)^\top \right)
    &\leq \norm{\Matricize{1}{\bDelta_{\bcG}}}_* \norm{ \big(\bU^{(1)\top} \bU^{(1)} \big)^{-1} \bU^{(1)\top} \Matricize{1}{\bDelta_{\bcS}} \left( \big(\bU^{(3)\top} \bU^{(3)} \big)^{-1} \bU^{(3)\top} \otimes \big( \bU^{(2)\top} \bU^{(2)} \big)^{-1} \bU^{(2)\top} \right)^\top}_{\op} \\
    & \leq \sqrt{r_1} \norm{\bDelta_{\bcG}}_{\fro} \prod_{k=1}^3 \norm{\bU^{(k)} \big(\bU^{(k)\top} \bU^{(k)} \big)^{-1} }_{\op}  \norm{\Matricize{1}{\bDelta_{\bcS}}}_{\op} ,
\end{align*}
where the second line uses H\"{o}lder's inequality, and the last line uses $\norm{\Matricize{1}{\bDelta_{\bcG}}}_* \leq \sqrt{r_1} \norm{\bDelta_{\bcG}}_{\fro}$ with the fact that $\Matricize{1}{\bDelta_{\bcG}}$ is of rank $r_1$.  To continue, invoke Lemma~\ref{more_perturb} (cf.~\eqref{eq:perturb_spec_b})  as well as \eqref{eq:mk_delta_S_op} to further obtain
\begin{align*}
    |\mfk{B}_{1,2}| &\leq \frac{2\zeta_{t+1} \alpha \sqrt{n_1 n_2 n_3 r_1}}{(1-\epsilon_0)^3} \norm{\bDelta_{\bcG}}_{\fro} =  \frac{16 \alpha \sqrt{\mu^3 r_1^2 r_2 r_3}}{(1-\epsilon_0)^3}  \rho^t \sigma_{\min}(\bcX_\star) \norm{\bDelta_{\bcG}}_{\fro} ,
\end{align*}
where the second equality follows from the choice of $\zeta_{t+1}$ (cf. \eqref{eq:incoherence_contract_zeta_assump}). Finally, with the assumption on the sparsity level $\alpha \leq \frac{c_0 \epsilon_0} { \sqrt{\mu^3 r_1 r_2 r_3 r}}$ for sufficiently small $c_0$ and $\epsilon_0<0.01$, we have
\begin{align}\label{dist_contract_mfk_b_1_2}
    |\mfk{B}_{1,2}| &\leq \frac{16c_0}{   (1-\epsilon_0)^3} \norm{\bDelta_{\bcG}}_{\fro} \epsilon_0 \rho^t \sigma_{\min}(\bcX_\star) \leq 0.15 \norm{\bDelta_{\bcG}}_{\fro} \epsilon_0 \rho^t \sigma_{\min}(\bcX_\star) .
\end{align}
Put \eqref{dist_contract_mfk_b_1_1} and \eqref{dist_contract_mfk_b_1_2} together to see
\begin{align} \label{dist_contract_mfk_b_1}
    \mfk{B}_1 \geq  \sum_{k=1}^3 \norm{\bD^{(k)}}_{\fro}^2  - 0.01 \epsilon_0^2 \rho^{2t} \sigma_{\min}^2(\bcX_\star) - 0.15 \norm{\bDelta_{\bcG}}_{\fro} \epsilon_0 \rho^t \sigma_{\min}(\bcX_\star) .
\end{align}

\paragraph{Bounding $\mfk{B}_2$.}
 
Expanding out the square and applying Cauchy-Schwarz, we can upper bound $\mfk{B}_2$ by
\begin{align*}
    \mfk{B}_2 &\leq 2 \underbrace{\norm{ \left( \big(\bU^{(1)\top} \bU^{(1)}\big)^{-1} \bU^{(1)\top},  \ldots,  \big(\bU^{(3)\top} \bU^{(3)}\big)^{-1} \bU^{(3)\top} \right) \bcdot \left( \big(\bU^{(1)}, \bU^{(2)}, \bU^{(3)} \big) \bcdot \bcG_\star - \bcX_\star \right)}_{\fro}^2}_{=: \mfk{B}_{2,1}} \\
    &\quad +  2\underbrace{\norm{ \left( \big(\bU^{(1)\top} \bU^{(1)}\big)^{-1} \bU^{(1)\top},  \ldots,  \big(\bU^{(3)\top} \bU^{(3)}\big)^{-1} \bU^{(3)\top} \right) \bcdot \bDelta_{\bcS} }_{\fro}^2}_{=:\mfk{B}_{2,2}}.
\end{align*}
$\mfk{B}_{2,1}$ is identical to the term that is bounded in \cite[Section B.4]{tong2021scaling}, which obeys
\begin{align}
    \mfk{B}_{2,1} &\leq 3 \sum_{k=1}^3 \norm{\bD^{(k)}}_{\fro}^2  + C_2 \epsilon_0 \rho^t \dist^2(\bF_t, \bF_\star)  \leq 3  \sum_{k=1}^3 \norm{\bD^{(k)}}_{\fro}^2 + 0.01 \epsilon_0^2 \rho^{2t} \sigma_{\min}^2(\bcX_\star) ,\label{dist_contract_mfk_b_2_1}
\end{align}
where we make use of the induction hypothesis \eqref{eq:dist_contraction} and $C_2 \epsilon_0 \rho^t \leq C_2 \epsilon_0 <0.01$
as long as $\epsilon_0$ is small enough.

For $\mfk{B}_{2,2}$, the matricization of the term in the norm along the first mode is of rank at most $r_1$, so
\begin{align*}
    \mfk{B}_{2,2} &= \norm{\left(((\bU^{(1)})^\top \bU^{(1)})^{-1} (\bU^{(1)})^\top, ((\bU^{(2)})^\top \bU^{(2)})^{-1} (\bU^{(2)})^\top, ((\bU^{(3)})^\top \bU^{(3)})^{-1} (\bU^{(3)})^\top \right) \bcdot \bDelta_{\bcS}}_{\fro}^2 
    \\
    &\leq r_1 \norm{((\bU^{(1)})^\top \bU^{(1)})^{-1} (\bU^{(1)})^\top \Matricize{1}{\bDelta_{\bcS}} \left(((\bU^{(3)})^\top \bU^{(3)})^{-1} (\bU^{(3)})^\top \otimes ((\bU^{(2)})^\top \bU^{(2)})^{-1} (\bU^{(2)})^\top \right)^\top}_{\op}^2
    \\
    &\leq r_1 \norm{((\bU^{(1)})^\top \bU^{(1)})^{-1} (\bU^{(1)})^\top}_{\op}^2 \norm{((\bU^{(2)})^\top \bU^{(2)})^{-1} (\bU^{(2)})^\top}_{\op}^2 \norm{((\bU^{(3)})^\top \bU^{(3)})^{-1} (\bU^{(3)})^\top}_{\op}^2 \norm{\Matricize{1}{\bDelta_{\bcS}}}_{\op}^2 \\
    & \leq \frac{r_1}{(1-\epsilon_0)^6}  \norm{\Matricize{1}{\bDelta_{\bcS}}}_{\op}^2,
\end{align*}
where the last line follows from Lemma~\ref{more_perturb} (cf.~\eqref{eq:perturb_spec_b}). To continue, we use \eqref{eq:mk_delta_S_op} to obtain  
\begin{align*}
    \mfk{B}_{2,2} &\leq \frac{4 \zeta_{t+1}^2 \alpha^2 n_1 n_2 n_3 r_1}{(1-\epsilon_0)^6} = \frac{128 \alpha^2 \mu^3 r_1^2 r_2 r_3}{(1-\epsilon_0)^6}  \rho^{2t}\sigma_{\min}^2(\bcX_\star) ,
\end{align*}
where the second relation follows by the choice of $\zeta_{t+1}$ (cf. \eqref{eq:incoherence_contract_zeta_assump}). Lastly, with the assumption on the sparsity level $\alpha \leq \frac{c_0 \epsilon_0} { \sqrt{\mu^3 r_1 r_2 r_3 r}}$ for sufficiently small $c_0$ and $\epsilon_0<0.01$, we have
\begin{align} \label{dist_contract_mfk_b_2_2}
    \mfk{B}_{2,2} &\leq \frac{128 c_0^2}{ (1-\epsilon_0)^6} \epsilon_0^2 \rho^{2t}\sigma_{\min}^2(\bcX_\star) \leq 0.02 \epsilon_0^2 \rho^{2t}\sigma_{\min}^2(\bcX_\star) .
\end{align}
Combining \eqref{dist_contract_mfk_b_2_1} and \eqref{dist_contract_mfk_b_2_2}, we get 
\begin{align} \label{dist_contract_mfk_b_2}
    \mfk{B}_{2} &\leq 6 \sum_{k=1}^3 \norm{\bD^{(k)}}_{\fro}^2  + 0.06 \epsilon_0^2 \rho^{2t}\sigma_{\min}^2(\bcX_\star) .
\end{align}

\paragraph{Sum up.}
Going back to \eqref{dist_contract_mfk_b_combine}, we can substitute in our bounds for $\mfk{B}_1$ (cf. \eqref{dist_contract_mfk_b_1}) and $\mfk{B}_2$ (cf. \eqref{dist_contract_mfk_b_2}) to get
\begin{align*}
    &\norm{\big( \bQ_t^{(1)})^{-1}, (\bQ_t^{(2)})^{-1}, (\bQ_t^{(3)})^{-1} \big) \bcdot \bcG_{t+1} - \bcG_\star}_{\fro}^2
    \\
    &\leq (1-\eta)^2 \norm{\bDelta_{\bcG}}_{\fro}^2 
     - 2\eta(1-\eta)\left(  \sum_{k=1}^3  \norm{\bD^{(k)}}_{\fro}^2   - 0.01 \epsilon_0^2 \rho^{2t} \sigma_{\min}^2(\bcX_\star) - 0.15 \norm{\bDelta_{\bcG}}_{\fro} \epsilon_0 \rho^t \sigma_{\min}(\bcX_\star)\right) \\
    &\quad + \eta^2 \left(6 \sum_{k=1}^3 \norm{\bD^{(k)}}_{\fro}^2  + 0.06 \epsilon_0^2 \rho^{2t}\sigma_{\min}^2(\bcX_\star)\right) .
\end{align*}
Notice that $-2\eta(1-\eta) \norm{\bD^{(k)}}_{\fro}^2  + 6 \eta^2 \norm{\bD^{(k)}}_{\fro}^2  = -2\eta(1-4\eta) \norm{\bD^{(k)}}_{\fro}^2\leq 0$ whenever $0 < \eta \leq 1/4$, leading to the conclusion that 
\begin{align*}
    \norm{\big(\bQ_t^{(1)})^{-1}, (\bQ_t^{(2)})^{-1}, (\bQ_t^{(3)})^{-1} \big) \bcdot \bcG_{t+1} - \bcG_\star}_{\fro}^2 & \leq (1-\eta)^2 \norm{\bDelta_{\bcG}}_{\fro}^2 + 2 \cdot 0.15 \eta(1-\eta) \norm{\bDelta_{\bcG}}_{\fro} \epsilon_0 \rho^t \sigma_{\min}(\bcX_\star)
    \\
    & \quad + 0.02 \eta(1-\eta) \epsilon_0^2 \rho^{2t} \sigma_{\min}^2(\bcX_\star)  + 0.06 \eta^2 \epsilon_0^2 \rho^{2t}\sigma_{\min}^2(\bcX_\star) .
\end{align*}

\section{Proof of Lemma \ref{incoherence_contract}}
In view of Lemma~\ref{dist_contract_incoherence_contract} and \cite[Lemma 6]{tong2021scaling}, the optimal alignment matrices $\big\{\bQ_t^{(k)}\big\}_{k = 1}^3$ (resp.~$\big\{\bQ_{(t+1)}^{(k)}\big\}_{k = 1}^3)$) between $\bF_t$ (resp.~$\bF_{t+1}$) and $\bF_\star$ exist. 
Fix any $k=1,2,3$. By the triangle inequality, we have
\begin{align} \label{eq:triangle_alignment}
\norm{ \big(\bU_{t+1}^{(k)} \bQ_{t+1}^{(k)} - \bU^{(k)}_\star \big) \bSigma_\star^{(k)}}_{2, \infty} 
    &\leq \norm{\big(\bU_{t+1}^{(k)} \bQ_t^{(k)} - \bU^{(k)}_\star \big) \bSigma_\star^{(k)}}_{2, \infty} + \norm{\bU_{t+1}^{(k)} \big(\bQ_{t+1}^{(k)} - \bQ_t^{(k)}\big) \bSigma_\star^{(k)}}_{2, \infty}.  
\end{align}    
Below we control the two terms in turn. 
% We begin with establishing a bound on $\norm{\big(\bU_{t+1}^{(k)} \bQ_t^{(k)} - \bU^{(k)}_\star \big) \bSigma_\star^{(k)}}_{2, \infty}$ using the previous alignment matrix $\bQ_t^{(k)}$ at the $t$-th iteration, and then translate it to the desired term using the proper alignment matrix $\norm{\big(\bU_{t+1}^{(k)} \bQ_{t+1}^{(k)} - \bU^{(k)}_\star \big) \bSigma_\star^{(k)}}_{2, \infty}$ at the $(t+1)$-th iteration.

\paragraph{Step 1: controlling $\norm{\big(\bU_{t+1}^{(k)} \bQ_t^{(k)} - \bU^{(k)}_\star \big) \bSigma_\star^{(k)}}_{2, \infty}$.} 
By the update rule, we have
 \begin{align*}  
 \big(\bU_{t+1}^{(k)} \bQ_{t}^{(k)} - \bU^{(k)}_\star \big) \bSigma_\star^{(k)}
    &= (1-\eta) \bDelta_{\bU^{(k)}} \bSigma_\star^{(k)} - \eta \left(\Matricize{k}{\bDelta_{\bcS}} + \bU^{(k)}_\star \bDelta_{\Breve{\bU}^{(k)}}^\top\right) \Breve{\bU}^{(k)} \big(\Breve{\bU}^{(k)\top} \Breve{\bU}^{(k)} \big)^{-1} \bSigma_\star^{(k)} .  
\end{align*}
Take the $\ell_{2,\infty}$-norm of both sides and apply the triangle inequality to see that 
\begin{align*}
    \norm{\big(\bU_{t+1}^{(k)} \bQ_t^{(k)} - \bU^{(k)}_\star \big) \bSigma_\star^{(k)}}_{2, \infty} &\leq (1-\eta) \underbrace{\norm{\bDelta_{\bU^{(k)}} \bSigma_\star^{(k)}}_{2, \infty}}_{=: \mfk{C}_{1,k}} + \eta \underbrace{\norm{\Matricize{k}{\bDelta_{\bcS}}  \Breve{\bU}^{(k)} \big(\Breve{\bU}^{(k)\top} \Breve{\bU}^{(k)} \big)^{-1} \bSigma_\star^{(k)}}_{2, \infty}}_{=: \mfk{C}_{2,k}} \\
    &\quad + \eta \underbrace{\norm{\bU^{(k)}_\star \bDelta_{\Breve{\bU}^{(k)}}^\top \Breve{\bU}^{(k)} \big( \Breve{\bU}^{(k)\top} \Breve{\bU}^{(k)} \big)^{-1} \bSigma_\star^{(k)}}_{2, \infty}}_{=: \mfk{C}_{3,k}}.
\end{align*}
We then proceed to bound each term separately.
\begin{itemize}
\item $\mfk{C}_{1,k}$ is captured by the induction hypothesis \eqref{eq:incoh_contraction}, which directly implies
\begin{align} \label{incoherence_contract_a_1}
    \mfk{C}_{1,k} \leq \rho^t \sqrt{\frac{\mu r_k}{n_k}} \sigma_{\min}(\bcX_\star) .
\end{align}
\item We now move on to $\mfk{C}_{2,k}$, which can be bounded by
\begin{align*}
 \mfk{C}_{2,k} &\leq \norm{\Matricize{k}{\bDelta_{\bcS}}}_{2, \infty} \norm{\Breve{\bU}^{(k)}\big(\Breve{\bU}^{(k)\top} \Breve{\bU}^{(k)} \big)^{-1} \bSigma_\star^{(k)}}_{\op}   \leq \frac{1}{(1-\epsilon_0)^3}\norm{\Matricize{k}{\bDelta_{\bcS}}}_{2, \infty} ,
\end{align*}
where the second inequality follows from Lemma~\ref{more_perturb} (cf.~\eqref{eq:perturb_spec_f}). Recall that $\bDelta_{\bcS}$ is $\alpha$-sparse following the choice of $\zeta_{t+1}$, which gives us
$$ \norm{\Matricize{k}{\bDelta_{\bcS}}}_{2, \infty} \leq \sqrt{\frac{\alpha n_1n_2 n_3}{n_k}}\norm{\bDelta_{\bcS}}_\infty \leq 2\sqrt{\frac{\alpha n_1n_2 n_3}{n_k}}\zeta_{t+1} = 16 \sqrt{\frac{\alpha \mu^3 r_1 r_2 r_3}{n_k}}   \rho^t \sigma_{\min}(\bcX_\star) $$
due to Lemma \ref{sparsity_norm_eq}, \eqref{eq:mk_delta_S_op}, and the choice of $\zeta_{t+1}$ (cf. \eqref{eq:incoherence_contract_zeta_assump}).
Plug this into the previous bound to obtain  
\begin{align} \label{incoherence_contract_a_2}
 \mfk{C}_{2,k} &\leq  \frac{16}{(1 - \epsilon_0)^3} \sqrt{\frac{\alpha \mu^3 r_1 r_2 r_3}{n_k}}   \rho^t \sigma_{\min}(\bcX_\star) \leq 0.15 \sqrt{\frac{\mu r_k}{n_k}}  \rho^t \sigma_{\min}(\bcX_\star),
\end{align}
where the last inequality follows from the assumption on the sparsity level $\alpha \leq \frac{c_1  } { \mu^2 r_1 r_2 r_3 }$ with a sufficiently small constant $c_1$.

\item Finally, for $\mfk{C}_{3,k}$, we have the upper bound 
\begin{align*}
    \mfk{C}_{3,k} &\leq \norm{\bU^{(k)}_\star}_{2, \infty} \norm{\bDelta_{\Breve{\bU}^{(k)}}}_{\fro} \norm{\Breve{\bU}^{(k)} \big(\Breve{\bU}^{(k)\top} \Breve{\bU}^{(k)} \big)^{-1} \bSigma_\star^{(k)}}_{\op} \\
    & \leq \sqrt{\frac{3\mu r_k}{n_k}}  \left( 1+ \epsilon_0 + \frac{\epsilon_0^2}{3}\right)\frac{1}{(1-\epsilon_0)^3} \dist(\bF_t, \bF_\star) ,
\end{align*}
where the second inequality follows from the incoherence assumption $\norm{\bU^{(k)}_\star}_{2, \infty} \leq \sqrt{\frac{\mu r_k}{n_k}}$, and Lemma \ref{more_perturb} (cf.~\eqref{eq:perturb_spec_f} and \eqref{eq:perturb_fro_c}). Since $\dist(\bF_t, \bF_\star) \leq \epsilon_0 \rho^t \sigma_{\min}(\bcX_\star)$, we arrive at
\begin{align} \label{incoherence_contract_a_3}
    \mfk{C}_{3,k} &\leq \frac{1}{(1-\epsilon_0)^3}\sqrt{\frac{3 \mu r_k}{n_k}} \left( 1+ \epsilon_0 + \frac{\epsilon_0^2}{3}\right) \epsilon_0 \rho^t \sigma_{\min}(\bcX_\star) \leq 0.02 \sqrt{\frac{\mu r_k}{n_k}} \rho^t \sigma_{\min}(\bcX_\star) 
\end{align}
as long as $\epsilon_0<0.01$.
\end{itemize}

Combining \eqref{incoherence_contract_a_1}, \eqref{incoherence_contract_a_2}, and \eqref{incoherence_contract_a_3} together, we reach the conclusion that 
\begin{align}
    \norm{(\bU_{t+1}^{(k)} \bQ_t^{(k)} - \bU^{(k)}_\star) \bSigma_\star^{(k)}}_{2, \infty} &\leq (1- 0.83\eta)  \sqrt{\frac{\mu r_k}{n_k}} \rho^t \sigma_{\min}(\bcX_\star) . \label{incoherence_contract_1}
\end{align}
In view of the basic relation
$$ \norm{(\bU_{t+1}^{(k)} \bQ_t^{(k)} - \bU^{(k)}_\star) \bSigma_\star^{(k)}}_{2, \infty} \geq  \norm{\bU_{t+1}^{(k)} \bQ_t^{(k)} - \bU^{(k)}_\star}_{2, \infty} \sigma_{\min}(\bcX_\star), $$ 
 this also implies
\begin{align}
    \norm{\bU_{t+1}^{(k)} \bQ_t^{(k)} - \bU^{(k)}_\star}_{2, \infty} \leq (1- 0.83\eta)  \sqrt{\frac{\mu r_k}{n_k}} \rho^t. \label{incoherence_contract_2}
\end{align}

\paragraph{Step 2: controlling $\norm{\bU_{t+1}^{(k)} \big(\bQ_{t+1}^{(k)} - \bQ_t^{(k)}\big) \bSigma_\star^{(k)}}_{2, \infty}$. }
 Observe that
\begin{align*}    
\norm{\bU_{t+1}^{(k)} \big(\bQ_{t+1}^{(k)} - \bQ_t^{(k)}\big) \bSigma_\star^{(k)}}_{2, \infty}     &=  \norm{\bU_{t+1}^{(k)} \bQ_t^{(k)} \big(\bQ_t^{(k)}\big)^{-1} \big(\bQ_{t+1}^{(k)} - \bQ_t^{(k)} \big) \bSigma_\star^{(k)}}_{2, \infty}
    \\
    &\leq  \norm{\bU_{t+1}^{(k)} \bQ_t^{(k)} }_{2, \infty} \norm{\big(\bQ_t^{(k)} \big)^{-1} \big(\bQ_{t+1}^{(k)} - \bQ_t^{(k)} \big) \bSigma_\star^{(k)}}_{\op}.
    \end{align*}

For the first term $\norm{\bU_{t+1}^{(k)} \bQ_t^{(k)} }_{2, \infty}$, we have 
 \begin{align*}  
 \norm{\bU_{t+1}^{(k)} \bQ_t^{(k)} }_{2, \infty} & \leq \norm{\bU_{t+1}^{(k)} \bQ_t^{(k)} - \bU_\star^{(k)}}_{2, \infty} + \norm{\bU_\star^{(k)}}_{2, \infty} \\
 & \leq  \left( (1- 0.83\eta) \rho^t+ 1 \right) \sqrt{\frac{\mu r_k}{n_k}} \leq (2 - 0.83\eta)\sqrt{\frac{\mu r_k}{n_k}},
\end{align*}
where the second line follows from \eqref{incoherence_contract_2} and the incoherence assumption $\norm{\bU_\star^{(k)}}_{2, \infty} \leq \sqrt{\frac{\mu r_k}{n_k}}$.

Moving on to the second term $\norm{\big(\bQ_t^{(k)} \big)^{-1} \big(\bQ_{t+1}^{(k)} - \bQ_t^{(k)} \big) \bSigma_\star^{(k)}}_{\op}$, we plan to invoke Lemma \ref{align_perturb}. Given the assumption of the sparsity level $\alpha$ satisfies the requirement of Lemma \ref{dist_contract_incoherence_contract}, 
following Lemma \ref{dist_contract_incoherence_contract} as well as its proof (cf.~\eqref{eq:dist_contract_t1}), we have
\begin{align*}
    \norm{\big(\bU_{t+1}^{(k)} \bQ_{t+1}^{(k)} - \bU^{(k)}_\star \big) \bSigma_\star^{(k)}}_{\fro}  & \leq \dist(\bF_{t+1}, \bF_{\star}) \leq \epsilon_0 \rho^{t+1} \sigma_{\min}(\bcX_\star),\\
     \norm{\big(\bU_{t+1}^{(k)} \bQ_t^{(k)} - \bU^{(k)}_\star \big) \bSigma_\star^{(k)}}_{\fro} & \leq \epsilon_0 \rho^{t+1} \sigma_{\min}(\bcX_\star),
\end{align*}
which in turn implies
\begin{align}\label{incoherence_contract_3}
    \max\left\{\norm{\bU_{t+1}^{(k)} \bQ_{t}^{(k)} - \bU_{\star}^{(k)} }_{\op}, \frac{\norm{\big(\bU_{t+1}^{(k)} \bQ_{t+1}^{(k)} - \bU^{(k)}_\star \big) \bSigma_\star^{(k)}}_{\op}}{\sigma_{\min}(\bcX_\star)}, \frac{\norm{\big(\bU_{t+1}^{(k)} \bQ_t^{(k)} - \bU^{(k)}_\star \big) \bSigma_\star^{(k)}}_{\op}}{\sigma_{\min}(\bcX_\star)}\right\} \leq \epsilon_0 \rho^{t+1} < 1 .
\end{align}
Setting $\bU := \bU_{t+1}^{(k)}$, $\bU_{\star}: = \bU_{\star}^{(k)}$, $\bQ := \bQ_{t+1}^{(k)}$, $\bar{\bQ} := \bQ_{t}^{(k)}$ and $\bSigma := \bSigma_\star^{(k)}$, \eqref{incoherence_contract_3} demonstrates that Lemma \ref{align_perturb} is applicable, leading to 
%it is verified that by the assumption $\dist(\bF_t,\bF_{\star})\leq \epsilon_0 \rho^t \sigma_{\min}(\bcX_{\star})$ 
%$$\norm{\bU_{t+1}^{(k)} \bQ_{t}^{(k)}  -\bU_{\star}^{(k)} }_{\op}  \leq \norm{\bU_{t+1}^{(k)} \bQ_{t}^{(k)}  -\bU_{\star}^{(k)} }_{\fro} \leq   \frac{ \norm{\big(\bU_{t+1}^{(k)} \bQ_{t}^{(k)}  -\bU_{\star}^{(k)} \big) \bSigma_\star^{(k)}}_{\fro}}{\sigma_{\min}\big(\bSigma_\star^{(k)}\big)} \leq \epsilon <1, $$
\begin{align*}
 \norm{\big(\bQ_t^{(k)} \big)^{-1} \big(\bQ_{t+1}^{(k)} - \bQ_t^{(k)} \big) \bSigma_\star^{(k)}}_{\op} &\leq \frac{\norm{\bU_{t+1}^{(k)} \big(\bQ_{t+1}^{(k)} - \bQ_t^{(k)} \big)   \bSigma_{\star}^{(k)}}_{\op}}{\sigma_{\min}(\bU_\star) - \norm{\bU_{t+1}^{(k)} \bQ_{t}^{(k)}  -\bU_{\star}^{(k)} }_{\op}}  \\
 & \leq \frac{\norm{ \big( \bU_{t+1}^{(k)} \bQ_{t+1}^{(k)} - \bU_{\star}^{(k)} \big)   \bSigma_{\star}^{(k)}}_{\op}+\norm{ \big( \bU_{t+1}^{(k)} \bQ_{t}^{(k)} - \bQ_{\star}^{(k)} \big)   \bSigma_{\star}^{(k)}}_{\op}}{1-\epsilon_0} \\
 & \leq  \frac{2 \epsilon_0 }{1 - \epsilon_0} \rho^{t+1} \sigma_{\min}(\bcX_\star) ,
\end{align*}
where we used \eqref{incoherence_contract_3} in the second and third inequalities.
Combining the above two bounds, we have 
\begin{align}\label{eq:alignment_diff}    
\norm{\bU_{t+1}^{(k)} \big(\bQ_{t+1}^{(k)} - \bQ_t^{(k)}\big) \bSigma_\star^{(k)}}_{2, \infty}     &\leq     ( 2- 0.83\eta   ) \frac{2 \epsilon_0 }{1 - \epsilon_0} \sqrt{\frac{\mu r_k}{n_k}}  \rho^{t+1} \sigma_{\min}(\bcX_\star) .
\end{align}

\paragraph{Step 3: combining the bounds. } 

Plug \eqref{incoherence_contract_1} and \eqref{eq:alignment_diff} into \eqref{eq:triangle_alignment} to get
\begin{align*}
    \norm{(\bU_{t+1}^{(k)} \bQ_{t+1}^{(k)} - \bU^{(k)}_\star) \bSigma_\star^{(k)}}_{2, \infty} 
    &\leq \left[ (1- 0.83\eta)    +   ( 2- 0.83\eta   ) \frac{2\rho \epsilon_0 }{1 - \epsilon_0} \right]  \sqrt{\frac{\mu r_k}{n_k}}  \rho^{t} \sigma_{\min}(\bcX_\star) \\
    & \leq(1.05- 0.84\eta)  \sqrt{\frac{\mu r_k}{n_k}} \rho^t \sigma_{\min}(\bcX_\star) \\
    &\leq (1- 0.45\eta) \sqrt{\frac{\mu r_k}{n_k}} \rho^t \sigma_{\min}(\bcX_\star) = \sqrt{\frac{\mu r_k}{n_k}} \rho^{t+1} \sigma_{\min}(\bcX_\star)
\end{align*}
where the last line follows from $\eta \geq \frac{1}{7}$ and $\rho = 1-0.45\eta$.

% !TEX root = ./Tensor_RPCA.tex

\section{Proof of Lemma \ref{dist_init}}
 
In view of \cite[Lemma 8]{tong2021scaling}, one has
\begin{align} \label{dist_init1}
    \dist(\bF_0, \bF_\star) &\leq (\sqrt{2} + 1)^{\frac{3}{2}} \norm{ \big(\bU_0^{(1)}, \bU_0^{(2)}, \bU_0^{(3)} \big) \bcdot \bcG_0 - \bcX_\star}_{\fro},
\end{align}
where we have used the definition $\bF_0 = (\bU_0^{(1)}, \bU_0^{(2)}, \bU_0^{(3)}) \bcdot \bcG_0 $. As a result, we focus on bounding the term $\norm{ \big(\bU_0^{(1)}, \bU_0^{(2)}, \bU_0^{(3)} \big) \bcdot \bcG_0 - \bcX_\star}_{\fro}$ below. 

Recall from the definition of the HOSVD that $\bcG_0 = ((\bU_0^{(1)})^\top, (\bU_0^{(2)})^\top, (\bU_0^{(3)})^\top) \bcdot (\bcY - \Shrink{\zeta_0}{\bcY})$, which in turn implies 
\begin{align} \label{dist_init2}
	\big(\bU_0^{(1)}, \bU_0^{(2)}, \bU_0^{(3)} \big) \bcdot \bcG_0 = \left(\bU_0^{(1)}\bU_0^{(1)\top}, \bU_0^{(2)} \bU_0^{(2)\top}, \bU_0^{(3)} \bU_0^{(3)\top} \right) \bcdot \big(\bcY - \Shrink{\zeta_0}{\bcY} \big) .
\end{align}
Note that $\bU_0^{(k)}$ has orthonormal columns. We thus 
define $\bP^{(k)} := \bU_0^{(k)} \bU_0^{(k)\top}$, the orthogonal projection onto the column space of $\bU_0^{(k)}$.
This allows us to rewrite the squared Frobenius norm in \eqref{dist_init1} as
\begin{align*}
    \norm{\big(\bU_0^{(1)}, \bU_0^{(2)}, \bU_0^{(3)}\big) \bcdot \bcG_0 - \bcX_\star}_{\fro}^2 &= \norm{ \big(\bP^{(1)}, \bP^{(2)}, \bP^{(3)} \big) \bcdot \big(\bcY - \Shrink{\zeta_0}{\bcY} \big) - \bcX_\star}_{\fro}^2 .
\end{align*}
Since $\bcX_\star$ can be decomposed into a sum of orthogonal projections and its orthogonal complements, namely,
\begin{align*}
\bcX_\star &= \big(\bP^{(1)}, \bP^{(2)}, \bP^{(3)} \big) \bcdot \bcX_\star + \big(\bI_{n_1} - \bP^{(1)}, \bP^{(2)}, \bP^{(3)} \big) \bcdot \bcX_\star \\
&\qquad + \big(\bI_{n_1}, \bI_{n_2} - \bP^{(2)}, \bP^{(3)} \big) \bcdot \bcX_\star + \big(\bI_{n_1}, \bI_{n_2}, \bI_{n_3} - \bP^{(3)} \big) \bcdot \bcX_\star, 
\end{align*}
we have the following identity
\begin{align}
    &\norm{\big(\bU_0^{(1)}, \bU_0^{(2)}, \bU_0^{(3)} \big) \bcdot \bcG_0 - \bcX_\star}_{\fro}^2 
   = \norm{\big(\bP^{(1)}, \bP^{(2)}, \bP^{(3)} \big) \bcdot \big(\bcY - \Shrink{\zeta_0}{\bcY} - \bcX_\star \big)}_{\fro}^2 \nonumber \\
   &+ \norm{\big(\bI_{n_1} - \bP^{(1)}, \bP^{(2)}, \bP^{(3)} \big) \bcdot \bcX_\star}_{\fro}^2 
    + \norm{\big(\bI_{n_1}, \bI_{n_2} - \bP^{(2)}, \bP^{(3)} \big) \bcdot \bcX_\star}_{\fro}^2 + \norm{\big(\bI_{n_1}, \bI_{n_2}, \bI_{n_3} - \bP^{(3)} \big) \bcdot \bcX_\star}_{\fro}^2 . \label{dist_init3}
\end{align}
In what follows, we bound each term respectively. 

\paragraph{Bounding the first term. }
Matricize along the first mode and change to the operator norm to obtain 
\begin{align}
\norm{\big(\bP^{(1)}, \bP^{(2)}, \bP^{(3)} \big) \bcdot \big(\bcY - \Shrink{\zeta_0}{\bcY} - \bcX_\star \big)}_{\fro} &\leq \sqrt{r_1} \norm{\bP^{(1)} \Matricize{1}{\bcY - \Shrink{\zeta_0}{\bcY} - \bcX_\star} \big(\bP^{(3)} \otimes \bP^{(2)} \big)^\top}_{\op} \nonumber \\
 &\leq \sqrt{r_1} \norm{\Matricize{1}{\bcY - \Shrink{\zeta_0}{\bcY} - \bcX_\star}}_{\op} \nonumber \\
    &= \sqrt{r_1} \norm{\Matricize{1}{\bcS_\star - \Shrink{\zeta_0}{\bcY}}}_{\op}, \label{dist_init3_bound1}
\end{align}
where the last relation holds due to the definition of $\bcY$.

\paragraph{Bounding the remaining three terms. }
We present the bound on the second term, while the remaining two can be bounded in a similar fashion. Matricize along the first mode, and in view of the fact that $\|\bP^{(k)}\| \leq 1$, we have  
\begin{align*}
 \norm{\big(\bI_{n_1} - \bP^{(1)}, \bP^{(2)}, \bP^{(3)} \big) \bcdot \bcX_\star}_{\fro} & \leq    \norm{\big(\bI_{n_1} - \bP^{(1)} \big) \Matricize{1}{\bcX_\star}}_{\fro} \\
&\leq \sqrt{r_1} \norm{ \big(\bI_{n_1} - \bP^{(1)} \big) \Matricize{1}{\bcX_\star}}_{\op} \\
    &= \sqrt{r_1} \norm{ \big(\bI_{n_1} - \bP^{(1)} \big) \Matricize{1}{\bcY - \Shrink{\zeta_0}{\bcY} - \bcS_\star + \Shrink{\zeta_0}{\bcY}}}_{\op} \\
    &\leq \sqrt{r_1} \norm{   \Matricize{1}{\bcS_\star - \Shrink{\zeta_0}{\bcY}}}_{\op} + \sqrt{r_1} \norm{ \big(\bI_{n_1} - \bP^{(1)} \big) \Matricize{1}{\bcY - \Shrink{\zeta_0}{\bcY}}}_{\op} .
\end{align*}
To continue, note that  $\bP^{(1)} \Matricize{1}{\bcY - \Shrink{\zeta_0}{\bcY}}$ is the best rank-$r_1$ approximation to $\Matricize{1}{\bcY - \Shrink{\zeta_0}{\bcY}}$, which implies
\begin{align*}
   \norm{ \big(\bI_{n_1} - \bP^{(1)} \big) \Matricize{1}{\bcY - \Shrink{\zeta_0}{\bcY}}}_{\op}  &\leq   \sigma_{r_1 + 1}(\Matricize{1}{\bcY - \Shrink{\zeta_0}{\bcY}}) \\
   & \leq \sigma_{r_1 + 1}(\Matricize{1}{\bcX_\star}) + \norm{\Matricize{1}{\bcS_\star - \Shrink{\zeta_0}{\bcY}}}_{\op} = \norm{\Matricize{1}{\bcS_\star - \Shrink{\zeta_0}{\bcY}}}_{\op},
\end{align*}
where the last line follows from Weyl's inequality and the fact that $\Matricize{1}{\bcX_\star}$ has rank $r_1$. Plug this into the previous inequality to obtain
\begin{align} \label{dist_init3_bound2}
 \norm{\big(\bI_{n_1} - \bP^{(1)}, \bP^{(2)}, \bP^{(3)} \big) \bcdot \bcX_\star}_{\fro} &\leq 2\sqrt{r_1} \norm{\Matricize{1}{\bcS_\star - \Shrink{\zeta_0}{\bcY}}}_{\op} .
\end{align}

Plug our bounds~\eqref{dist_init3_bound1} and \eqref{dist_init3_bound2} into \eqref{dist_init3} to obtain
\begin{align}
    \norm{ \big(\bU_0^{(1)}, \bU_0^{(2)}, \bU_0^{(3)}\big) \bcdot \bcG_0 - \bcX_\star}_{\fro}^2 &\leq r_1 \norm{\Matricize{1}{\bcS_\star - \Shrink{\zeta_0}{\bcY}}}_{\op}^2 + 4 r_1 \norm{\Matricize{1}{\bcS_\star - \Shrink{\zeta_0}{\bcY}}}_{\op}^2 \nonumber \\
    &\quad + 4 r_2 \norm{\Matricize{2}{\bcS_\star - \Shrink{\zeta_0}{\bcY}}}_{\op}^2 +   4 r_3 \norm{\Matricize{3}{\bcS_\star - \Shrink{\zeta_0}{\bcY}}}_{\op}^2 .\label{eq:bound_init_hosvd}
\end{align}
It then boils down to controlling $\norm{\Matricize{k}{\bcS_\star - \Shrink{\zeta_0}{\bcY}}}_{\op}$ for $k=1,2,3$. 
With our choice of $\zeta_0$ (i.e. $\| \bcX_\star\|_{\infty} \leq \zeta_0 \leq 2 \| \bcX_\star\|_{\infty}$), setting $\bcX = \bm{0}$ in Lemma \ref{corrupt_iter} guarantees that $\Matricize{k}{\bcS_\star - \Shrink{\zeta_0}{\bcY}}$ is $\alpha$-sparse for all $k$.
Hence, we can apply Lemma \ref{sparsity_norm_eq} to arrive at 
\begin{align}\label{eq:mk_delta_init}
\norm{\Matricize{k}{\bcS_\star - \Shrink{\zeta_0}{\bcY}}}_{\op} \leq \alpha \sqrt{n_1 n_2 n_3}\norm{ \bcS_\star - \Shrink{\zeta_0}{\bcY} }_{\infty}  \leq  2\alpha \sqrt{n_1 n_2 n_3} \zeta_{0}\leq 4\alpha \sqrt{n_1 n_2 n_3}   \| \bcX_\star\|_{\infty},
\end{align}
where the penultimate inequality follows from Lemma \ref{corrupt_iter} (cf.~\eqref{eq:sparse_bound}). Plug \eqref{eq:mk_delta_init} into \eqref{eq:bound_init_hosvd} to obtain 
\begin{align*}
    \norm{\big(\bU_0^{(1)}, \bU_0^{(2)}, \bU_0^{(3)} \big) \bcdot \bcG_0 - \bcX_\star}_{\fro}^2 &\leq 208 \alpha^2 n_1 n_2 n_3 r  \| \bcX_\star\|_{\infty}^2 \leq 208 \alpha^2 \mu^3 r_1 r_2 r_3 r  \kappa^2 \sigma_{\min}^2(\bcX_\star) ,
\end{align*}
where the second inequality follows from Lemma~\ref{x_star_inf_bound}.
Under the assumption that $\alpha \leq \frac{c_0}{\sqrt{\mu^3 r_1 r_2 r_3 r}  \kappa}$, it follows
\begin{align*}
    \norm{ \big(\bU_0^{(1)}, \bU_0^{(2)}, \bU_0^{(3)} \big) \bcdot \bcG_0 - \bcX_\star}_{\fro}^2 &\leq 208 c_0^2 \sigma_{\min}^2(\bcX_\star).
\end{align*}
This combined with \eqref{dist_init1} finishes the proof. 

% \begin{align*} 
%     \dist(\bF_0, \bF_\star) \leq (\sqrt{2} + 1)^{\frac{3}{2}} \sqrt{208} c_0 \sigma_{\min}(\bcX_\star) \leq 54.1 c_0 \sigma_{\min}(\bcX_\star) .
% \end{align*}
 
%While \eqref{hosvd_bound} implies the HOSVD does not guarantee that $(\bU_{0}^{(1)}, \bU_{0}^{(2)}, \bU_{0}^{(3)}) \bcdot \bcG_0$ is the best rank-$\br$ approximation to $\bcY - \Shrink{\zeta_0}{\bcY}$

\section{Proof of Lemma \ref{incoherence_init}}

We provide the control on the first mode, as the other two modes can be 
bounded using the same arguments. 

We begin with a useful decomposition of the quantity we care about, whose proof will be supplied in the end of this section:
\begin{align}\label{eq:useful-decomposition}
    \bDelta_{\bU^{(1)}} \bSigma_\star^{(1)} &= \Matricize{1}{\bcS_\star - \Shrink{\zeta_0}{\bcY}} \Matricize{1}{\bcY - \Shrink{\zeta_0}{\bcY}}^\top \bU_{0}^{(1)} (\bQ_0^{(1)})^{-\top} (\Breve{\bU}^{(1)\top} \Breve{\bU}^{(1)})^{-1} \bSigma_\star^{(1)} \nonumber \\
    &\quad + \bU_\star^{(1)} \bDelta_{\Breve{\bU}^{(1)}}^\top \Matricize{1}{\bcY - \Shrink{\zeta_0}{\bcY}}^\top \bU_{0}^{(1)} (\bQ_0^{(1)})^{-\top} (\Breve{\bU}^{(1)\top} \Breve{\bU}^{(1)})^{-1} \bSigma_\star^{(1)} .
\end{align} 
Taking the $\ell_{2,\infty}$-norm and using the triangle inequality, we obtain
\begin{align}
    \norm{\bDelta_{\bU^{(1)}} \bSigma_\star^{(1)}}_{2, \infty} &\leq \underbrace{\norm{\bU_\star^{(1)} \bDelta_{\Breve{\bU}^{(1)}}^\top \Matricize{1}{\bcY - \Shrink{\zeta_0}{\bcY}}^\top \bU_{0}^{(1)} (\bQ_0^{(1)})^{-\top} (\Breve{\bU}^{(1)\top} \Breve{\bU}^{(1)})^{-1} \bSigma_\star^{(1)} }_{2, \infty}}_{=: \mfk{A}_1} \nonumber \\
    &\quad + \underbrace{\norm{\Matricize{1}{\bcS_\star - \Shrink{\zeta_0}{\bcY}} \Matricize{1}{\bcY - \Shrink{\zeta_0}{\bcY}}^\top \bU_{0}^{(1)} (\bQ_0^{(1)})^{-\top} (\Breve{\bU}^{(1)\top} \Breve{\bU}^{(1)})^{-1} \bSigma_\star^{(1)}}_{2, \infty}}_{ =:\mfk{A}_2} \label{eq:incoherence_init_decomp}.
\end{align}
We now proceed to bound these two terms separately.

\paragraph{Step 1: bounding $\mfk{A}_1$.} To begin, note that
\begin{align*}
    \mfk{A}_1 &\leq \norm{\bU_\star^{(1)}}_{2, \infty} \norm{\bDelta_{\Breve{\bU}^{(1)}}}_{\op} \norm{\Matricize{1}{\bcY - \Shrink{\zeta_0}{\bcY}}^\top \bU_0^{(1)} (\bQ_0^{(1)})^{-\top} (\Breve{\bU}^{(1)\top} \Breve{\bU}^{(1)})^{-1} \bSigma_\star^{(1)}}_{\op} \\
    &= \norm{\bU_\star^{(1)}}_{2, \infty} \norm{\bDelta_{\Breve{\bU}^{(1)}}}_{\op} \norm{\Breve{\bU}^{(1)} (\Breve{\bU}^{(1)\top} \Breve{\bU}^{(1)})^{-1} \bSigma_\star^{(1)}}_{\op} ,
\end{align*}
where in the second line we have used the relation
$$ \Breve{\bU}_0^{(1)\top} \Breve{\bU}_0^{(1)} = \bU_{0}^{(1) \top} \Matricize{1}{\bcY - \Shrink{\zeta_0}{\bcY}} \Matricize{1}{\bcY - \Shrink{\zeta_0}{\bcY}}^\top \bU_{0}^{(1)}$$ 
given by Lemma \ref{trunc_hosvd} (cf. \eqref{eq:trunc_hosvd2}), and the short-hand notation in \eqref{eq:short_hand}. Invoking Lemma \ref{more_perturb} (cf. \eqref{eq:perturb_spec_f}) and the incoherence assumption $\norm{\bU_\star^{(1)}}_{2, \infty} \leq \sqrt{\frac{\mu r_1}{n_1}}$, we arrive at
\begin{align*}
    \mfk{A}_1 &\leq \frac{1}{(1-\epsilon_0)^3} \sqrt{\frac{\mu r_1}{n_1}} \norm{\bDelta_{\Breve{\bU}^{(1)}}}_{\op}.
\end{align*}
Furthermore, by Lemma \ref{more_perturb} (cf. \eqref{eq:perturb_fro_c}), it holds that 
\begin{align*}
    \norm{\bDelta_{\Breve{\bU}^{(1)}}}_{\op} \leq \norm{\bDelta_{\Breve{\bU}^{(1)}}}_{\fro} \leq 2\left(1+\epsilon_0+\frac{\epsilon_0^2}{3} \right)\dist(\bF_0, \bF_\star) \leq 2\left(1+\epsilon_0+\frac{\epsilon_0^2}{3} \right)\epsilon_0 \sigma_{\min}(\bcX_\star) ,
\end{align*}
leading to
\begin{align} \label{eq:incoherence_init_a_1}
    \mfk{A}_1 \leq \frac{2 \epsilon_0}{(1-\epsilon_0)^3}  \left(1+\epsilon_0+\frac{\epsilon_0^2}{3} \right) \sqrt{\frac{\mu r_1}{n_1}} \sigma_{\min}(\bcX_\star) \leq 0.57 \sqrt{\frac{\mu r_1}{n_1}} \sigma_{\min}(\bcX_\star),
\end{align}
where the last inequality holds as long as $\epsilon_0 \leq  0.15$.

\paragraph{Step 2: bounding $\mfk{A}_2$.}  
For $\mfk{A}_2$, we have
\begin{align*}
    \mfk{A}_2 &\leq \norm{\Matricize{1}{\bcS_\star - \Shrink{\zeta_0}{\bcY}}}_{1, \infty}  \norm{\Matricize{1}{\bcY - \Shrink{\zeta_0}{\bcY}}^\top \bU_0^{(1)} (\bQ_0^{(1)})^{-\top} (\bSigma_\star^{(1)})^{-1}}_{2, \infty} \norm{\bSigma_\star^{(1)} (\Breve{\bU}^{(1)\top} \Breve{\bU}^{(1)})^{-1} \bSigma_\star^{(1)}}_{\op} \\
    &\leq \frac{\alpha n_2 n_3}{(1-\epsilon_0)^6} \norm{\bcS_\star - \Shrink{\zeta_0}{\bcX_\star +\bcS_\star}}_\infty  \norm{\Matricize{1}{\bcY - \Shrink{\zeta_0}{\bcY}}^\top \bU_0^{(1)} (\bQ_0^{(1)})^{-\top} (\bSigma_\star^{(1)})^{-1}}_{2, \infty}.
\end{align*}
where we have used the $\alpha$-sparsity of $\Matricize{1}{\bcS_\star - \Shrink{\zeta_0}{\bcY}}$ given by Lemma \ref{corrupt_iter} and Lemma \ref{more_perturb} (cf. \eqref{eq:perturb_spec_h}) in the second line.
Apply Lemma \ref{corrupt_iter} with $\bcX = \bm{0}$ to get
\begin{align} \label{eq:incoherence_init_a_2_1}
    \mfk{A}_2 &\leq \frac{2 \alpha n_2n_3}{(1-\epsilon_0)^6} \zeta_0 \norm{\Matricize{1}{\bcY - \Shrink{\zeta_0}{\bcY}}^\top \bU_0^{(1)} (\bQ_0^{(1)})^{-\top} (\bSigma_\star^{(1)})^{-1}}_{2, \infty} \nonumber \\
    & \leq \frac{4 c_0 }{(1-\epsilon_0)^6}  \sqrt{\frac{n_2 n_3}{\mu n_1r_1r_2r_3}} \sigma_{\min}(\bcX_\star)  \norm{\Matricize{1}{\bcY - \Shrink{\zeta_0}{\bcY}}^\top \bU_0^{(1)} (\bQ_0^{(1)})^{-\top} (\bSigma_\star^{(1)})^{-1}}_{2, \infty}  ,
\end{align}
where the second line follows from $\zeta_0 \leq 2 \| \bcX_\star\|_{\infty} \leq 2\sqrt{\frac{\mu^3 r_1 r_2 r_3}{n_1 n_2 n_3}} \kappa \sigma_{\min}(\bcX_\star)$ (cf. Lemma~\ref{x_star_inf_bound}), as well as the assumption
$\alpha \leq \frac{c_0}{\mu^2 r_1r_2 r_3  \kappa}$. To continue, 
\begin{align*}
    \norm{\Matricize{1}{\bcY - \Shrink{\zeta_0}{\bcY}}^\top \bU_0^{(1)} (\bQ_0^{(1)})^{-\top} (\bSigma_\star^{(1)})^{-1}}_{2, \infty} &\leq r_1 \norm{\Matricize{1}{\bcY - \Shrink{\zeta_0}{\bcY}}^\top \bU_0^{(1)} (\bQ_0^{(1)})^{-\top} (\bSigma_\star^{(1)})^{-1}}_\infty
    \\
    &\leq r_1 \norm{(\bSigma_\star^{(1)})^{-1} (\bQ_0^{(1)})^{-1} \bU_0^{(1) \top} \Matricize{1}{\bcY - \Shrink{\zeta_0}{\bcY}}}_{2, \infty}
    \\
    &= r_1 \norm{(\bSigma_\star^{(1)})^{-1} \Breve{\bU}^{(1)\top} \Breve{\bU}^{(1)} (\bSigma_\star^{(1)})^{-1}}_{\infty}^{1/2}
    \\
    &= r_1 \norm{\Breve{\bU}^{(1)}(\bSigma_\star^{(1)})^{-1}}_{2, \infty} ,
\end{align*}
where we have used Lemma~\ref{trunc_hosvd} (cf. \eqref{eq:trunc_hosvd2}) and the relation $\norm{\bA}_{2, \infty}^2 = \norm{\bA \bA^\top}_\infty$ in the first equality.
Using the definition of $\Breve{\bU}^{(k)}$ from \eqref{eq:matricization_property}, we have 
\begin{align*}
 \norm{\Breve{\bU}^{(1)}(\bSigma_\star^{(1)})^{-1}}_{2, \infty} &=    \norm{(\bU^{(3)} \otimes \bU^{(2)}) \Matricize{1}{\bcG}^\top  \big(\bSigma_\star^{(1)} \big)^{-1}}_{2, \infty} \leq   \norm{\bU^{(3)} \otimes \bU^{(2)}}_{2, \infty} \norm{\Matricize{1}{\bcG}^\top  \big(\bSigma_\star^{(1)} \big)^{-1}}_{\op} \\
 & \leq  \norm{\bU^{(3)}}_{2, \infty} \norm{\bU^{(2)}}_{2, \infty} \left(\norm{\Matricize{1}{\bDelta_{\bcG}}^\top  (\bSigma_\star^{(1)})^{-1}}_{\op} + \norm{\Matricize{1}{\bcG_\star}^\top  (\bSigma_\star^{(1)})^{-1}}_{\op} \right).
\end{align*}
Applying the triangle inequality on the decompositions $\bU^{(k)} = \bDelta_{\bU^{(k)}} + \bU_\star^{(k)}$ and $\bcG = \bDelta_{\bcG} + \bcG_\star$, and with Lemma \ref{more_perturb} (cf. \eqref{eq:perturb_spec_a}) and $\norm{\Matricize{1}{\bcG_\star}^\top  (\bSigma_\star^{(1)})^{-1}}_{\op}=1$, it follows from the above inequalities that
\begin{align*}
    &\norm{\Matricize{1}{\bcY - \Shrink{\zeta_0}{\bcY}}^\top \bU_0^{(1)} (\bQ_0^{(1)})^{-\top} (\bSigma_\star^{(1)})^{-1}}_{2, \infty} \\
    &\leq (1 + \epsilon_0) r_1 \left(\norm{\bDelta_{\bU^{(3)}}}_{2, \infty} + \norm{\bU_\star^{(3)}}_{2, \infty} \right) \left(\norm{\bDelta_{\bU^{(2)}}}_{2, \infty} + \norm{\bU_\star^{(2)}}_{2, \infty} \right) \\
        &\leq (1 + \epsilon_0) r_1 \left(\frac{\norm{\bDelta_{\bU^{(3)}} \bSigma_\star^{(3)}}_{2, \infty}}{\sigma_{\min}(\bcX_\star)} + \sqrt{\frac{\mu r_3}{n_3}} \right) \left(\frac{\norm{\bDelta_{\bU^{(2)}} \bSigma_\star^{(2)}}_{2, \infty}}{\sigma_{\min}(\bcX_\star)} + \sqrt{\frac{\mu r_2}{n_2}} \right) ,
\end{align*}
where the last line uses the relationship $\norm{\bDelta_{\bU^{(k)}}}_{2, \infty} \leq \frac{\norm{\bDelta_{\bU^{(k)}} \bSigma_\star^{(k)}}_{2, \infty}}{\sigma_{\min}(\bSigma_\star^{(k)})}$ and the inequality $\norm{\bU_\star^{(k)}}_{2, \infty} \leq \sqrt{\frac{\mu r_k}{n_k}}$.
Plug this back into \eqref{eq:incoherence_init_a_2_1} to arrive at
\begin{align} 
    \mfk{A}_2 
    %& \leq   \frac{4 (1 + \epsilon_0) c_0 }{(1-\epsilon_0)^6 \sigma_{\min}(\bcX_\star) }  \sqrt{\frac{r_1}{\mu n_1 }}   \left( \sqrt{\frac{  n_3}{r_3}} \norm{\bDelta_{\bU^{(3)}} \bSigma_\star^{(3)}}_{2, \infty} + \sqrt{ \mu  } \sigma_{\min}(\bcX_\star) \right) \left( \sqrt{\frac{  n_2}{r_2}} \norm{\bDelta_{\bU^{(2)}} \bSigma_\star^{(2)}}_{2, \infty} + \sqrt{ \mu } \sigma_{\min}(\bcX_\star)  \right) \notag \\
    & \leq \frac{0.02}{ \sigma_{\min}(\bcX_\star)} \sqrt{\frac{r_1}{\mu n_1}}  \left( \sqrt{\frac{  n_3}{r_3}} \norm{\bDelta_{\bU^{(3)}} \bSigma_\star^{(3)}}_{2, \infty} + \sqrt{ \mu  } \sigma_{\min}(\bcX_\star) \right) \left( \sqrt{\frac{  n_2}{r_2}} \norm{\bDelta_{\bU^{(2)}} \bSigma_\star^{(2)}}_{2, \infty} + \sqrt{ \mu } \sigma_{\min}(\bcX_\star)  \right), \label{eq:incoherence_init_a_2}
\end{align}
where we simplified the constants using the assumption  $\epsilon_0 = 54.1 c_0 \leq  0.15$.

\paragraph{Step 3: Putting things together.} 
Combining \eqref{eq:incoherence_init_a_1}, \eqref{eq:incoherence_init_a_2}, and~\eqref{eq:incoherence_init_decomp}, we have
\begin{align*}
 & \frac{1 }{ \sigma_{\min}(\bcX_\star)  }  \sqrt{\frac{n_1}{\mu r_1}}\norm{\bDelta_{\bU^{(1)}} \bSigma_\star^{(1)}}_{2, \infty} \\
& \leq  0.57   +    0.02   \left(   \frac{1 }{ \sigma_{\min}(\bcX_\star)  }  \sqrt{\frac{  n_3}{\mu r_3}} \norm{\bDelta_{\bU^{(3)}} \bSigma_\star^{(3)}}_{2, \infty} + 1  \right) \left(   \frac{1 }{ \sigma_{\min}(\bcX_\star)  }  \sqrt{\frac{  n_2}{\mu r_2}} \norm{\bDelta_{\bU^{(2)}} \bSigma_\star^{(2)}}_{2, \infty} +  1 \right) .
\end{align*}
Similar inequalities hold for $   \sqrt{\frac{n_2}{r_2}} \norm{\bDelta_{\bU^{(2)}} \bSigma_\star^{(2)}}_{2, \infty}$ and $    \sqrt{\frac{n_3}{r_3}} \norm{\bDelta_{\bU^{(3)}} \bSigma_\star^{(3)}}_{2, \infty}$. 
Taking the maximum as $\cI: = \max_k   \frac{1 }{ \sigma_{\min}(\bcX_\star)  }  \sqrt{\frac{n_k}{\mu r_k}}\norm{\bDelta_{\bU^{(k)}} \bSigma_\star^{(k)}}_{2, \infty} $, we have
\begin{align*}
\cI  \leq    0.57   +    0.02 (\cI + 1)^2  \qquad \Longrightarrow \qquad \cI \leq 0.62,
\end{align*}
and consequently, $    \max_k \sqrt{\frac{n_k}{r_k}} \norm{\bDelta_{\bU^{(k)}} \bSigma_\star^{(k)}}_{2, \infty} <  \sqrt{ \mu}  \sigma_{\min}(\bcX_\star)$ as claimed.

We are left with proving the decomposition~\eqref{eq:useful-decomposition}.
\paragraph{Proof of~\eqref{eq:useful-decomposition}}. 
By the assumption
$$\alpha \leq \frac{c_0}{\mu^2 r_1r_2 r_3  \kappa} \leq \frac{c_0}{\sqrt{\mu^3 r_1 r_2 r_3 r }  \kappa}, $$
in view of Lemma \ref{dist_init}, we have that
\begin{equation}\label{eq:dist_cor}
\dist(\bF_0, \bF_\star)<54.1 c_0 \sigma_{\min}(\bcX_\star) =:  \epsilon_0 \sigma_{\min}(\bcX_\star) \leq 0.15  \sigma_{\min}(\bcX_\star), 
\end{equation}
where $\epsilon_0 =54.1 c_0 \leq 0.15$ as long as  $c_0 > 0$ is small enough. Given $\dist(\bF_0, \bF_\star) < \sigma_{\min}(\bcX_\star)$, we know that $\{\bQ_0^{(k)}\}_{k=1}^3$, the optimal alignment matrices between $\bF_0$ and $\bF_\star$ exist by \cite[Lemma 6]{tong2021scaling}. 

We now aim to control the incoherence. We begin with the equality guaranteed by Lemma \ref{trunc_hosvd} (cf. \eqref{eq:trunc_hosvd2}),
\begin{align*}
    \bU_{0}^{(k)} \Breve{\bU}_0^{(k)\top} \Breve{\bU}_0^{(k)} = \Matricize{k}{\bcY - \Shrink{\zeta_0}{\bcY}} \Matricize{k}{\bcY - \Shrink{\zeta_0}{\bcY}}^\top \bU_{0}^{(k)}.
\end{align*}
Again, we will focus on the case with $k=1$; the other modes will follow from the same arguments. Given that $\big(\Breve{\bU}_0^{(1)\top} \Breve{\bU}_0^{(1)} \big)^{-1}$ exists since $\Breve{\bU}_0^{(1)\top} \Breve{\bU}_0^{(1)} = \Matricize{1}{\bcG_0} \Matricize{1}{\bcG_0}^\top$ is positive definite, right-multiplying $\big(\Breve{\bU}_0^{(1)\top} \Breve{\bU}_0^{(1)} \big)^{-1}$ on both sides of the above equation yields
\begin{align*}
    \bU_{0}^{(1)} &= \Matricize{1}{\bcY - \Shrink{\zeta_0}{\bcY}} \Matricize{1}{\bcY - \Shrink{\zeta_0}{\bcY}}^\top \bU_{0}^{(1)} (\Breve{\bU}_0^{(1)\top} \Breve{\bU}_0^{(1)})^{-1}.
\end{align*}
Plug in the relation $\bcY = \bcX_\star + \bcS_\star$ to get
\begin{align*}
    \bU_{0}^{(1)} &= \Matricize{1}{\bcS_\star - \Shrink{\zeta_0}{\bcY}} \Matricize{1}{\bcY - \Shrink{\zeta_0}{\bcY}}^\top \bU_{0}^{(1)} (\Breve{\bU}_0^{(1)\top} \Breve{\bU}_0^{(1)})^{-1} \\
    &\quad + \Matricize{1}{\bcX_\star} \Matricize{1}{\bcY - \Shrink{\zeta_0}{\bcY}}^\top \bU_{0}^{(1)} (\Breve{\bU}_0^{(1)\top} \Breve{\bU}_0^{(1)})^{-1}
    \\  
    &= \Matricize{1}{\bcS_\star - \Shrink{\zeta_0}{\bcY}} \Matricize{1}{\bcY - \Shrink{\zeta_0}{\bcY}}^\top \bU_{0}^{(1)} (\Breve{\bU}_0^{(1)\top} \Breve{\bU}_0^{(1)})^{-1} \\
    &\quad + \bU_\star^{(1)} \Breve{\bU}_\star^{(1) \top} \Matricize{1}{\bcY - \Shrink{\zeta_0}{\bcY}}^\top \bU_{0}^{(1)} (\Breve{\bU}_0^{(1)\top} \Breve{\bU}_0^{(1)})^{-1} .
\end{align*}
Subtracting $\bU_\star^{(1)} (\bQ_0^{(1)})^{-1}$ on both sides gets us
\begin{align}
    \bU_{0}^{(1)} - \bU_\star^{(1)} (\bQ_0^{(1)})^{-1} &= \Matricize{1}{\bcS_\star - \Shrink{\zeta_0}{\bcY}} \Matricize{1}{\bcY - \Shrink{\zeta_0}{\bcY}}^\top \bU_{0}^{(1)} (\Breve{\bU}_0^{(1)\top} \Breve{\bU}_0^{(1)})^{-1} \nonumber \\
    &\quad + \bU_\star^{(1)} \Breve{\bU}_\star^{(1) \top} \Matricize{1}{\bcY - \Shrink{\zeta_0}{\bcY}}^\top \bU_{0}^{(1)} (\Breve{\bU}_0^{(1)\top} \Breve{\bU}_0^{(1)})^{-1} - \bU_\star^{(1)} (\bQ_0^{(1)})^{-1} \label{eq:incoherence_init1}.
\end{align}
Observe that 
\begin{align*}
    \bU_\star^{(1)} (\bQ_0^{(1)})^{-1} &= \bU_\star^{(1)} (\bQ_0^{(1)})^{-1} \Breve{\bU}_0^{(1)\top} \Breve{\bU}_0^{(1)} (\Breve{\bU}_0^{(1)\top} \Breve{\bU}_0^{(1)})^{-1} \\
    &= \bU_\star^{(1)} (\bQ_0^{(1)})^{-1} \Breve{\bU}_0^{(1) \top} \Matricize{1}{\bcY - \Shrink{\zeta_0}{\bcY}}^\top \bU_0^{(1)} (\Breve{\bU}_0^{(1)\top} \Breve{\bU}_0^{(1)})^{-1},
\end{align*}
where we have used Lemma \ref{trunc_hosvd} (cf. \eqref{eq:trunc_hosvd1}) in the second step.
Plug this result into \eqref{eq:incoherence_init1}, multiply both sides with $\bQ_0^{(1)} \bSigma_\star^{(1)}$, and recall the short-hand notation in \eqref{eq:short_hand} initiated at $t=0$ to arrive at the 
claimed decomposition.

% !TEX root = ./Tensor_RPCA.tex
% Intermediate Lemmas that build up to the main lemmas

\section{Technical lemmas}\label{sec:technical_lemmas}

This section collects several technical lemmas that are useful in the main proofs. 

\subsection{Tensor algebra}

We start with a simple bound on the element-wise maximum norm of an incoherent tensor. 

\begin{lemma}\label{x_star_inf_bound}
Suppose that $\bcX_\star = (\bU_\star^{(1)}, \bU_\star^{(2)}, \bU_\star^{(3)}) \bcdot \bcG_\star \in \RR^{n_1 \times n_2 \times n_3}$ have multilinear rank $\br = (r_1, r_2, r_3)$ and is $\mu$-incoherent.
Then one has 
$
    \norm{\bcX_\star}_\infty \leq \sqrt{\frac{\mu^3 r_1 r_2 r_3}{n_1 n_2 n_3}} \kappa \sigma_{\min}(\bcX_\star).
$
\end{lemma}

\begin{proof} By the property of matricizations (cf.~\eqref{eq:matricization_property}), we have $\Matricize{k}{\bcX_\star} = \bU_\star^{(k)} \Breve{\bU}_\star^{(k)\top}$ for any $k=1,2,3$.
Furthermore, $\norm{\cdot}_\infty$ is invariant to matricizations, so 
$$\norm{\bcX_\star}_\infty = \norm{\Matricize{k}{\bcX_\star}}_\infty = \norm{\bU_\star^{(k)}\Breve{\bU}_\star^{(k)\top}}_\infty .$$
Without loss of generality, we choose $k=1$. It then follows that
\begin{align*}
    \norm{\bcX_\star}_\infty \leq \norm{\bU_\star^{(1)}}_{2, \infty} \norm{\Breve{\bU}_\star^{(1)}}_{2, \infty}
   &  \leq \norm{\bU_\star^{(1)}}_{2, \infty} \norm{\bU_\star^{(3)}}_{2, \infty} \norm{\bU_\star^{(2)}}_{2, \infty} \norm{\Matricize{1}{\bcG_\star}}_{\op} \\
   & \leq  \sqrt{\frac{\mu^3 r_1 r_2 r_3}{n_1 n_2 n_3}} \sigma_{\max}(\Matricize{1}{\bcX_\star}) ,
\end{align*}
where the second line follows from the definition of incoherence of $\bcX_\star$, i.e., $\norm{\bU_\star^{(k)}}_{2, \infty} \leq \sqrt{\frac{\mu r_k}{n_k}}$, and $ \norm{\Matricize{1}{\bcG_\star}}_{\op} =  \sigma_{\max}(\Matricize{1}{\bcX_\star}) $. Applying the above bound for any $k$ and taking the tightest bound, we have
\begin{align*}
    \norm{\bcX_\star}_\infty &\leq \sqrt{\frac{\mu^3 r_1 r_2 r_3}{n_1 n_2 n_3}} \min_k \sigma_{\max}(\Matricize{k}{\bcX_\star}) =\sqrt{\frac{\mu^3 r_1 r_2 r_3}{n_1 n_2 n_3}}  \kappa \sigma_{\min}(\bcX_\star) ,
\end{align*}
where we have used the definition of $\kappa$.
\end{proof}

We next show a key tensor algebraic result that is crucial in establishing the incoherence property of the spectral initialization.

\begin{lemma} \label{trunc_hosvd}
Given a tensor  $\bcT \in \RR^{n_1 \times n_2 \times n_3}$, suppose its rank-$\br$ truncated HOSVD is $(\bU^{(1)}, \bU^{(2)}, \bU^{(3)}) \bcdot \bcG$ with $\bU^{(k)} \in \RR^{n_k \times r_k}$ and $\bcG \in \RR^{r_1 \times r_2 \times r_3}$, and $r_k \leq n_k$ for $k= 1, 2, 3$. Then, 
\begin{align}
    \bU^{(k) \top} \Matricize{k}{\bcT} \Breve{\bU}^{(k)} &= \Breve{\bU}^{(k) \top} \Breve{\bU}^{(k)} , \label{eq:trunc_hosvd1} \\
    \Matricize{k}{\bcT} \Matricize{k}{\bcT}^\top \bU^{(k)} &= \bU^{(k)} \Breve{\bU}^{(k) \top} \Breve{\bU}^{(k)}   \label{eq:trunc_hosvd2},
\end{align}
where $\Breve{\bU}^{(k)}$ is defined in~\eqref{eq:matricization_property}.
\end{lemma}

\begin{proof}
Set the full HOSVD of $\bcT$ to be $\big(\bU_{\bcT}^{(1)}, \bU_{\bcT}^{(2)}, \bU_{\bcT}^{(3)} \big) \bcdot \bcG_{\bcT}$.
Since $\bU_{\bcT}^{(k)}$ contains all the left singular vectors of $\Matricize{k}{\bcT}$, it can be decomposed into the following block structure:
\begin{align} \label{eq:trunc_hosvd_factor_block}
    \bU_{\bcT}^{(k)} = \begin{bmatrix}
        \bU^{(k)} &\Bar{\bU}^{(k)}
    \end{bmatrix},
\end{align}
where $\Bar{\bU}^{(k)} \in \RR^{n_k \times (n_k - r_k)}$ contains the bottom $(n_k - r_k)$ left singular vectors. The rest of the proof  focuses on the first mode (i.e., $k=1$), while other modes follow from similar arguments.

Let us begin with proving \eqref{eq:trunc_hosvd1}.
Plugging in the definition of $\Breve{\bU}^{(1)}$ from $\eqref{matricize1}$, we see that 
\begin{align*}
    \bU^{(1) \top} \Matricize{1}{\bcT} \Breve{\bU}^{(1)} &= \bU^{(1) \top} \bU_{\bcT}^{(1)} \Matricize{1}{\bcG_{\bcT}} \big( \bU_{\bcT}^{(3)} \otimes \bU_{\bcT}^{(2)} \big)^\top \big(\bU^{(3)} \otimes \bU^{(2)} \big) \Matricize{1}{\bcG}^\top
    \\
    &= 
    \bU^{(1) \top}
    \begin{bmatrix}
        \bU^{(1)} &\Bar{\bU}^{(1)}
    \end{bmatrix}
    \Matricize{1}{\bcG_{\bcT}}
    \left(
    \begin{bmatrix}
        \bU^{(3)\top} \\
        \Bar{\bU}^{(3)\top}
    \end{bmatrix}
    \otimes 
    \begin{bmatrix}
        \bU^{(2)\top} \\
        \Bar{\bU}^{(2)\top}
    \end{bmatrix}
    \right) \big(\bU^{(3)} \otimes \bU^{(2)} \big) \Matricize{1}{\bcG}^\top ,
\end{align*}
where the second line uses the block structure \eqref{eq:trunc_hosvd_factor_block}.
By the mixed product property of Kronecker products, we have
\begin{align}
    \bU^{(1) \top} \Matricize{1}{\bcT} \Breve{\bU}^{(1)} &= 
    \begin{bmatrix}
        \bI_{r_1} &\bm 0
    \end{bmatrix}
    \Matricize{1}{\bcG_{\bcT}}
    \left(
    \begin{bmatrix}
        \bU^{(3)\top} \bU^{(3)}\\
        \Bar{\bU}^{(3)\top} \bU^{(3)}
    \end{bmatrix}
    \otimes 
    \begin{bmatrix}
        \bU^{(2)\top} \bU^{(2)}\\
        \Bar{\bU}^{(2)\top} \bU^{(2)}
    \end{bmatrix}
    \right) \Matricize{1}{\bcG}^\top  \nonumber
    \\
    &= 
    \begin{bmatrix}
        \bI_{r_1} &\bm 0
    \end{bmatrix}
    \Matricize{1}{\bcG_{\bcT}}
    \left(
    \begin{bmatrix}
        \bI_{r_3}\\
        \bm 0
    \end{bmatrix}
    \otimes 
    \begin{bmatrix}
        \bI_{r_2}\\
        \bm 0
    \end{bmatrix}
    \right) \Matricize{1}{\bcG}^\top,  \label{eq:trunc_hosvd1_1}
\end{align}
where we used the fact that the singular vectors are orthonormal. Note that
$$
\begin{bmatrix}
    \bI_{r_1} &\bm 0
\end{bmatrix}
\Matricize{1}{\bcG_{\bcT}}
\left(
\begin{bmatrix}
    \bI_{r_3}\\
    \bm 0
\end{bmatrix}
\otimes 
\begin{bmatrix}
    \bI_{r_2}\\
    \bm 0
\end{bmatrix}
\right) =  \Matricize{1}{ \left(
\begin{bmatrix}
    \bI_{r_1} &\bm 0
\end{bmatrix}, 
\begin{bmatrix}
    \bI_{r_2} &\bm 0
\end{bmatrix},
\begin{bmatrix}
    \bI_{r_3} &\bm 0
\end{bmatrix}
\right) \bcdot \bcG_{\bcT}} \in \RR^{r_1 \times r_2 r_3},
$$
where $\left(
\begin{bmatrix}
    \bI_{r_1} &\bm 0
\end{bmatrix}, 
\begin{bmatrix}
    \bI_{r_2} &\bm 0
\end{bmatrix},
\begin{bmatrix}
    \bI_{r_3} &\bm 0
\end{bmatrix}
\right) \bcdot \bcG_{\bcT}$ is equivalent to trimming off the entries $[\bcG_{\bcT}]_{i_1, i_2, i_3}$ for all $i_1 > r_1$, $i_2 > r_2$, or $i_3 > r_3$.
%In other words,
%\begin{align*}
%    \left[\left(
%    \begin{bmatrix}
%        \bI_{r_1} &\bm 0
%    \end{bmatrix}, 
%    \begin{bmatrix}
%        \bI_{r_2} &\bm 0
%    \end{bmatrix},
%    \begin{bmatrix}
%        \bI_{r_3} &\bm 0
%    \end{bmatrix}
%    \right) \bcdot \bcG_{\bcT}\right]_{i_1, i_2, i_3} &= [\bcG_{\bcT}]_{i_1, i_2, i_3} \qquad \text{for } 1\leq i_1 \leq n_1, 1\leq i_2 \leq r_2, 1\leq i_3 \leq r_3 .
%\end{align*}
Since $\bcG$ is the section of $[\bcG_{\bcT}]_{i_1, i_2, i_3}$ where $1\leq i_1 \leq r_1, 1\leq i_2 \leq r_2, 1\leq i_3 \leq r_3$ \cite{vannieuwenhoven2012new}, we have
\begin{align*}
   \left(
    \begin{bmatrix}
        \bI_{r_1} &\bm 0
    \end{bmatrix}, 
    \begin{bmatrix}
        \bI_{r_2} &\bm 0
    \end{bmatrix},
    \begin{bmatrix}
        \bI_{r_3} &\bm 0
    \end{bmatrix}
    \right) \bcdot \bcG_{\bcT}  &= \bcG .
\end{align*}
This allows us to simplify \eqref{eq:trunc_hosvd1_1} as
\begin{align*}
    \bU^{(1) \top} \Matricize{1}{\bcT} \Breve{\bU}^{(1)} = \Matricize{1}{\bcG} \Matricize{1}{\bcG}^\top = \Breve{\bU}^{(1) \top} \Breve{\bU}^{(1)}, 
\end{align*}
where the last equality follows from the definition of $\Breve{\bU}^{(1)}$ from \eqref{eq:matricization_property} and $(\bU^{(3)} \otimes \bU^{(2)})^\top (\bU^{(3)} \otimes \bU^{(2)}) = \bI$ by construction of the HOSVD. This completes the proof of \eqref{eq:trunc_hosvd1}.

Turning to \eqref{eq:trunc_hosvd2}, we begin with the observation
\begin{align}
    \Matricize{1}{\bcT} \Matricize{1}{\bcT}^\top &= \bU_{\bcT}^{(1)} \Matricize{1}{\bcG_{\bcT}} \big(\bU_{\bcT}^{(3)} \otimes \bU_{\bcT}^{(2)} \big)^\top \big(\bU_{\bcT}^{(3)} \otimes \bU_{\bcT}^{(2)} \big) \Matricize{1}{\bcG_{\bcT}}^\top \bU_{\bcT}^{(1) \top} \nonumber\\
    &= \bU_{\bcT}^{(1)} \Matricize{1}{\bcG_{\bcT}} \Matricize{1}{\bcG_{\bcT}}^\top \bU_{\bcT}^{(1) \top} \label{eq:trunc_hosvd2_1}.
\end{align}
By the ``all-orthogonal'' property in \cite{de2000multilinear}, $\Matricize{1}{\bcG_{\bcT}} \Matricize{1}{\bcG_{\bcT}}^\top = (\bSigma_{\bcT}^{(1)})^2$, the squared singular value matrix of $\Matricize{1}{\bcT}$.
By assigning $\bSigma_{\bcT}^{(1)}$ with the block representation
\begin{align*}
    \bSigma_{\bcT}^{(1)} &= \begin{bmatrix}
        \bSigma^{(1)} &\bm 0\\
        \bm 0 &\Bar{\bSigma}^{(1)}
    \end{bmatrix},
\end{align*}
where $\bSigma^{(1)} = \sqrt{\Matricize{1}{\bcG} \Matricize{1}{\bcG}^\top}$ and $\Bar{\bSigma}^{(1)}$ contain the top $r_1$ singular values and bottom $n_1 - r_1$ singular values, respectively, of $\Matricize{1}{\bcT}$.
Coupled with same block structure as in \eqref{eq:trunc_hosvd_factor_block}, \eqref{eq:trunc_hosvd2_1} becomes
\begin{align*}
    \Matricize{1}{\bcT} \Matricize{1}{\bcT}^\top &= \begin{bmatrix}
        \bU^{(1)} &\Bar{\bU}^{(1)}
    \end{bmatrix} 
    \begin{bmatrix}
        (\bSigma^{(1)})^2 &\bm 0\\
        \bm 0 &(\Bar{\bSigma}^{(1)})^2
    \end{bmatrix}
    \begin{bmatrix}
        \bU^{(1) \top} \\
        \Bar{\bU}^{(1) \top}
    \end{bmatrix} .
\end{align*}
Multiply by $\bU^{(1)}$ on the right to arrive at
$
    \Matricize{1}{\bcT} \Matricize{1}{\bcT}^\top \bU^{(1)} = \bU^{(1)} (\bSigma^{(1)})^2  = \bU^{(1)} \Breve{\bU}^{(1) \top} \Breve{\bU}^{(1)} .
$
\end{proof}

% We now proceed with a few useful technical lemmas that will be used throughout our analysis. 
% The first lemma bounds the $\ell_{\infty}$ norm of a $\mu$-incoherent low-rank tensor.

\subsection{Perturbation bounds}

Below is a useful perturbation bound for matrices.

\begin{lemma} \label{align_perturb}
%\todo[inline]{probably dont need pd-ness}
Given two matrices $\bU, \bU_\star \in \RR^{n \times r}$ that have full column rank, two invertible matrices $\Bar{\bQ}, \bQ \in \RR^{r \times r}$, and a positive definite matrix $\bSigma \in \RR^{r \times r}$. Suppose that $\sigma_{\min}(\bU_\star) > \norm{ \bU \Bar{\bQ} - \bU_\star }_{\op}$. Then the following holds true
\begin{align*}
   \norm{\Bar{\bQ}^{-1} (\bQ - \Bar{\bQ}) \bSigma}_{\op} &\leq \frac{\norm{\bU (\bQ - \Bar{\bQ})  \bSigma}_{\op}}{\sigma_{\min}(\bU_\star) - \norm{\bU \Bar{\bQ} - \bU_\star}_{\op}} .
\end{align*} 
\end{lemma}

\begin{proof}
It follows that
\begin{align*}
    \norm{\Bar{\bQ}^{-1} (\bQ - \Bar{\bQ}) \bSigma}_{\op} &= \norm{\Bar{\bQ}^{-1} (\bU^\top \bU)^{-1} \bU^\top \bU (\bQ - \Bar{\bQ})  \bSigma}_{\op}
    \\
    &\leq \norm{\Bar{\bQ}^{-1} (\bU^\top \bU)^{-1} \bU^\top}_{\op} \norm{\bU (\bQ - \Bar{\bQ})  \bSigma}_{\op}
    \\
    &=\frac{\norm{\bU (\bQ - \Bar{\bQ})  \bSigma}_{\op}}{\sigma_{\min}(\bU \Bar{\bQ})},
\end{align*}
where the last equality comes from the fact that $\Bar{\bQ}^{-1} (\bU^\top \bU)^{-1} \bU^\top$ and $\bU \Bar{\bQ}$ are pseudoinverses.
By Weyl's inequality, we know $|\sigma_{i}(\bM) - \sigma_{i}(\bM +\bDelta_{\bM})| \leq \norm{\bDelta_{\bM}}$.
Taking $\bM := \bU_\star$ and $\bDelta_{\bM} := \bU \Bar{\bQ} - \bU_\star$, we have
\begin{align*}
    \norm{\Bar{\bQ}^{-1} (\bQ - \Bar{\bQ}) \bSigma}_{\op} &\leq \frac{\norm{\bU (\bQ - \Bar{\bQ})  \bSigma}_{\op}}{\sigma_{\min}(\bU_\star) - \norm{\bU \Bar{\bQ} - \bU_\star}_{\op}} 
\end{align*}
as long as $\sigma_{\min}(\bU_\star) > \norm{ \bU \Bar{\bQ} - \bU_\star }_{\op}$.
\end{proof}

We also collect a useful lemma regarding perturbation bounds for tensors from \cite{tong2021scaling}.

\begin{lemma}[{\cite[Lemma 10]{tong2021scaling}}] \label{more_perturb}
Suppose $\bF = (\bU^{(1)}, \bU^{(2)}, \bU^{(3)}, \bcG)$ and $\bF = (\bU_\star^{(1)}, \bU_\star^{(2)}, \bU_\star^{(3)}, \bcG_\star)$ are aligned, and $\dist(\bF, \bF_\star) \leq \epsilon \sigma_{\min}(\bcX_\star)$ for some $0< \epsilon < 1$.
Then, the following bounds are true:
\begin{subequations}
\begin{align}
    \norm{\Matricize{k}{\bDelta_{\bcG}}^\top  (\bSigma_\star^{(k)})^{-1}}_{\op} &\leq \epsilon; \label{eq:perturb_spec_a}\\ 
    \norm{\bU^{(k)} \big(\bU^{(k)\top} \bU^{(k)}\big)^{-1}}_{\op} &\leq \frac{1}{1-\epsilon} ; \label{eq:perturb_spec_b}\\
    \norm{\Breve{\bU}^{(k)} \big(\Breve{\bU}^{(k)\top} \Breve{\bU}^{(k)} \big)^{-1}  \bSigma_\star^{(k)} }_{\op} &\leq \frac{1}{(1-\epsilon)^3} ; \label{eq:perturb_spec_f}\\
    \norm{\bSigma_\star^{(k)} \big(\Breve{\bU}^{(k)\top} \Breve{\bU}^{(k)} \big)^{-1}  \bSigma_\star^{(k)}  }_{\op} &\leq \frac{1}{(1-\epsilon)^6} ; \label{eq:perturb_spec_h}\\
    \norm{\Breve{\bU}^{(1)} - \Breve{\bU}_\star^{(1)}}_{\fro} \leq \left(1 + \epsilon + \frac{\epsilon^2}{3} \right) &\left(\norm{ \big( \bU^{(2)}-\bU^{2)}_\star \big) \bSigma_\star^{(2)}}_{\fro} + \norm{  \big( \bU^{(3)}-\bU^{3)}_\star \big) \bSigma_\star^{(3)}}_{\fro} + \norm{\bcG - \bcG_{\star}}_{\fro}\right) . \label{eq:perturb_fro_c}
\end{align}
\end{subequations}
For \eqref{eq:perturb_fro_c}, similar bounds exist for the other modes.
Furthermore, if $0 < \epsilon \leq 0.2$,
\begin{align}\label{eq:perturb_dist}
    \norm{ \big(\bU^{(1)}, \bU^{(2)}, \bU^{(3)} \big) \bcdot \bcG - \bcX_\star}_{\fro} \leq 3 \dist(\bF, \bF_\star) . 
\end{align}
\end{lemma}

 The next set of lemmas, which is crucial in our analysis, deals with perturbation bounds when relating a tensor  $\bcX = \big(\bU^{(1)}, \bU^{(2)}, \bU^{(3)} \big) \bcdot \bcG$ to the ground truth $\bcX_\star$, where the tensor tuples $\bF= (\bU^{(1)}, \bU^{(2)}, \bU^{(3)}, \bcG)$ and $\bF_\star = (\bU_\star^{(1)}, \bU_\star^{(2)}, \bU_\star^{(3)}, \bcG_\star)$ are aligned. 
 
\begin{lemma} \label{a_delta_a_bounds}
Let $\bcX_\star\in\mathbb{R}^{n_1\times n_2\times n_3}$ be $\mu$-incoherent with the  Tucker decomposition $\bcX_{\star} = \big(\bU_\star^{(1)}, \bU_\star^{(2)}, \bU_\star^{(3)} \big) \bcdot \bcG_\star$ of rank $\br=(r_1,r_2,r_3)$, and  $\big\{\bSigma_\star^{(k)} \big\}_{k=1, 2, 3}$ be the set of singular value matrices of different matricizations of $\bcX_\star$.  
In addition, let $\bF := (\bU^{(1)}, \bU^{(2)}, \bU^{(3)}, \bcG)$ and $\bF_\star := (\bU_\star^{(1)}, \bU_\star^{(2)}, \bU_\star^{(3)}, \bcG_\star)$ be aligned, where $\bcX = \big(\bU^{(1)}, \bU^{(2)}, \bU^{(3)} \big) \bcdot \bcG$. Suppose 
\begin{align} \label{eq:a_delta_a_assump}
    \max_k \left\{\sqrt{\frac{n_k}{r_k}} \norm{ \big( \bU^{(k)}-\bU^{(k)}_\star \big) \bSigma_\star^{(k)}}_{2, \infty} \right\} \leq c \sqrt{\mu} \sigma_{\min}(\bcX_\star) 
\end{align}
for some $0 < c \leq 1$. Then for $k=1, 2, 3$, 
\begin{align} \label{eq:a_delta_a_bound}
    \norm{ \bU^{(k)}-\bU^{(k)}_\star }_{2, \infty} &\leq c \sqrt{\frac{\mu r_k}{n_k}} , \qquad \mbox{and}\qquad
    \norm{\bU^{(k)}}_{2, \infty}  \leq 2\sqrt{\frac{\mu r_k}{n_k}}.  
\end{align}
\end{lemma}

\begin{proof}
It follows that for all $k$,
\begin{align*}
\norm{ \bU^{(k)}-\bU^{(k)}_\star }_{2, \infty} \leq \frac{1}{\sigma_{\min}(\bSigma_\star^{(k)}) }  \norm{ \big( \bU^{(k)}-\bU^{(k)}_\star \big) \bSigma_\star^{(k)}}_{2, \infty} \leq c \sqrt{\frac{\mu r_k}{n_k}} ,
\end{align*}
where the second inequality follows from \eqref{eq:a_delta_a_assump} and $\sigma_{\min}(\bcX_\star) \leq \sigma_{\min}(\bSigma_\star^{(k)})$. This completes the proof for the first part of \eqref{eq:a_delta_a_bound}. With this and the incoherence assumption $\norm{\bU_\star^{(k)}}_{2, \infty} \leq \sqrt{\frac{\mu r_k}{n_k}}$, after applying triangle inequality, we arrive at
\begin{align*}
    \norm{\bU^{(k)}}_{2, \infty} \leq \norm{ \bU^{(k)}-\bU^{(k)}_\star }_{2, \infty} + \norm{\bU_\star^{(k)}}_{2, \infty} \leq 2\sqrt{\frac{\mu r_k}{n_k}},
\end{align*}
which completes the proof.
\end{proof}

%The next lemma bounds $\norm{\bcX - \bcX_\star}_\infty$, the maximum entry-wise error of our low rank approximation.

\begin{lemma} \label{delta_x_inf_bound_contract_incoherence}
Let $\bcX_\star\in\mathbb{R}^{n_1\times n_2\times n_3}$ be $\mu$-incoherent with the  Tucker decomposition $\bcX_{\star} = \big(\bU_\star^{(1)}, \bU_\star^{(2)}, \bU_\star^{(3)} \big) \bcdot \bcG_\star$ of rank $\br=(r_1,r_2,r_3)$. In addition, let $\bF := (\bU^{(1)}, \bU^{(2)}, \bU^{(3)}, \bcG)$ and $\bF_\star := (\bU_\star^{(1)}, \bU_\star^{(2)}, \bU_\star^{(3)}, \bcG_\star)$ be aligned, where $\bcX = \big(\bU^{(1)}, \bU^{(2)}, \bU^{(3)} \big) \bcdot \bcG$. For $0<\epsilon_0< 0.1$ and $0 < c \leq 1$, if
\begin{subequations}
\begin{align} 
    \dist(\bF, \bF_\star) & \leq \epsilon_0 c \sigma_{\min}(\bcX_\star) , \label{delta_x_bound_dist_assump} \\
    \max_k \left\{\sqrt{\frac{n_k}{r_k}} \norm{ \big(\bU^{(k)} -\bU_{\star}^{(k)} \big)  \bSigma_\star^{(k)}}_{2, \infty} \right\} & \leq c \sqrt{\mu} \sigma_{\min}(\bcX_\star)  \label{delta_x_bound_incoherence_assump}
\end{align}
\end{subequations}
are satisfied, then
\begin{align*}
    \norm{\bcX - \bcX_\star}_\infty &\leq \sqrt{\frac{\mu^3 r_1 r_2 r_3}{n_1 n_2 n_3}}  \left(8 \epsilon_0  + 7 \right) c \sigma_{\min}(\bcX_\star)  \leq 8 \sqrt{\frac{\mu^3 r_1 r_2 r_3}{n_1 n_2 n_3}} c \sigma_{\min}(\bcX_\star).
\end{align*}
\end{lemma}

\begin{proof}
% For the proof of this lemma, define the following:
% \begin{align*}
%     \bU^{(k)} &:= \bU_t^{(k)} \bQ_t^{(k)} \\
%     \bcG &:= ((\bQ_t^{(1)})^{-1}, (\bQ_t^{(2)})^{-1}, (\bQ_t^{(3)})^{-1}) \bcdot \bcG_t \\
%     \bDelta_{\bU^{(k)}} &:= \bU^{(k)}-\bU^{(k)}_\star \\
%     \bDelta_{\bcG} &:= \bcG - \bcG_\star
% \end{align*}
% Note that 
% \begin{align*}
%     (\bU^{(1)}, \bU^{(2)}, \bU^{(3)}) \bcdot \bcG = \left( (\bU^{(1)} \bQ^{(1)}, \bU^{(2)} \bQ^{(2)}, \bU^{(3)} \bQ^{(3)}) \right) \bcdot \left(((\bQ^{(1)})^{-1}, (\bQ^{(2)})^{-1}, (\bQ^{(3)})^{-1}) \bcdot \bcG \right) ,
% \end{align*}
We can decompose $\bcX - \bcX_\star$ into 
\begin{align} 
    \bcX - \bcX_\star 
    &= \big(\bU^{(1)}, \bU^{(2)}, \bU^{(3)} \big) \bcdot (\bcG - \bcG_\star) + \big( \bU^{(1)} - \bU_{\star}^{(1)}, \bU^{(2)}, \bU^{(3)} \big) \bcdot \bcG_\star \nonumber \\
    &\qquad + \big(\bU_\star^{(1)}, \bU^{(2)} - \bU_{\star}^{(2)}, \bU^{(3)} \big) \bcdot \bcG_\star + \big(\bU_\star^{(1)}, \bU_\star^{(2)}, \bU^{(3)} - \bU_{\star}^{(3)} \big) \bcdot \bcG_\star \label{delta_x_decomp}.
\end{align}
Then by the triangle inequality, 
\begin{align*}
    \norm{ \bcX - \bcX_\star }_\infty &\leq \norm{\big(\bU^{(1)}, \bU^{(2)}, \bU^{(3)} \big) \bcdot (\bcG - \bcG_\star)  }_\infty + \norm{\big( \bU^{(1)} - \bU_{\star}^{(1)}, \bU^{(2)}, \bU^{(3)} \big) \bcdot \bcG_\star}_\infty \\
    &\qquad + \norm{\big(\bU_\star^{(1)}, \bU^{(2)} - \bU_{\star}^{(2)}, \bU^{(3)} \big) \bcdot \bcG_\star}_\infty + \norm{\big(\bU_\star^{(1)}, \bU_\star^{(2)}, \bU^{(3)} - \bU_{\star}^{(3)} \big) \bcdot \bcG_\star}_\infty
    \\
    &= \underbrace{\norm{\bU^{(1)} \Matricize{1}{\bcG - \bcG_\star}(\bU^{(3)} \otimes \bU^{(2)})^\top}_\infty}_{=: \mfk{A}_{\mathrm{core}}} + \underbrace{\norm{  \big(\bU^{(1)} -\bU_{\star}^{(1)} \big)  \Matricize{1}{\bcG_\star}(\bU^{(3)} \otimes \bU^{(2)})^\top}_\infty}_{=:  \mfk{A}_1} \\
    &\quad + \underbrace{\norm{  \big(\bU^{(2)} -\bU_{\star}^{(2)} \big)   \Matricize{2}{\bcG_\star}(\bU^{(3)} \otimes \bU_\star^{(1)})^\top}_\infty}_{=: \mfk{A}_2} + \underbrace{\norm{ \big(\bU^{(3)} -\bU_{\star}^{(3)} \big)\Matricize{3}{\bcG_\star}(\bU_\star^{(2)} \otimes \bU_\star^{(1)})^\top}_\infty}_{=: \mfk{A}_3},
\end{align*}
where the second inequality follows from the invariance of $\ell_\infty$ norm to matricizations.
We will bound each term separately.  

\begin{itemize}
\item For $\mfk{A}_{\mathrm{core}}$, it follows from basic norm relations that
\begin{align*}
    \mfk{A}_{\mathrm{core}} &\leq \norm{\bU^{(1)}}_{2, \infty} \norm{\Matricize{1}{\bcG - \bcG_\star}}_{\op} \norm{\bU^{(3)} \otimes  \bU^{(2)} }_{2, \infty} \\
    &\leq \norm{\bU^{(1)}}_{2, \infty} \norm{\bU^{(2)}}_{2, \infty} \norm{\bU^{(3)}}_{2, \infty} \norm{\bcG - \bcG_\star}_{\fro} \leq 8\sqrt{\frac{\mu^3 r_1 r_2 r_3}{n_1 n_2 n_3}}\epsilon_0 c \sigma_{\min}(\bcX_\star) ,
\end{align*}
where the last inequality follows from $\norm{\bcG - \bcG_\star}_{\fro} \leq \dist(\bF, \bF_\star) \leq \epsilon_0 c \sigma_{\min}(\bcX_\star)$ by assumption  \eqref{delta_x_bound_dist_assump} and Lemma \ref{a_delta_a_bounds} by assumption \eqref{delta_x_bound_incoherence_assump}.

\item Next, for $\mfk{A}_1$,  
\begin{align*}
    \mfk{A}_1 &\leq \norm{ \big( \bU^{(1)} - \bU_{\star}^{(1)} \big) \bSigma_\star^{(1)}}_{2, \infty}\norm{\bU^{(2)}}_{2, \infty} \norm{\bU^{(3)}}_{2, \infty}  \norm{\Matricize{1}{\bcG_\star}^\top \big(\bSigma_\star^{(1)} \big)^{-1}}_{\op}
    \\
    & = \norm{ \big( \bU^{(1)} - \bU_{\star}^{(1)} \big) \bSigma_\star^{(1)}}_{2, \infty}\norm{\bU^{(2)}}_{2, \infty} \norm{\bU^{(3)}}_{2, \infty} \leq 4 \sqrt{\frac{\mu^3 r_1 r_2 r_3}{n_1 n_2 n_3}} c \sigma_{\min}(\bcX_\star) ,
\end{align*}
where the equality follows from $\norm{\Matricize{k}{\bcG_\star}^\top \big(\bSigma_\star^{(k)} \big)^{-1}}_{\op} = 1$ since $\Matricize{k}{\bcG_\star} \Matricize{k}{\bcG_\star}^\top = \left(\bSigma_\star^{(k)}\right)^2$, and the last inequality follows from the assumption \eqref{delta_x_bound_incoherence_assump} and Lemma \ref{a_delta_a_bounds} by assumption \eqref{delta_x_bound_incoherence_assump}.

\item Similarly, for $\mfk{A}_2$, it follows
\begin{align*}
    \mfk{A}_2 &\leq \norm{  \big( \bU^{(2)} - \bU_{\star}^{(2)} \big)  \bSigma_\star^{(2)}}_{2, \infty}   \norm{\bU^{(3)}}_{2, \infty} \norm{\bU_\star^{(1)}}_{2, \infty} \norm{\Matricize{2}{\bcG_\star}^\top \big(\bSigma_\star^{(2)} \big)^{-1}}_{\op}  \\
    &\leq 2  \sqrt{\frac{\mu^3 r_1 r_2 r_3}{n_1 n_2 n_3}} c \sigma_{\min}(\bcX_\star) .
\end{align*}

\item Finally, repeat the same approach for $\mfk{A}_3$ to get
\begin{align*}
    \mfk{A}_3 &\leq \norm{  \big( \bU^{(3)} - \bU_{\star}^{(3)} \big)  \bSigma_\star^{(3)}}_{2, \infty} \norm{\bU_\star^{(2)}}_{2,\infty}  \norm{\bU_\star^{(1)}}_{2,\infty} \norm{ \Matricize{3}{\bcG_\star}^\top \big(\bSigma_\star^{(3)} \big)^{-1}}_{\op} \\
    &\leq \sqrt{\frac{\mu^3 r_1 r_2 r_3}{n_1 n_2 n_3}} c \sigma_{\min}(\bcX_\star) .
\end{align*}
\end{itemize}
Putting these together, we have the advertised bound.
%\begin{align*}
%    \norm{\bcX - \bcX_\star}_\infty &\leq \sqrt{\frac{\mu^3 r_1 r_2 r_3}{n_1 n_2 n_3}}  \left( 8 \epsilon_0  + 7 \right) c \sigma_{\min}(\bcX_\star) .
%\end{align*}
\end{proof}

\subsection{Sparse outliers}
The following two lemmas are useful to control the sparse corruption term, of which the second lemma follows directly from translating \cite[Lemma 5]{cai2021learned} to the tensor case.

\begin{lemma}[{\cite[Lemma 1]{yi2016fast}}{\cite[Lemma 6]{cai2021learned}}]\label{sparsity_norm_eq}
Suppose that $\bS \in \RR^{m\times n}$ is $\alpha$-sparse. Then one has
\begin{align*}
    \norm{\bS}_{\op} \leq \alpha \sqrt{m n} \norm{\bS}_\infty, \qquad  \norm{\bS}_{2,\infty} \leq   \sqrt{\alpha n} \norm{\bS}_\infty, \qquad  \text{and} \qquad \norm{\bS}_{1,\infty} \leq  \alpha n \norm{\bS}_\infty.
\end{align*}
\end{lemma}

\begin{lemma}[{\cite[Lemma 5]{cai2021learned}}]\label{corrupt_iter}
Suppose that $\bcY = \bcX_\star + \bcS_\star$ for some $\alpha$-sparse $\bcS_\star$.
Fix a tensor $\bcX$, and let $\bcS = \Shrink{\zeta}{\bcY - \bcX}$ where the threshold satisfies $\zeta \geq \norm{\bcX - \bcX_\star}_\infty$. We then have
\begin{equation}\label{eq:sparse_bound}
    \norm{\bcS - \bcS_\star}_\infty \leq \norm{\bcX - \bcX_\star}_\infty + \zeta \leq 2 \zeta
\end{equation}
and
\begin{equation}\label{eq:sparse_supp}
    \supp(\bcS) \subseteq \supp(\bcS_\star) .
\end{equation}
The relation \eqref{eq:sparse_supp} also implies that $\bcS - \bcS_\star$ is $\alpha$-sparse.
\end{lemma}

\end{document}